\documentclass{article} 

\usepackage[utf8]{inputenc} 
\usepackage[T1]{fontenc}    

\usepackage[colorlinks,citecolor=blue!70!black,linkcolor=blue!70!black, urlcolor=blue!70!black]{hyperref}       

\usepackage{url}            
\usepackage{booktabs}       
\usepackage{amsfonts}       
\usepackage{amssymb}
\usepackage{amsthm}
\usepackage{amsmath}
\usepackage{nicefrac}       
\usepackage{microtype}      
\usepackage{xcolor}         
\usepackage{wrapfig}
\usepackage{sidecap}
\usepackage{mathtools}
\usepackage{dsfont}
\usepackage{sidecap}
\usepackage{lipsum}
\usepackage{fullpage}
\usepackage{graphicx}
\usepackage{subfig}
\usepackage{booktabs}      

\usepackage{comment}
\usepackage[toc,page,header]{appendix}
\usepackage{minitoc}
\usepackage{wrapfig}

\usepackage{enumitem}
\usepackage{multirow}
\usepackage{floatrow}
\newfloatcommand{capbtabbox}{table}[][\FBwidth]

\allowdisplaybreaks
\usepackage{natbib}
\usepackage{xspace}

\usepackage{tikz}
\usepackage{comment} 
\usepackage{caption}

\usepackage{algorithm2e}
\usepackage{algorithmicx}
\usepackage[noend]{algpseudocode}

\usepackage{std_definitions}

\newcommand{\edo}{\phi^\star}
\newcommand{\exo}{\phi^\star_\xi}
\newcommand{\Tex}{T_\xi}
\newcommand{\muex}{\mu_\xi}
\newcommand{\TV}[1]{\nbr{#1}_{\textrm{TV}}}
\newcommand{\pairD}{\Dcal_{\textrm{for}}}
\newcommand{\etamin}{\eta_{\textrm{min}}}
\newcommand{\nsamp}{n_{\textrm{samp}}}
\newcommand{\ghat}{\hat{g}}

\newcommand{\upairdist}{D_{\textrm{unf}}}
\newcommand{\kpairdist}{D_k}
\newcommand{\pairdist}{D_{pr}}

\newcommand{\lossauto}{\ell_{\textrm{auto}}}
\newcommand{\losstemp}{\ell_{\textrm{temp}}}
\newcommand{\lossfor}{\ell_{\textrm{for}}}
\newcommand{\losscont}{\ell_{\textrm{cont}}}

\newcommand{\phihat}{\hat{\phi}}

\newcommand{\fhat}{\hat{f}}

\newcommand{\cont}{\textrm{cont}}
\newcommand{\deltacont}{\Delta_{\textrm{cont}}}
\newcommand{\distcont}{D_{\textrm{cont}}}
\newcommand{\dcoup}{D_{\textrm{coup}}}

\newcommand{\unf}{\textrm{Unf}}
\newcommand{\supp}{\textrm{supp}}
\newcommand{\psiall}{\Psi_{\textrm{all}}}

\newcommand{\formargin}{\beta_{\textrm{for}}}
\newcommand{\tempmargin}{\beta_{\textrm{temp}}}
\newcommand{\formargink}{\beta_{\textrm{for}}^{(k)}}
\newcommand{\tempmargink}{\beta_{\textrm{temp}}^{(k)}}
\newcommand{\formarginu}{\beta_{\textrm{for}}^{(u)}}
\newcommand{\tempmarginu}{\beta_{\textrm{temp}}^{(u)}}

\newcommand{\dpol}{\Pi_{D}}
\newcommand{\poly}{{\tt poly}}

\newcommand{\xtil}{\tilde{x}}
\newcommand{\stil}{\tilde{s}}
\newcommand{\xitil}{\tilde{\xi}}
\newcommand{\psamp}{p_{\textrm{samp}}}

\newcommand{\pidata}{\pi_{\textrm{data}}}

\newtheorem{assumption}{Assumption}

\usepackage{color-edits}
\renewcommand{\epsilon}{\varepsilon}
\renewcommand{\hat}{\widehat}

\usepackage{amsthm}
\usepackage{dsfont}
\usepackage{mathrsfs,nicefrac}

\usepackage[capitalize,nameinlink]{cleveref}

\Crefformat{equation}{#2Eq.\;(#1)#3}

\Crefformat{figure}{#2Figure #1#3}
\Crefformat{assumption}{#2Assumption #1#3}
\Crefname{assumption}{Assumption}{Assumptions}

\usepackage{crossreftools}
\Crefformat{figure}{#2Figure #1#3}
\Crefformat{assumption}{#2Assumption #1#3}
\Crefname{assumption}{Assumption}{Assumptions}

\makeatletter
\newcounter{HALG@line}
\renewcommand{\theHALG@line}{\thealgorithm.\arabic{ALG@line}}
\makeatother

\newcommand{\1}{\mathds 1}

\newcommand{\cpfname}[1]{\bf\em Proof of \cref{#1}}

\newcommand{\fct}[1]{{\scriptscriptstyle(#1)}}

\title{Towards Principled Representation Learning from Videos for Reinforcement Learning}

\author{Dipendra Misra\footnote{DM and AS contributed equally.} \quad Akanksha Saran$^*$\thanks{Work done while at Microsoft Research NYC.} \quad Tengyang Xie \quad Alex Lamb \quad John Langford\\[.5cm]
$^*$Microsoft Research New York\hspace{7mm}$^\dagger$Sony Research\\[.1cm]
\normalsize\texttt{ dimisra@microsoft.com, akanksha.saran@sony.com, \{tengyangxie, lambalex, jcl\}@microsoft.com}
}
\date{}

\begin{document}

\doparttoc 
\faketableofcontents 

\parttoc 

\maketitle

\begin{abstract}
We study pre-training representations for decision-making using video data, which is abundantly available for tasks such as game agents and software testing. Even though significant empirical advances have been made on this problem, a theoretical understanding remains absent. We initiate the theoretical investigation into principled approaches for representation learning and focus on learning the latent state representations of the underlying MDP using video data. We study two types of settings: one where there is iid noise in the observation, and a more challenging setting where there is also the presence of exogenous noise, which is non-iid noise that is temporally correlated, such as the motion of people or cars in the background. We study three commonly used approaches: autoencoding, temporal contrastive learning, and forward modeling. We prove upper bounds for temporal contrastive learning and forward modeling in the presence of only iid noise. We show that these approaches can learn the latent state and use it to do efficient downstream RL with polynomial sample complexity. When exogenous noise is also present, we establish a lower bound result showing that the sample complexity of learning from video data can be exponentially worse 
than learning from action-labeled trajectory data. This partially explains why reinforcement learning with video pre-training is hard. 
We evaluate these representational learning methods in two visual domains, yielding results that are consistent with our theoretical findings.
\end{abstract}

\section{Introduction}

Representations pre-trained on large amounts of offline data have led to significant advances in machine learning domains such as natural language processing~\citep{liu2019roberta,brown2020language} and multi-modal learning~\citep{lin2021vx2text,radford2021learning}. This has naturally prompted a similar undertaking in reinforcement learning (RL) with the goal of training a representation model that can be used in a policy to solve a downstream RL task. The natural choice of data for RL problems is trajectory data, which contains the agent's observation along with actions taken by the agent and the rewards received by it~\citep{sutton2018reinforcement}. A line of work has proposed approaches for learning representations with trajectory data in both offline~\citep{uehara2021representation,islam2022acro} and online learning settings~\citep{nachum2018near, bharadhwaj2022infopower}. However, unlike text and image data, which are abundant on the internet or naturally generated by users, trajectory data is comparatively limited and expensive to collect. In contrast, video data, which only contains a sequence of observations (without any action or reward labeling), is often plentiful, especially for domains such as gaming and software. 
This motivates a line of work considering learning representations for RL using video data~\citep{Zhao2022video}. \emph{But is there a principled foundation underlying these approaches?} \emph{Are representations learned from video data as useful as representations learned from trajectory data?} We initiate a theoretical understanding of these approaches to show when and how these approaches yield representations that can be used to solve a downstream RL task efficiently.

\begin{figure}
    \centering
    \includegraphics[scale=0.33]{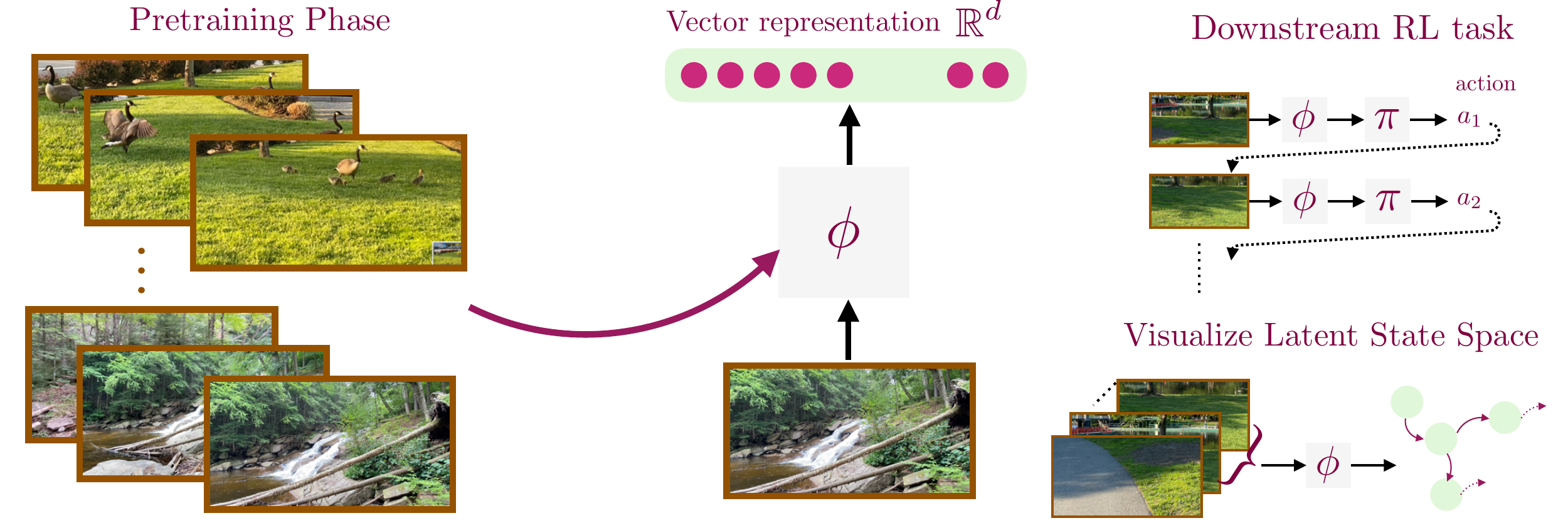}
    \caption{A flowchart of our video pre-training phase. \textbf{Left:} We assume access to a large set of videos (or, unlabeled episodes). \textbf{Center:} A representation learning method is used to train a model $\phi$ which maps an observation to a vector representation. \textbf{Right:} This representation can be used in a downstream task to do reinforcement learning or visualize the latent world state.}
    \label{fig:main-fig}
\end{figure}

Consider a representation learning pipeline shown in~\pref{fig:main-fig}. We are provided videos, or equivalently a sequence of observations, from agents navigating in the world. We make no assumption about the behavior of the agent in the video data. They can be trying to solve one task, many different tasks, or none at all. This video data is used to learn a model $\phi$ that maps any given observation to a vector representation. This representation is subsequently used to perform downstream RL --- defining a policy on top of the learned representation and only training the policy for the downstream task. 
We can also use this representation to define a dynamics model or a critique model.
The representation can also help visualize the agent state space or dynamics for the purpose of debugging.

A suitable representation for performing RL efficiently is aligned with the underlying dynamics of the world. Ideally, the representation captures the latent agent state, which contains information about the world relevant to decision-making while ignoring any noise in the observation. For example, in~\pref{fig:main-fig}, ignoring noise such as the motion of geese in the background is desirable if the task involves walking on the pavement. We distinguish between two types of noise: (1) temporally independent noise that occurs at each time step independent of the history, (2) temporally dependent noise, or exogenous noise, that can evolve temporally but in a manner independent of the agent's actions (such as the motion of geese in~\cref{fig:main-fig}). 

A range of approaches have been developed that provably recover the latent agent state from observations using trajectory data~\citep{misra2020kinematic,efroni2022ppe} which contains action labels. However, for many domains there is relatively little trajectory data that exists naturally, making it expensive to scale these learning approaches. In contrast, video data is more naturally available but these prior provable approaches do not work with video data. 
On the other hand, it is unknown whether approaches that empirically work with video data 
provably recover the latent representation and lead to efficient RL.
Motivated by this, we build a theoretical understanding of three such video-based representation learning approaches: \emph{autoencoder} which trains representations by reconstructing observations, \emph{forward modeling} which predicts future observations, and \emph{temporal contrastive} learning which trains a representation to determine if a pair of observations are causally related or not.

Our first theoretical result shows that in the absence of exogenous noise, forward modeling and temporal contrastive learning approaches both provably work. Further, they lead to efficient downstream RL that is strictly more sample-efficient than solving these tasks without any pre-training. Our second theoretical result establishes a lower bound showing that in the presence of exogenous noise, any compact and frozen representation that is pre-trained using video data cannot be used to do efficient downstream RL. In contrast, if the trajectory data was available, efficient pre-training would be possible. This establishes a statistical gap showing that video-based representation pre-training can be exponentially harder than trajectory-based representation pre-training.

We empirically test our theoretical results in three visual domains: GridWorld (a navigation domain), ViZDoom basic (a first-person 3D shooting game), and ViZDoom Defend The Center (a more challenging first-person 3D shooting game). We evaluate the aforementioned approaches along with ACRO~\citep{islam2022acro}, a representation pre-trained using trajectory data and designed to filter out exogenous noise. We observe that in accordance with our theory, both forward modeling and temporal contrastive learning succeed at RL when there is no exogenous noise. However, in the presence of exogenous noise, their performance degrades. Specifically, we find that temporal contrastive learning is especially prone to fail in the presence of exogenous noise, as it can rely exclusively on such noise to optimally minimize the contrastive loss. While we find that forward modeling is somewhat robust to exogenous noise, however, as exogenous noise increases, its performance quickly degrades as well. While any finite-sample guarantees for the autoencoding method remain an open question, empirically, we find that the performance of autoencoder-based representation learning is unpredictable. On the other hand, ACRO continues to perform well, highlighting a disadvantage of video pre-training. 
The code for all experiments is available as part of the Intrepid codebase at
~\url{https://github.com/microsoft/Intrepid}.
\looseness=-1

\section{Preliminaries and Overview}

We cover the theoretical settings and the dataset type in this section. We start by describing the notations in this work.

\paragraph{Mathematical Notation.} We use $[N]$ for $N \in \NN$ to define the set $\{1, 2, \cdots, N\}$. We use calligraphic notations to define sets (e.g., $\Ucal$, $\Scal$). For a given countable set $\Ucal$, we define $\Delta(\Ucal)$ as the space of all distributions over $\Ucal$. We denote the cardinality of a set $\Ucal$ by $|\Ucal|$.
We assume all sets to be countable. Finally, we denote $\poly\{\cdot\}$ to denote that something scales polynomially in listed quantities. 

\paragraph{Block MDPs.} We study episodic reinforcement learning in Block Markov Decision Processes (Block MDP)~\citep{du2019provably}. A Block MDP is defined by the tuple $(\Xcal, \Scal, \Acal, T, R, q, \mu, H)$ where $\Xcal$ is a set of observations that can be infinitely large, $\Scal$ is a finite set of \emph{latent} states, and $\Acal$ is a set of finite actions. The transition dynamics $T: \Scal \times \Acal \rightarrow \Delta(\Scal)$ define transitions in the latent state space. The reward function $R: \Scal \times \Acal \rightarrow [0, 1]$ assigns a reward $R(s, a)$ if action $a$ is taken in the latent state $s$. Observations are generated from states using the emission function $q: \Scal \rightarrow \Delta(\Xcal)$ that generates an observation $x \sim q(\cdot \mid s)$ when the agent is in latent state $s$. This emission process contains temporally independent noise but no exogenous noise. Finally, $\mu \in \Delta(\Scal)$ is the distribution over the initial latent state and $H$ is the horizon denoting the number of actions per episode. The agent interacts with a block MDP environment by repeatedly generating an episode $(x_1, a_1, r_1, x_2, a_2, r_2, \cdots, x_H, a_H, r_H)$ where $s_1 \sim \mu$ and for all $h \in [H]$ we have $x_h \sim q(\cdot \mid s_h)$, $r_h = R(s_h, a_h)$, and $s_{h+1} \sim T(\cdot \mid s_h, a_h)$, and all actions $\{a_h\}_{h=1}^H$ are taken by the agent. The agent never directly observes the latent states $(s_1, s_2, \cdots, s_H)$.

A key assumption in Block MDPs is that two different latent states cannot generate the same observation. This is called the \emph{disjoint emission property} and environments with this property are called rich observation environments, meaning the observation is \emph{sufficiently rich} to allow decoding of the latent state from it. Formally, this property allows us to define a decoder $\edo: \Xcal \rightarrow \Scal$ that maps an observation to the unique state that can generate it. The agent does not have access to $\edo$. If the agent had access to $\edo$, one could map each observation from an infinitely large space to the finite latent state space, which allows the use of classical finite RL methods~\citep{kearns2002near}.

\paragraph{Exogenous Block MDPs (Ex-Block MDP).} We also consider reinforcement learning in Exogenous Block MDPs (Ex-Block MDPs) which can be viewed as an extension of Block MDPs to include exogenous noise~\citep{efroni2022ppe}. An Ex-Block MDP is defined by $(\Xcal, \Scal, \Xi, \Acal, T, \Tex, R, q, H, \mu, \muex)$ where $\Xcal, \Scal, \Acal, T, R, H$ and $\mu$ have the same meaning and type as in Block MDPs. The additional quantities include $\Xi$ which is the space of exogenous noise. We use the notation $\xi \in \Xi$ to denote the exogenous noise. For the setting in \cref{fig:main-fig}, the exogenous noise variable $\xi$ captures variables such as the position of geese, the position of leaves on the trees in the background, and lighting conditions. We allow the space of exogenous noise to be infinitely large. The exogenous noise $\xi$ changes with time according to the transition function $\Tex: \Xi \rightarrow \Delta(\Xi)$. Note that unlike the agent state $s \in \Scal$, the exogenous noise $\xi \in \Xi$, evolves independently of the agent's action and does not influence the evolution of the agent's state. The emission process $q: \Scal \times \Xi \rightarrow \Delta(\Xcal)$ in Ex-Block MDP uses both the current agent state and exogenous noise, to generate the observation at a given time. For example, the image generated by the agent's camera contains information based on the agent's state (e.g., agent's position and orientation), along with exogenous noise (e.g., the position of geese). Finally, $\muex$ denotes the distribution over the initial exogenous noise $\xi$.

Similar to the Block MDP, we assume a disjoint emission property, i.e., we can decode the agent state and exogenous noise from a given observation. We use $\edo: \Xcal \rightarrow \Scal$ to denote the true decoder for the agent state ($s$), and $\exo: \Xcal \rightarrow \Xi$ to denote the true decoder for exogenous noise ($\xi$).

\paragraph{Provable RL.} We assume access to a policy class $\Pi = \{\pi: \Xcal \rightarrow \Acal\}$ where a policy $\pi \in \Pi$ allows the agent to take actions. For a given policy $\pi$, we use $\EE_\pi\sbr{\cdot}$ to denote expectation taken over an episode generated by sampling actions from $\pi$. We define the value of a policy $V(\pi) = \EE_\pi\sbr{\sum_{h=1}^H r_h}$ as the expected total reward or expected return. Our goal is to learn a near-optimal policy $\hat{\pi}$, i.e., $\sup_{\pi \in \Pi} V(\pi) - V(\hat{\pi}) \le \epsilon$ with probability at least $1-\delta$ for a given tolerance parameter $\epsilon > 0$ and failure probability $\delta \in (0, 1)$, using number of episodes that scale polynomially in $\nicefrac{1}{\epsilon}$, $\nicefrac{1}{\delta}$, and other relevant quantities. We will call such an algorithm as provably efficient. There exist several provably efficient RL approaches for solving Block MDPs~\citep{mhammedi2023representation,misra2020kinematic}, and Ex-Block MDPs~\citep{efroni2022ppe}. These approaches typically assume access to a decoder class $\Phi=\cbr{\phi:\Xcal \rightarrow [N]}$ and attempt to learn $\edo$ using it. These algorithms don't use any pre-training and instead directly interact with the environment and learn a near-optimal policy by using samples that scale with $\poly(S, A, H, \ln|\Phi|, \nicefrac{1}{\epsilon}, \nicefrac{1}{\delta})$. Crucially, the dependence on $\ln|\Phi|$ cannot be removed.
The decoder class $\Phi$ and all other function classes in this work are assumed to have bounded statistical complexity measures. For simplicity, we will assume that these function classes are finite and derive guarantees that scale logarithmically in their size (e.g., $\ln|\Pi|$).\footnote{Our theoretical analyses can be extended to other statistical complexity metrics such as Rademacher complexity.}

\paragraph{Representation Pre-training using Videos.} As mentioned earlier, existing algorithms for Block MDPs and Ex-Block MDPs require online interactions that scale with $\ln|\Phi|$. For real-world problems, the decoder class $\Phi$ can be a complex neural network and these approaches may, therefore, require lots of expensive online data. Offline RL approaches offer a substitute for expensive online interactions but require access to labeled episodes (with actions and rewards) that are not naturally available for many problems \citep[e.g.,][]{uehara2021representation}. Our goal instead is to pre-train the decoder $\phi$ using video data which is more easily available. 

\looseness=-1

\paragraph{Problem Statement.} Here, we state our problem formally. We are given access to two hyperparameters $\epsilon > 0$ and $\delta \in (0, 1)$ and a sufficiently large dataset of videos. We are also given a decoder class $\Phi = \{\phi: \Xcal \rightarrow [N]\}$ containing decoders that map an observation to one of the $N$ possible \emph{abstract states}.\footnote{We use a discrete representation to learn the latent state since Block MDP and Ex-Block MDP have a finite state space, and a discrete representation suffices for this purpose.} During the pre-training phase, we learn a decoder $\phi \in \Phi$ using the video data. We then freeze $\phi$ and use it to do RL in a downstream task. Instead of using any particular choice of algorithm for RL, we assume we are given a provably efficient tabular RL algorithm $\mathscr{A}$. We convert the observation-based RL problem to a tabular MDP problem by converting observations $x$ to its abstract state representation $\phi(x)$ using the frozen learned decoder $\phi$. The algorithm $\mathscr{A}$ uses $\phi(x)$ instead of $x$ and outputs an \emph{abstract policy} $\varphi: [N] \rightarrow \Acal$. We want that $\sup_{\pi \in \Pi} V(\pi) - V(\varphi \circ \phi) \le \epsilon$ with probability at least $1-\delta$, where $\varphi \circ \phi: x \mapsto \varphi(\phi(x))$ is our learned policy. Additionally, we want the number of online episodes needed in the downstream RL task to not scale with the size of the decoder class $\Phi$. This allows us to minimize online episodes while using more naturally available video data for pre-training.

\section{Representation Learning for RL using Video Dataset}
\label{sec:rep_alg}

We assume access to a dataset $\Dcal$ of $n$ videos $\Dcal = \{(x^{(i)}_1, x^{(i)}_2, \cdots, x^{(i)}_H)\}_{i=1}^n$ where $x^{(i)}_j$ is the $j^{th}$ observation (or frame) of the $i^{th}$ video. We are provided a decoder class $\Phi = \{\phi: \Xcal \rightarrow [N]\}$, and our goal is to learn a decoder $\phi \in \Phi$ that captures task-relevant information in the underlying state $\phi^\star(x)$ while throwing away as much exogenous noise as possible. Instead of proposing a new algorithm, we consider the following three classes of well-known representation learning methods that only assume access to video data. Our goal is to understand whether these methods provably learn useful representations. 

We define the loss function for these approaches that we use for our theoretical analysis. In practice, these methods either use continuous representations or use a Vector Quantized bottleneck to model discrete representations. We discuss these practical details at the end of this section.

\paragraph{Autoencoder.} This approach first maps a given observation $x$ to an abstract state $\phi(x)$ using a decoder $\phi \in \Phi$, and then uses it to reconstruct the observation $x$ with the aid of a reconstruction model $z \in \Zcal$ where $\Zcal = \cbr{z: [N] \rightarrow \Xcal}$ is our reconstruction model class. Formally, we optimize the following loss:
\begin{equation}\label{eqn:loss-auto}
   \lossauto(z, \phi) = \frac{1}{nH}\sum_{i=1}^n \sum_{h=1}^H \|z(\phi(x^{(i)}_h)) - x^{(i)}_h\|^2_2.
\end{equation}
In practice, autoencoders are typically implemented using a Vector Quantized bottleneck trained in a Variational AutoEncoder manner, which is called the VQ-VAE approach~\citep{oord2017vq}.

\paragraph{Forward Modeling.} This approach is similar to the autoencoder approach but instead of reconstructing the input observation, we reconstruct a future observation using a model class $\Fcal = \cbr{f: [N] \times [K] \rightarrow \Delta(\Xcal)}$ where $N$ is the output size of the decoder class $\Phi$ and $K \in \NN$ is a hyperparameter representing the forward time steps from the current observation. We collect a dataset of \emph{multistep transitions} $\Dcal_{\textrm{for}} = \{(x^{(i)}, k^{(i)}, x'^{(i)})\}_{i=1}^n$ 
sampled iid using the video dataset $\Dcal$ where the observation $x^{(i)}$ is sampled randomly from the $i^{th}$ video, $k^{(i)} \in [K]$, and $x'^{(i)}$ is the frame $k^{(i)}$-steps ahead of $x^{(i)}$ in the $i^{th}$ video. We distinguish between two types of sampling procedures, one where $k^{(i)}$ is always a fixed given value $k \in [K]$, and one where $k^{(i)} \sim \unf\rbr{[K]}$. Given the dataset $\Dcal_{\textrm{for}}$, we optimize the following loss:
\begin{equation}\label{eqn:loss-for}
    \lossfor(f, \phi) = \frac{1}{n} \sum_{i=1}^n \ln f\rbr{x'^{(i)} \mid \phi(x^{(i)}), k^{(i)}}.
\end{equation}

\paragraph{Temporal Contrastive Learning.} Finally, this approach trains the decoder $\phi$ to learn to separate a pair of temporally causal observations from a pair of temporally \emph{acausal} observations. We collect a dataset of $\Dcal_{\textrm{temp}} = \{(x^{(i)}, k^{(i)}, x'^{(i)}, z^{(i)})\}_{i=1}^{\lfloor n/2\rfloor}$ tuples using the multistep transitions dataset $\Dcal_{\textrm{for}}.$ 
We use 2 multistep transitions to create a single datapoint for $\Dcal_{\textrm{temp}}$ to keep the datapoints independent.
To create the $i^{th}$ datapoint for $\Dcal_{\textrm{temp}}$, we use the multistep transitions $(x^{(2i)}, k^{(2i)}, x'^{(2i)})$ and $(x^{(2i+1)}, k^{(2i+1)}, x'^{(2i+1)})$ and sample $z^{(i)} \sim \unf(\{0, 1\})$. If $z^{(i)}=1$, then our $i^{th}$ datapoint is a causal observation pair $(x^{(2i)}, k^{(2i)}, x'^{(2i)}, z^{(i)})$, otherwise, it is an acausal observation pair $(x^{(2i)}, k^{(2i)}, x^{(2i + 1)}, z^{(i)})$. Depending on how we sample $k$, we collect a different multistep transition dataset $\Dcal_{\textrm{for}}$, and accordingly a different contrastive learning dataset $\Dcal_{\textrm{temp}}$.
Given the dataset $\Dcal_{\textrm{temp}}$, we optimize the following loss using a regression model $g$ belonging to a model class $\Gcal = \{g: \Xcal \times [K] \times \Xcal \rightarrow [0, 1]\}$:
\begin{equation}\label{eqn:loss-temp}
   \losstemp(g, \phi) = \frac{1}{\lfloor n/2 \rfloor} \sum_{i=1}^{\lfloor n/2 \rfloor} \rbr{z^{(i)} - g(\phi(x^{(i)}), k^{(i)}, x'^{(i)})}^2.
\end{equation}

\paragraph{Practical Implementations.} We use the description of the aforementioned methods for theoretical analysis. However, their practical implementations differ in a few notable ways. Firstly, we either use a continuous vector representation $\phi: \Xcal \rightarrow \RR^d$ for modeling the decoder class $\Phi$, or apply a Vector Quantized (VQ) bottleneck on top of the vector representation to model a discrete-representation decoder. The VQ bottleneck is applied on the final output of $\phi(x)$ to force the representations into a discrete state space~\citep{oord2017vq}. The VQ-bottleneck is expressive yet forces the model to remove excess information~\citep{liu2021dvnc}. The gradients into $\phi$ are produced using the straight-through estimator, while an extra loss attracts the output of $\phi$ to the discrete embeddings. 

We also optimize all loss functions described above using the Adam optimizer with mini-batches (instead of directly minimizing the loss in~\pref{eqn:loss-auto}-\ref{eqn:loss-temp}). Finally, we use square loss instead of log loss for forward modeling and SimCLR-style batched loss for temporal contrastive learning. 

There are several variations of the above methods in prior literature. Our goal is to not comprehensively cover all such variants but instead describe the core approach and analyze it theoretically. 
We will empirically show that the findings of our theoretical analysis apply to our practical implementations.

\section{Is Video Based Representation Learning Provably Correct?}
\label{sec:theory}

In this section, we establish the theoretical foundation for when video-based representation learning is useful for downstream RL. We first present upper bounds for video representation pre-training via forward modeling and temporal contrastive learning in Block MDPs. We then present a lower-bound construction that shows that the video-based representation pre-training setting is exponentially harder than representation pre-training using trajectories where actions are available. This also explains why learning from video data is challenging in practice~\citep{Zhao2022video}.
For most of our results in this section, we will defer the proof to the Appendix and only provide a sketch here.

\subsection{Upper Bound in Block MDP Setting}

We start by stating our theoretical setting and our main assumptions. 

\paragraph{Theoretical Setting.}  We assume a Block MDP setting and access to a dataset $\Dcal = \cbr{(x^{(i)}_1, x^{(i)}_2, \cdots, x^{(i)}_H)}_{i=1}^n$ of $n$ independent and identically distributed (iid) videos sampled from data distribution $D$. We denote the probability of a video as $D(x_1, x_2, \cdots, x_H)$. We assume that $D$ is generated by a mixture of Markovian policies $\dpol$, i.e., the generative procedure for $D$ is to sample a policy $\pi \in \dpol$ with some probability and then generate an entire episode using it. We assume that observations encode time steps. This can be trivially accomplished by simply concatenating the time step information to the observation. We also assume that the video data has good state space coverage and that the data is collected by \emph{noise-free policies.}
\begin{assumption}[Requirements on Data Collection]\label{assum:data-collection} There exists an $\etamin > 0$ such that if $s$ is a state reachable at time step $h$ by some policy in $\Pi$, then $D\rbr{\edo(x_h)=s} \ge \etamin$. Further, we assume that every data collection policy $\pi \in \dpol$ is noise-free, i.e., $\pi(a \mid x_h) = \pi(a \mid \edo(x_h))$ for all $(a, x_h)$.
\end{assumption}
\noindent\textbf{Justification for~\pref{assum:data-collection}}
In practice, we expect this assumption to hold for tasks such as gaming, or software debugging, where video data is abundant and, therefore, can be expected to provide good coverage of the underlying state space. This assumption is far weaker than the assumption in batch RL which also requires actions and rewards to be labeled, which makes it more expensive to collect data that has good coverage~\citep{chen2019information}.
Further, unlike imitation learning from observations (ILO)~\citep{torabi2019recent}, we don't require that these videos provide demonstrations of the desired behavior. E.g., video streaming of games is extremely common on the internet, and one can get many hours of video data this way. However, this data wouldn't come with actions (which will be mouse or keyboard strokes) or reward labeling, and the game levels or tasks that these players are solving can be all different or even unrelated to the downstream tasks we want to solve. As such, the video data do not provide demonstrations of any one desired task. Further, as the video data is typically generated by humans, we can expect the data collection policies to be noise-free, as these policies are realized by humans who would not make decisions based on noise. E.g., a human player is unlikely to turn left due to the background motion of leaves that is unrelated to the game's control or objective.\looseness=-1 

We analyze the temporal contrastive learning and forward modeling approaches and derive upper bounds for these methods in Block MDPs. While autoencoder-based approaches sometimes do well in practice, it is an open question whether finite-sample bounds exist for them and we leave their theoretical analysis to future work and instead evaluate them empirically. We also do not use a VQ bottleneck for our theoretical analysis since our decoder class directly outputs discrete values. We can view the VQ bottleneck as a way to do the discrete encoding in practice. In addition to the decoder class $\Phi$, we assume a function class $\Fcal$ to model $f$ for forward modeling and $\Gcal$ to model $g$ for temporal contrastive learning.  We make a realizability assumption on these function classes.

\begin{assumption}[Realizability]\label{assum:realizability} There exists $f^\star \in \Fcal, g^\star \in \Gcal$ and $\phi_{\textrm{for}}, \phi_{\textrm{temp}} \in \Phi$ such that $f^\star(X' \mid \phi_{\textrm{for}}(x), k) = \PP_{\textrm{for}}(X' \mid x, k)$ and $g^\star(z \mid \phi_{\textrm{temp}}(x), k, x') = \PP_{\textrm{temp}}(z=1 \mid x, k, x')$ on the appropriate support, and where $\PP_{\textrm{for}}$ and $\PP_{\textrm{temp}}$ are respectively the Bayes classifier for the forward modeling and temporal contrastive learning methods.
\end{assumption}
\noindent\textbf{Justification for~\pref{assum:realizability}.} Realizability is a typical assumption made in theoretical analysis of RL algorithms~\citep{agarwal2020flambe}. Intuitively, the assumption states that the function classes are expressive enough to represent the Bayes classifier of their problem. In practice, this is usually not a concern as we will use expressive deep neural networks to model these function classes. We will empirically show the feasibility of this assumption in our experiments. 

Our statement of~\pref{assum:realizability} implicitly makes use of~\pref{assum:data-collection} which enables us to write the Bayes classifier for these problems as a function of just the state $\edo(x)$ rather then the actual observation $x$. This is formally proven in our analysis. This allows us to make a realizability assumption while still assuming a discrete decoder class $\Phi$.

Finally, we assume that our data distribution has the required information to separate the latent states. We state this assumption formally below and then show settings where this is true.
\begin{assumption}[Margin Assumption]\label{assum:margin} We assume that the margins $\formargin$ and $\tempmargin$ defined below:
\begin{align*}
    \formargin &= \inf_{s_1, s_2 \in \Scal, s_1 \ne s_2}\EE_{k}\sbr{\TV{\PP_{\textrm{for}}(X' \mid s_1, k) - \PP_{\textrm{for}}(X' \mid s_2, k)}}\\
    \tempmargin &= \inf_{s_1, s_2 \in \Scal, s_1 \ne s_2} \frac{1}{2}\EE_{k, s'}\sbr{\abr{\PP_{\textrm{temp}}(z=1 \mid s_1, k, s') - \PP_{\textrm{temp}}(z=1\mid s_2, k, s')}},
\end{align*}
are strictly positive, and where in the definition of $\tempmargin$, we sample $s'$ from the video data distribution and $k$ is sampled according to our data collection procedure.
\end{assumption}
\noindent\textbf{Justification for~\pref{assum:margin}.} This assumption states that we need margins ($\formargin$) for forward modeling and ($\tempmargin)$ for temporal contrastive learning. A common scenario where these assumptions are true is when for any pair of different states $s_1, s_2$, there is a third state $s'$ that is reachable from one but not the other. If the video data distribution $D$ supports all underlying transitions, then this immediately implies that $\TV{\PP_{\textrm{for}}(X' \mid s_1, k) - \PP_{\textrm{for}}(X' \mid s_2, k)} > 0$ which implies $\formargin > 0$. This scenario occurs in almost all navigation tasks. Specifically, it occurs in the three domains we experiment with. While it is less clear, under this assumption we also have $\tempmargin > 0$.

We now state our main result for forward modeling under \pref{assum:data-collection}-\ref{assum:margin}.

\begin{theorem}[Forward Modeling Result]\label{thm:main-upper-bound} Fix $\epsilon > 0$ and $\delta \in (0, 1)$ and let $\mathscr{A}$ be any provably efficient RL algorithm for tabular MDPs with sample complexity $\nsamp(S, A, H, \epsilon, \delta)$. If $n$ is $\poly\left\{S, H, \nicefrac{1}{\etamin}, \nicefrac{1}{\formargin}, \nicefrac{1}{\epsilon},
\ln(\nicefrac{1}{\delta}),\right.$ $\left. \ln|\Fcal|, \ln|\Phi|\right\}$ for a suitable polynomial, then forward modeling learns a decoder $\phihat: \Xcal \rightarrow [|\Scal|]$. Further, running $\mathscr{A}$ on the tabular MDP with $\nsamp(S, A, H, T, \epsilon/2, \delta/4)$ episodes returns a latent policy $\hat{\varphi}$. Then there exists a bijective mapping $\alpha: \Scal \rightarrow [|\Scal|]$ such that with probability at least $1-\delta$ we have:
\vspace*{-0.09cm}
\begin{equation*}
    \forall s \in \Scal, \qquad \PP_{x \sim q(\cdot \mid s)}\rbr{\phihat(x) = \alpha(s) \mid \edo(x)=s} \ge 1 - \frac{4S^3H^2}{\etamin^2\formargin}\sqrt{\frac{1}{n}\ln\rbr{\frac{|\Fcal|\cdot|\Phi|}{\delta}}},
\end{equation*}
and the learned observation-based policy $\hat{\varphi} \circ \phihat: x \mapsto \hat{\varphi}(\phihat(x))$ is $\epsilon$-optimal, i.e.,
    \begin{equation*}
        V(\pi^\star) - V(\hat{\varphi} \circ \phihat) \le \epsilon.
    \end{equation*}
Finally, the number of online episodes used in the downstream RL task is given by $\nsamp(S, A, H, \epsilon_\circ/2, \delta_\circ/4)$ and doesn't scale with the complexity of function classes $\Phi$ and $\Fcal$.
\end{theorem}

The result for temporal contrastive is identical to~\pref{thm:main-upper-bound} but instead of $\formargin$ we have $\tempmargin$ and instead of $\Fcal$ we have $\Gcal$. These upper bounds provide the desired result which shows that not only can we learn the right representation and near-optimal policy but also do so without online episodes scaling with $\ln|\Phi|$. 

\pref{thm:main-upper-bound} shows that in the absence of exogenous noise, we can expect forward modeling and temporal contrastive learning to learn the right representation under~\pref{assum:data-collection}-\ref{assum:margin}. The upper bounds for these two approaches are nearly identical except for the difference in margins $\tempmargin$ vs $\formargin$, and the difference in function class. Typically, we expect the function class $\Fcal$ for forward modeling to have a higher statistical complexity than temporal contrastive learning $\Gcal$, as the former is modeling the entire observation compared to the latter which is solving a binary classification task. This leaves the comparison between their margins. We prove a result that relates their margin.

\begin{theorem}[Margin Relation] For any Block MDP and $K \in \NN$ we have $\frac{\etamin^2}{4H^2}\formargink \le \tempmargink \le \formargink$.
\end{theorem}

The proof of a generalized statement of this theorem can be found in~\pref{app:margin-relations}. This result shows that the margin for forward modeling $\formargin$ is at least as large as the margin for temporal contrastive approach $\tempmargin$. In summary, this leads to a trade-off between forward modeling and temporal contrastive learning approaches in terms of performance, where the former has a function class with a higher statistical complexity but has a higher margin, compared to temporal contrastive approaches.

\subsection{Learning from Video is Exponentially Harder Than Learning from Trajectory Data}

As discussed before, when online reinforcement learning (RL) is possible, some algorithms can extract latent states from observation. These algorithms can be used to extract a decoder $\hat{\phi}$ that with high probability learns to map observations from different states to different values and whose output range $N$ is roughly of the order $|\Scal|$.

This begs the following question: 
\emph{Is it always possible to learn a compact state representation that is useful for control from video data?}
From this following lower bound result, we argue that this is not always the case.

\begin{theorem}[Lower Bound for Video]
\label{thm:exo_lower_bound}
Suppose $|\Scal|, |\Acal|, H \geq 2$. Then, for any $\varepsilon \in (0,1)$, any algorithm $\mathscr A_1$ that outputs a state decoder $\phi$ with $\phi_h: \Xcal \to [L], ~ L \leq 2^{\nicefrac{1}{4\varepsilon}-1}, ~ \forall h \in [H]$ given a video dataset $\Dcal$ sampled from some MDP and satisfies \cref{assum:data-collection}, and any online RL algorithm $\mathscr A_2$ uses that state decoder $\phi$ in its interaction with such an MDP (i.e., $\mathscr A_2$ only observes states through $\phi$) and output a policy $\widehat \pi$, there exists an MDP instance $M$ in a class of MDPs which satisfies \cref{assum:margin} and is PAC learnable with $\widetilde O(\poly(|\Scal|,|\Acal|,H,\nicefrac{1}{\varepsilon}))$ complexity, such that
\begin{align*}
V_M(\pi_M^\star) - V_M(\widehat \pi) > \varepsilon,
\end{align*}
regardless of the size of the video dataset $\Dcal$ for algorithm $\mathscr A_1$ and the number of episodes of interaction for algorithm $\mathscr A_2$.
\end{theorem}

The basic idea behind that hard instance construction is that, without the action information, it is impossible for the learning agent to distinguish between endogenous states and exogenous noise. For example, consider an image consisting of $N\times N$ identical mazes but where the agent controls just one maze. Other mazes contain other agents which are exogenous for our purpose. In the absence of actions, we cannot tell which maze is the one we are controlling and must memorize the configuration of all $N\times N$ mazes which grow exponentially with $N$. 
Another implication from that hard instance is -- if the margin condition (\cref{assum:margin}) is violated, the exponentially large state decoder is also required for the regular block MDP without exogenous noise; a detailed discussion can also be found in \cref{sec:lb_proof}.

\paragraph{Can we get efficient learning under additional assumptions?} Our lower bound suggests that one can in general not learn efficient and correct representations with just video data. However, it may be possible in some cases to do so with an additional assumption. We highlight one example here and defer a proper formal analysis to future work. One such scenario can be when the best-in-class error: $\inf_{f, \phi}\EE_{x,k}\sbr{\TV{f(X' \mid x, k) - \PP_{\textrm{for}}(X' \mid x, k)}}$ is small. A domain where this can happen is when the endogenous state is more predictive of $x'$ than any other $\ln |\Scal|$ bits of information in $x$. E.g., in a navigation domain, there can be many people in the background, but memorizing all of them can easily overwhelm the decoder's model capacity. Focusing solely on modeling the agent's state can simplify the task of predicting the future. 

Recently some approaches have also considered recovering \emph{latent actions} from video data using an encoder-decoder approach~\citep{ye2022become}. In general, the lower bound in~\pref{thm:exo_lower_bound} applies to these methods and they do not provably work in the hard instances with exogenous noise. For example, the latent actions can capture \emph{exogenous noise} instead of actions, if the former is more predictive of changes in the observations. However, in simpler cases such as 3D games, where the agent's action is typically most predictive of changes in observations, or in settings with no exogenous noise, one can expect these approaches to do well.\looseness=-1

\section{Experimental Results and Discussion}
\label{sec:exp}

We empirically evaluate the above video-based representation learning methods on three visual domains: (1) a 2D GridWorld domain, (2) ViZDoom basic environment, a first-person 3D shooting game, and (3) ViZDoom Defend The Center environment, which is a more challenging shooting game. Additional experimental details and results are mentioned in~\cref{appendix:exps}. Our main goal is to validate our theoretical findings from~\pref{sec:theory} and justify our assumptions. 

Specifically, we study the following questions in our experimental analysis:
\begin{itemize}
    \item How do autoencoding, forward modeling, and temporal contrastive learning methods perform in the absence of any noise, in the presence of IID noise, and the presence of exogenous noise?
    \item How does the performance of forward modeling vary with harder exogenous noise?
    \item How do these different video representation learning methods compare to trajectory-based learning methods such as ACRO?
\end{itemize}

\subsection{Experimental Details}

\paragraph{GridWorld.} We consider navigation in a $12\times 12$ Minigrid environment~\citep{MinigridMiniworld23}. 
The agent is represented as a red triangle and can take three actions: move forward, turn left, and turn right (\cref{fig:gridworld-reconstruction}). The goal of the agent is to reach a yellow key. The position of the agent and key randomizes each episode. The agent only observes an area around itself (as an agent-centric-view). Horizon $H=12$, and the agent gets a reward of +1.0 for reaching the goal and -0.01 in other cases. 

\paragraph{ViZDoom Basic.}
We use the basic ViZDoom environment \citep{wydmuch2018vizdoom, kempka2016vizdoom}, which is a first-person shooting game (\cref{fig:vizdoom-reconstruction}). The player needs to kill a monster to win. The episode ends when the monster is killed or after 300 steps.
The map of the environment is a rectangle with gray walls, ceiling, and floor. The player is spawned along the longer wall in the center. A red, circular monster is spawned randomly somewhere along the opposite wall. The player can take one of three actions at each time step (left, right, shoot). One hit is enough to kill the monster. The episode finishes when the monster is killed or on timeout. 
The reward scheme is as follows: +101 for shooting the enemy
, -1 per time step, and -5 for missed shots.

\paragraph{VizDoom Defend The Center.} We use an additional ViZDoom Defend the Center environment \citep{wydmuch2018vizdoom, kempka2016vizdoom}, which is a first-person shooting game (\cref{fig:vizdoom-dtc-reconstruction}). The player needs to kill a variety of monsters to score in the game. The episode ends when the monster is killed or after 500 steps. Results for this environment are shown in \cref{fig:vizdoom_dtc_exps} and \cref{fig:vizdoom-dtc-reconstruction} and further validate our findings from theory and experiments.

\paragraph{Exogenous Noise.} For all domains, the observation is an RGB image. We add exogenous noise to it by superimposing 10 generated diamonds of a particular size. The color and position of these diamonds are our exogenous state. At the start of each episode, we randomly generate these diamonds, after which they move in a deterministic path. We also test the setting in which there is exogenous noise in the reward. We compute a score based on just the exogenous noise and add it to the reward presented to the agent. However, the agent is still evaluated on the original reward.\looseness=-1

\paragraph{Model and Learning.} Our decoder class $\Phi$ is a convolutional neural network. We use a deconvolutional neural network to model $f$ and $h$. We experimented with both using a vector representation for $\phi$ and also using a VQ-bottleneck to discretize the embeddings. We use PPO to do downstream RL and keep $\phi$ frozen during the RL training. We also visualize the learned representations by training a decoder on them and fixing $\phi$ to reconstruct the input observations. We then look at the generated images to see what information from the observation is preserved by the representation.\looseness=-1

\paragraph{ACRO.} We also evaluate the learned representations against ACRO~\citep{islam2022acro} which uses trajectory data. This approach learns representation $\phi$ by predicting action given a pair of observations $\EE\sbr{\ln p(a_h \mid \phi(x_h), x_{h+k})}$. ACRO is designed to filter out exogenous noise as this information is not predictive of the action. Our goal is to show how much better do representations get if we had trajectory data instead of video data.\looseness=-1

\begin{figure}[!thb]
    \centering
    {\includegraphics[width=13cm]{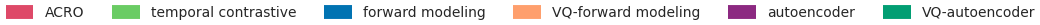}}
    \subfloat[\centering No Noise]{{\includegraphics[width=3.5cm]{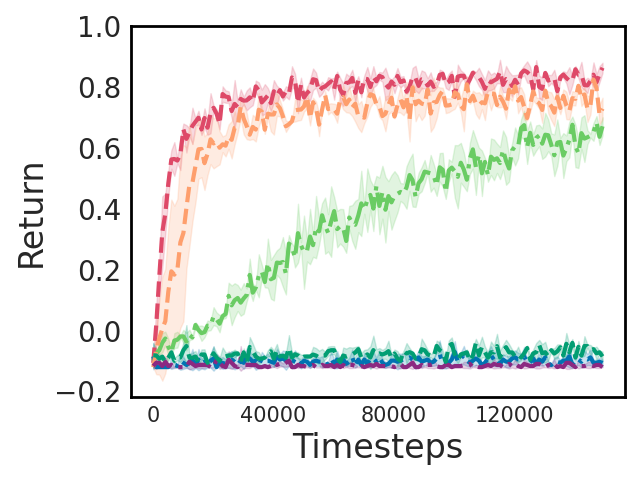} }}
    \subfloat[\centering Only Observation Noise]{{\includegraphics[width=3.5cm]{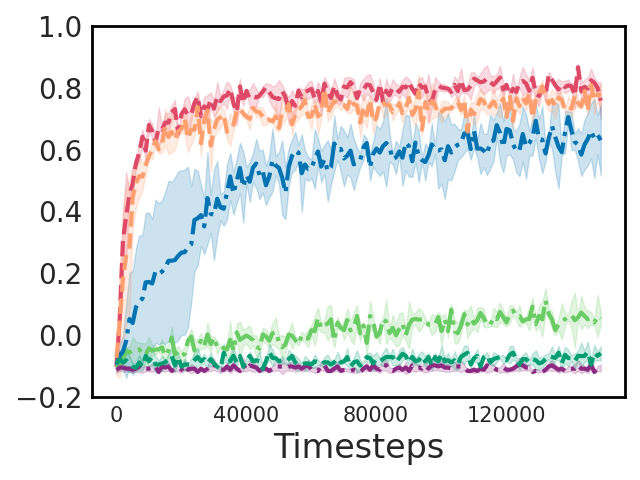} }}
    \subfloat[\centering Only Reward Noise]{{\includegraphics[width=3.5cm]{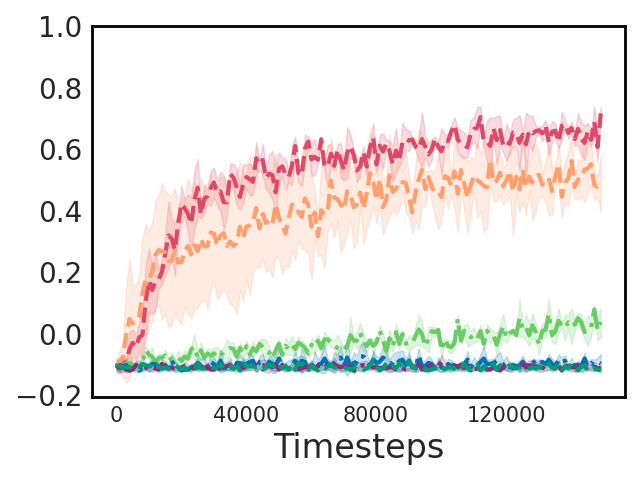} }}
    \subfloat[\centering Both]{{\includegraphics[width=3.5cm]{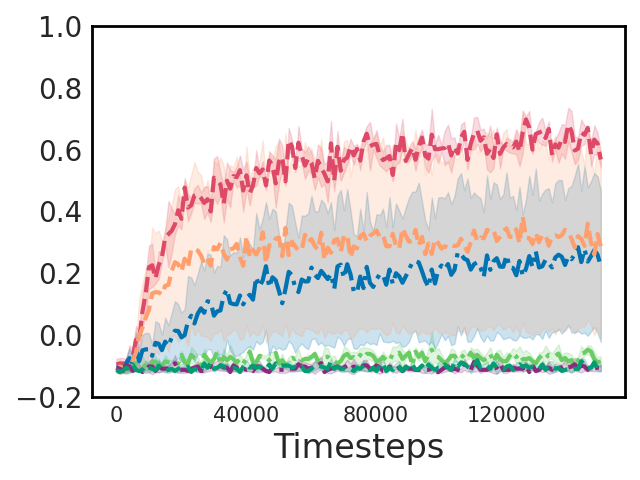} }}
    \caption{RL experiments in the GridWorld environment.} 
    \label{fig:gridworld_exps}
    \vspace{-3mm}
\end{figure}

\begin{figure}[!thb]
\centering
    {\includegraphics[width=2.7cm]{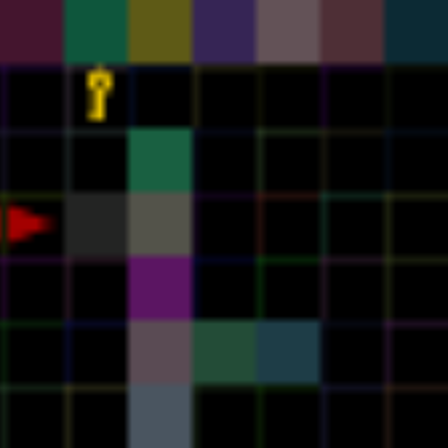} }
    {\includegraphics[width=2.7cm]{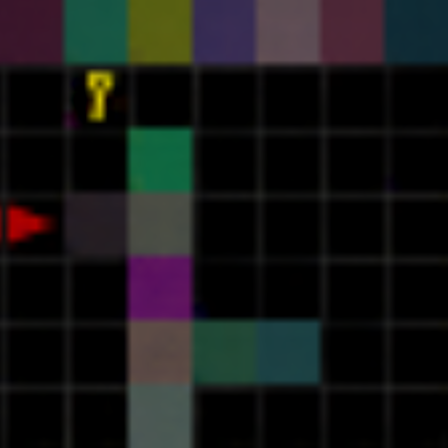} }
    {\includegraphics[width=2.7cm]{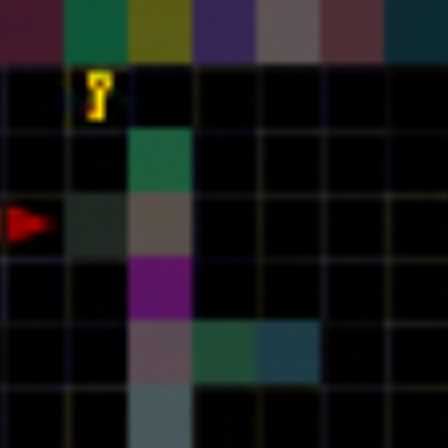} }
    {\includegraphics[width=2.7cm]{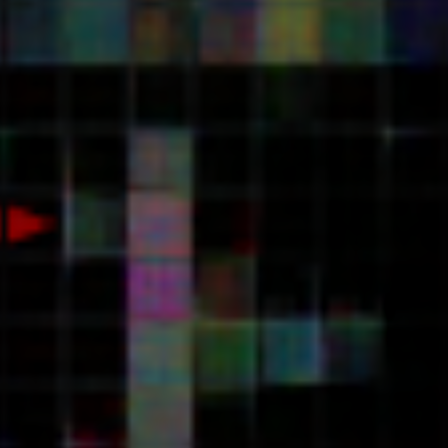} }\\[-0.1cm]
    \subfloat[\centering Original]{{\includegraphics[width=2.7cm]{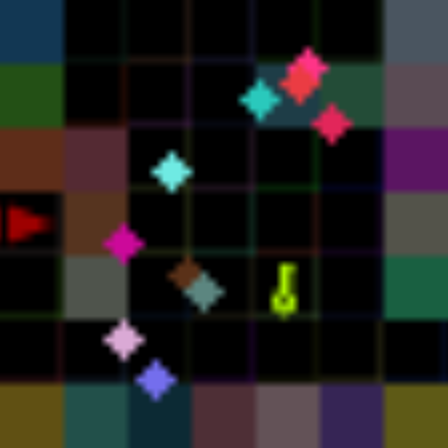} }}
    \hspace{0.02mm}
    \subfloat[\centering {Forward Model}]{{\includegraphics[width=2.7cm]{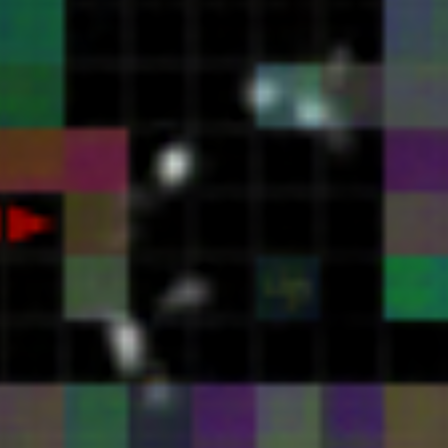} }}
    \hspace{0.02mm}
    \subfloat[\centering Autoencoder]{{\includegraphics[width=2.7cm]{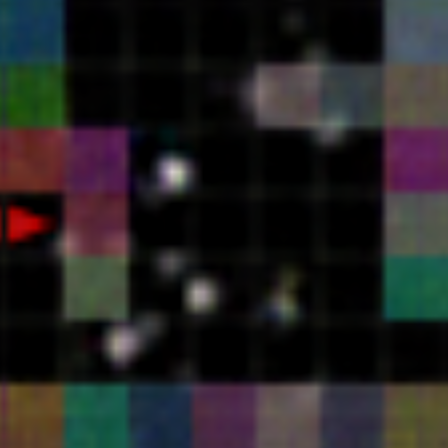} }}
    \hspace{0.02mm}
    \subfloat[\centering Contrastive]{{\includegraphics[width=2.7cm]{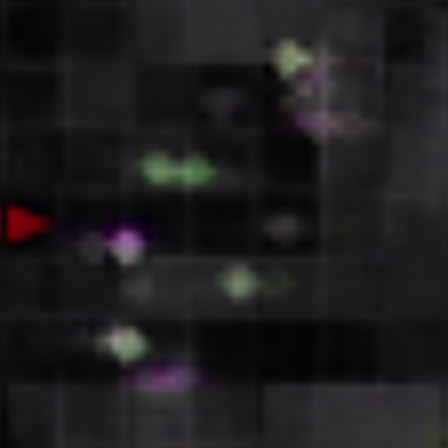} }}
    \caption{Decoded image reconstructions from different latent representation learning methods in the GridWorld environment. We train a decoder on top of frozen representations trained with the three video pre-training approaches. \textbf{Top row:} shows an example from the setting where there is no exogenous noise. \textbf{Bottom row:} shows an example with exogenous noise. 
    }
    \label{fig:gridworld-reconstruction}
\end{figure}

\subsection{Empirical Results and Discussion}

We present our main empirical results in~\cref{fig:gridworld_exps}, \cref{fig:vizdoom_exps}, and \cref{fig:vizdoom_dtc_exps} and discuss the results below.

\begin{figure}
    \centering
    {\includegraphics[width=10cm]{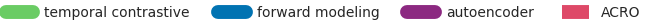}}
    \subfloat[\centering No Noise]{{\includegraphics[width=3.5cm]{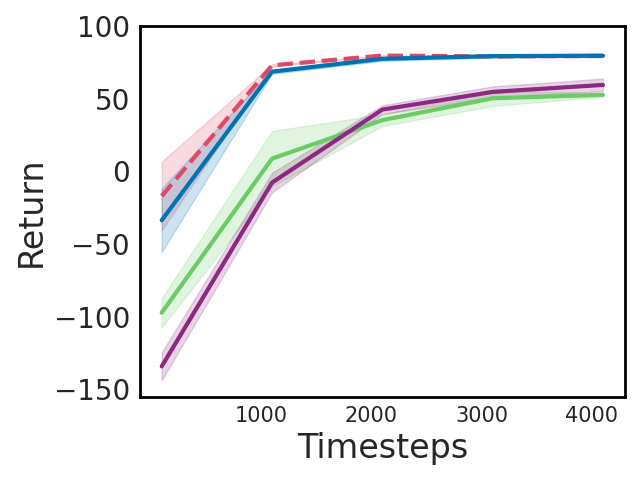}}}
    \subfloat[\centering Only Observation Noise]{{\includegraphics[width=3.5cm]{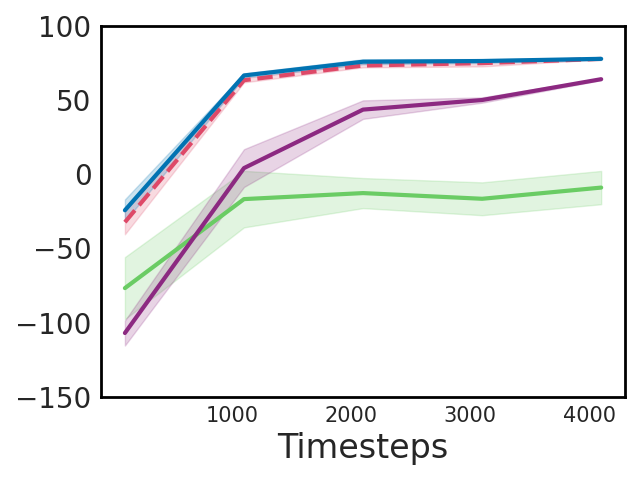} }}
    \subfloat[\centering Only Reward Noise]{{\includegraphics[width=3.5cm]{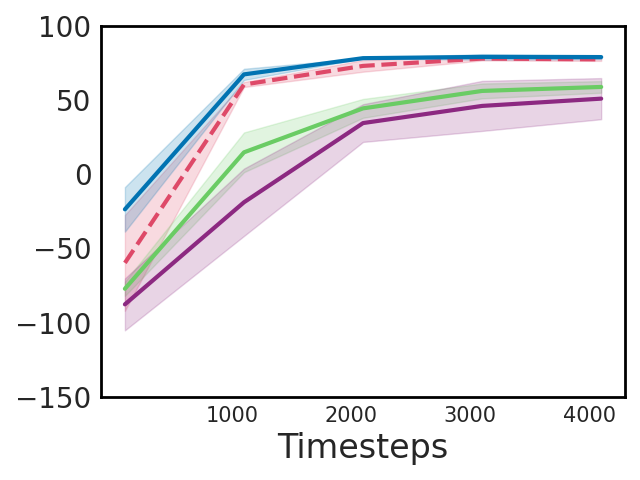} }}
    \subfloat[\centering Both]{{\includegraphics[width=3.5cm]{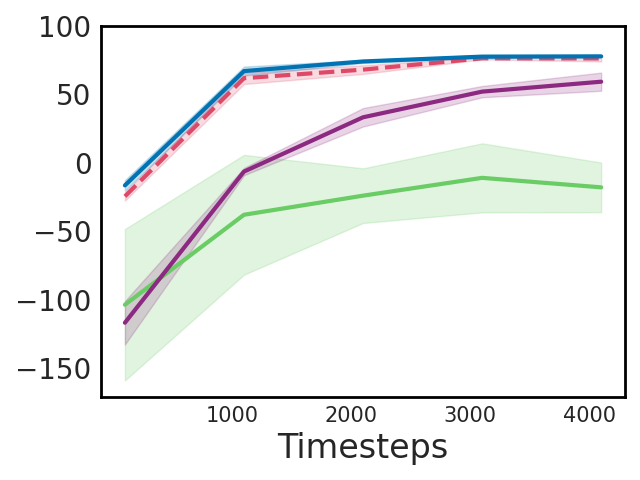} }}
    \caption{RL experiments using different latent representations for the ViZDoom environment.}
    \label{fig:vizdoom_exps}
\end{figure}

\begin{figure}[!thb]
\centering
    {\includegraphics[width=3.0cm]{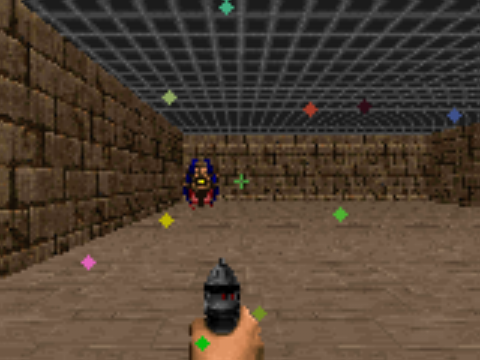} }
    {\includegraphics[width=3.0cm]{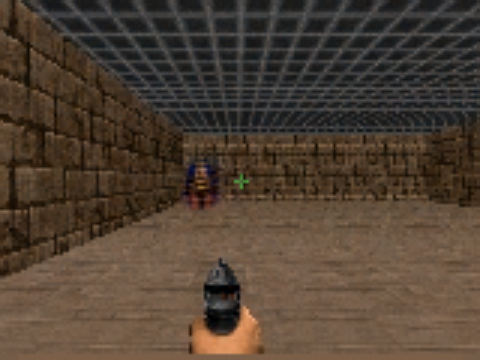} }
    {\includegraphics[width=3.0cm]{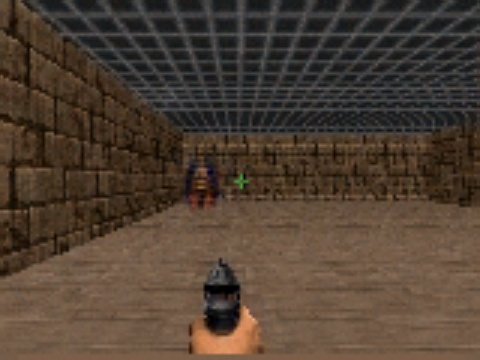} }
    {\includegraphics[width=3.0cm]{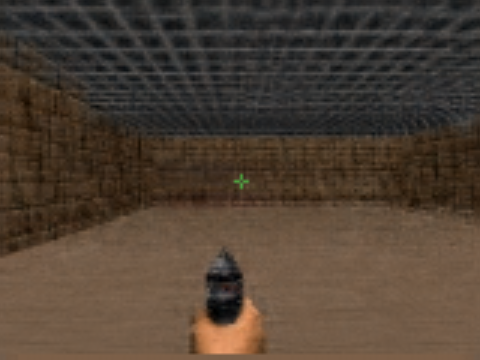} }
    \\
    \subfloat[\centering Original]{{\includegraphics[width=3.0cm]{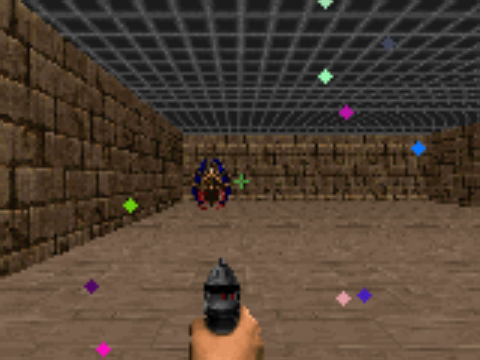} }}
    \hspace{0.05mm}
    \subfloat[\centering Forward Modeling]{{\includegraphics[width=3.0cm]{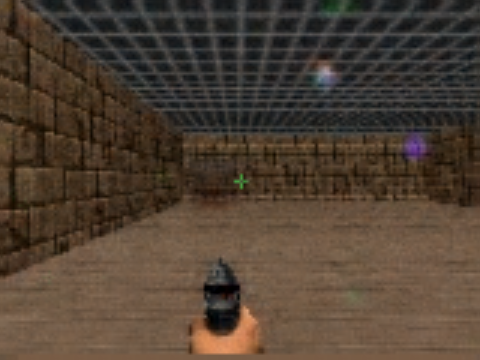} }}
    \hspace{0.05mm}
    \subfloat[\centering Autoencoder]{{\includegraphics[width=3.0cm]{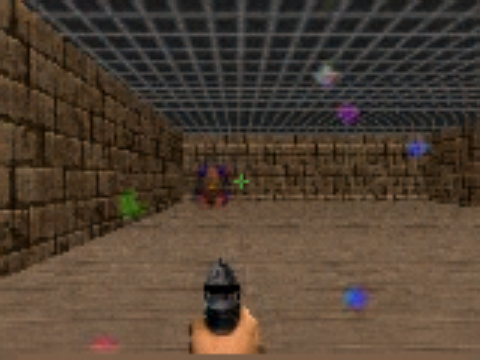} }}
    \hspace{0.05mm}
    \subfloat[\centering Contrastive]{{\includegraphics[width=3.0cm]{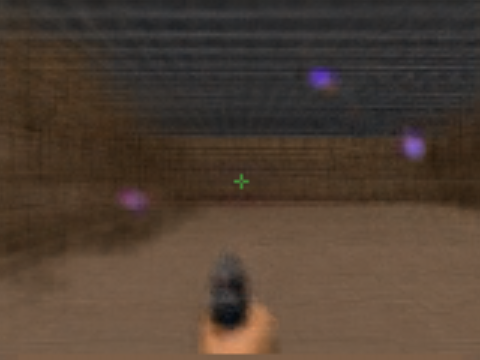} }}
    \caption{Decoded image reconstructions from different representation learning methods in the ViZDoom environment. We train a decoder on top of frozen representations trained with the three video pre-training approaches. Here, we show an example where we also have exogenous noise. 
    }
    \label{fig:vizdoom-reconstruction}
\end{figure}

\paragraph{Forward modeling and temporal contrastive both work when there is no exogenous noise.}  In accordance with~\pref{thm:main-upper-bound}, we observe that in the case of both GridWorld (\cref{fig:gridworld_exps}) and ViZDoom environments (\cref{fig:vizdoom_exps}, \cref{fig:vizdoom_dtc_exps}), these approaches learn a decoder $\phi$ that lead to success with reinforcement learning in the absence of any exogenous noise (both in the presence and absence of exogenous reward). For GridWorld, we find support for this result with VQ bottleneck during representation learning (\cref{fig:gridworld_exps}(a),(c) and \cref{fig:vizdoom_dtc_exps}(a), (c)), whereas for ViZDoom, we find support for this result even without the use of a VQ bottleneck (\cref{fig:vizdoom_exps}(a),(c)). These results are further supported via qualitative evaluation through image decoding from the learned latent representations (\cref{fig:gridworld-reconstruction}) which show that these representations can recover critical elements like walls. We find that autoencoder performs well in ViZDoom but not in gridworld, which aligns with a lack of any theoretical understanding of autoencoders.\looseness=-1

\begin{figure}
    \centering
    {\includegraphics[width=10cm]{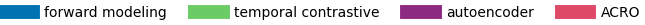}}
    \subfloat[\centering No Noise]{{\includegraphics[width=3.5cm]{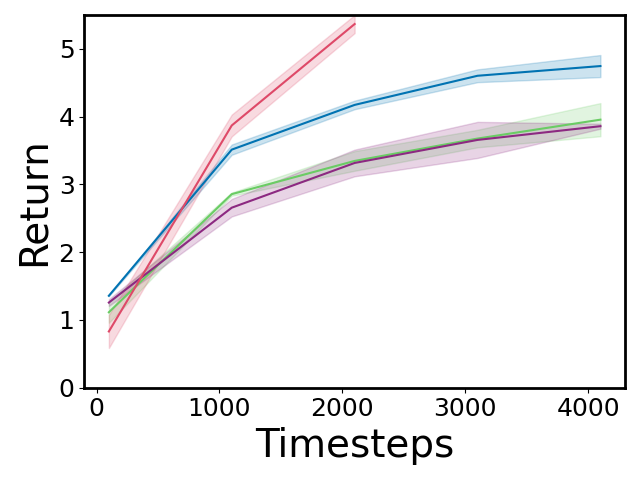}}}
    \subfloat[\centering Only Observation Noise]{{\includegraphics[width=3.5cm]{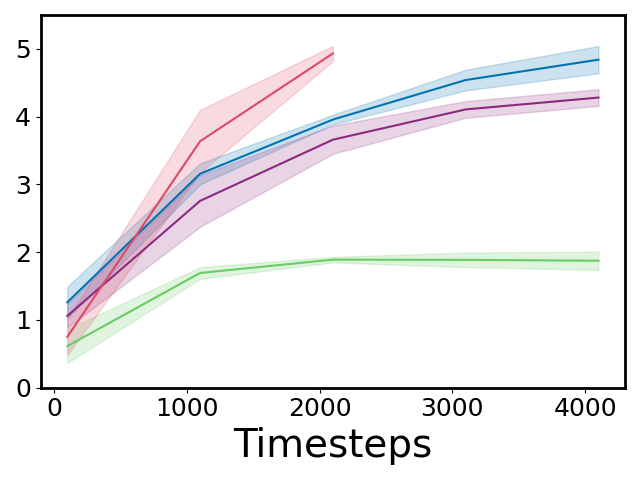} }}
    \subfloat[\centering Only Reward Noise]{{\includegraphics[width=3.5cm]{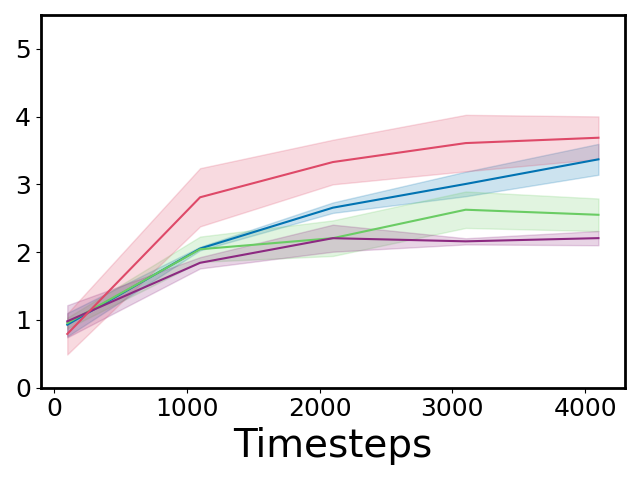} }}
    \subfloat[\centering Both]{{\includegraphics[width=3.5cm]{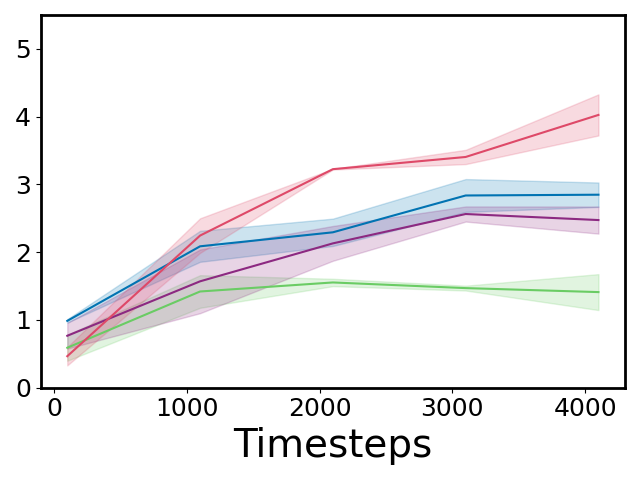} }}
    \vspace{2mm}
    \caption{RL experiments using different latent representations for the ViZDoom Defend the Center environment.}
    \label{fig:vizdoom_dtc_exps}
\end{figure}

\begin{figure}[!thb]
\centering
    \subfloat{{\includegraphics[width=3.5cm]{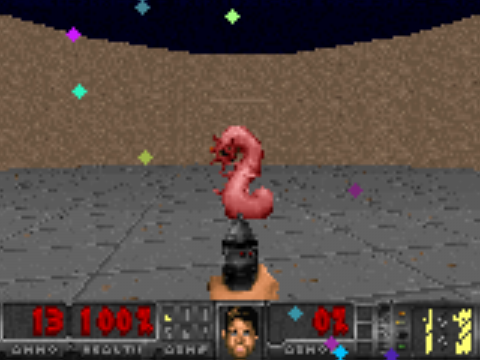} }}
    \subfloat{{\includegraphics[width=3.5cm]{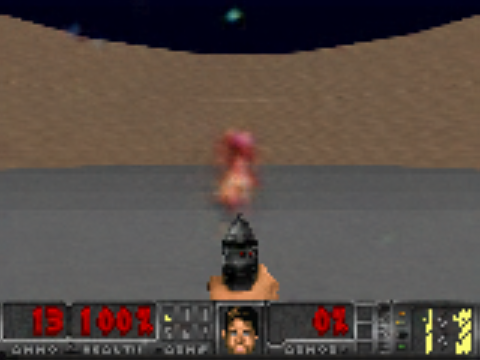} }}
    \subfloat{{\includegraphics[width=3.5cm]{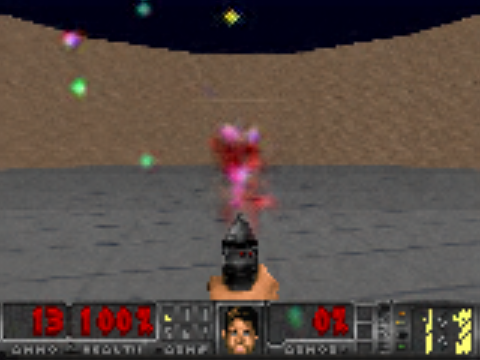} }}
    \subfloat{{\includegraphics[width=3.5cm]{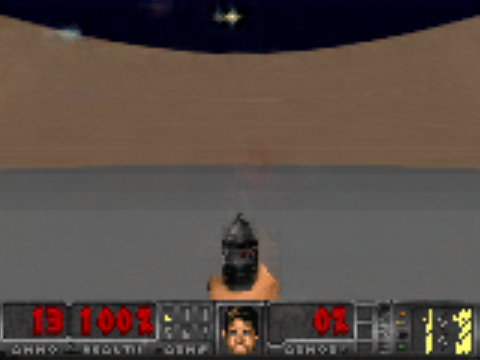} }}\\
    \subfloat{{\includegraphics[width=3.5cm]{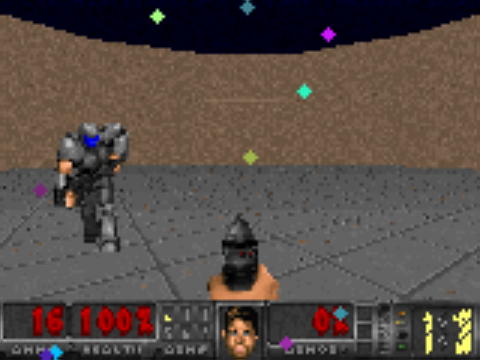} }}
    \subfloat{{\includegraphics[width=3.5cm]{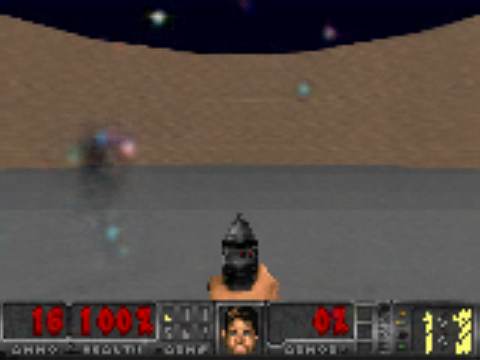} }}
    \subfloat{{\includegraphics[width=3.5cm]{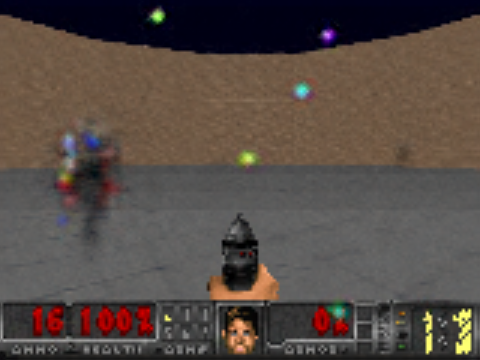} }}
    \subfloat{{\includegraphics[width=3.5cm]{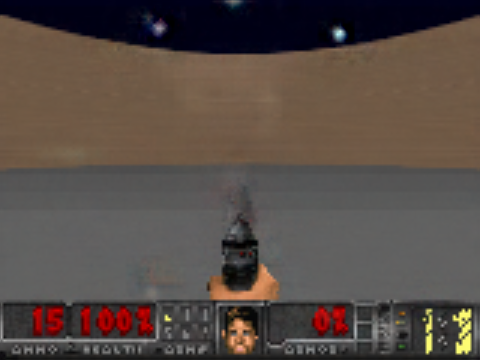} }}\\
    \subfloat[\centering Original]
    {{\includegraphics[width=3.5cm]{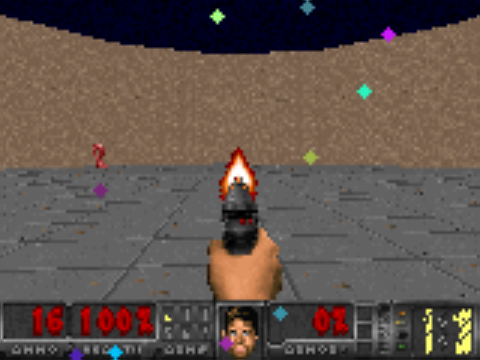} }}
    \subfloat[\centering Forward Modeling]{{\includegraphics[width=3.5cm]{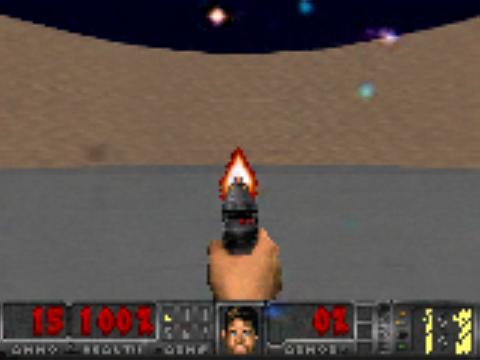} }}
    \subfloat[\centering Autoencoder]{{\includegraphics[width=3.5cm]{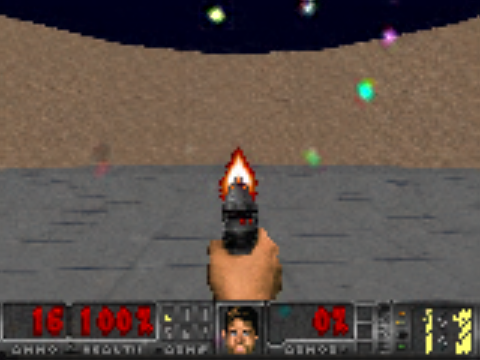} }}
    \subfloat[\centering Temporal Contrastive]{{\includegraphics[width=3.5cm]{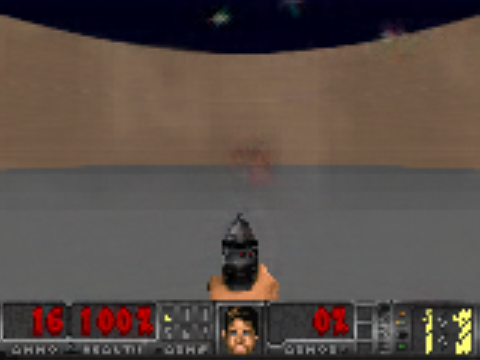} }}
    \caption{Decoded image reconstructions from different latent representation learning methods in the ViZDoom Defend the Center environment. We train a decoder on top of frozen representations trained with the three video pre-training approaches.
    }
    \label{fig:vizdoom-dtc-reconstruction}
    \vspace{-2mm}
\end{figure}

\paragraph{Performance with I.I.D. Noise.} 
We evaluate iid noise in the gridworld domain. We use the diamond-shaped exogenous noise that we used in~\cref{fig:gridworld_exps}, however, at each time step, we randomly sample the color and position of each diamond, independent of the agent's history.~\cref{fig:vizdoom_iid_noise}(c) shows the result for forward modeling and ~\cref{fig:vizdoom_iid_noise}(d) shows the same for ACRO. We also ablate the number of noisy diamonds. As expected, forward modeling and ACRO can learn a good policy while increase in the number of noisy diamonds (num noise var) only slightly decreases their performance. Additional results on performance with iid noise in the VizDoom basic environment are in ~\cref{appendix:exps}.

\begin{figure}
    \centering
    {\includegraphics[width=10cm]{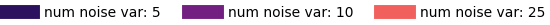}}\\
    \subfloat[\centering Forward Modeling]
    {\includegraphics[width=4.5cm]{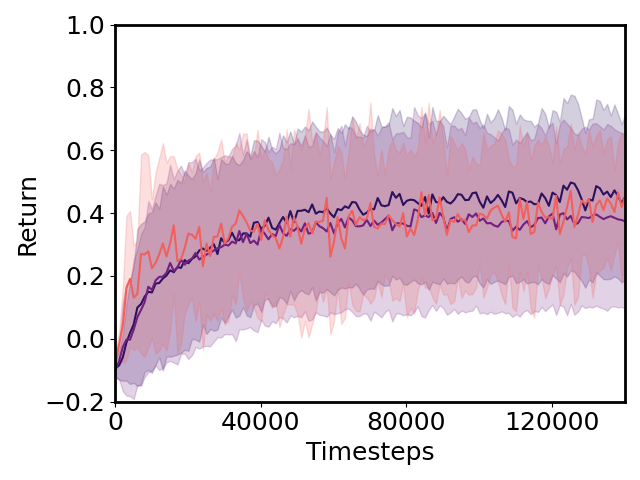} }
    \hspace{7mm}
    \subfloat[\centering ACRO]{{\includegraphics[width=4.5cm]{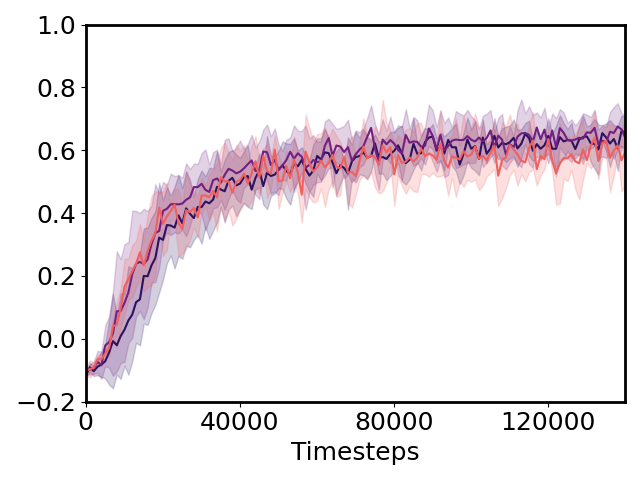} }}
    \caption{Experiments with iid noise for the Gridworld environment. `Num noise var' denotes the number of noisy diamonds constituting the exogenous noise.}
    \label{fig:vizdoom_iid_noise}
\end{figure}

\paragraph{Performance with exogenous noise.} 
We find that in the presence of exogenous noise for both GridWorld (\cref{fig:gridworld_exps}) and ViZDoom (\cref{fig:vizdoom_exps}, \cref{fig:vizdoom_dtc_exps}) environments, representations from forward modeling continue to succeed at RL albeit with a much lower performance in gridworld, whereas temporal contrastive representations completely fail. One hypothesis for the stark failure of temporal contrastive learning is that the agent can tell whether two observations are causal or not, by simply focusing on the noisy rhombus that moves in a predictive manner. Therefore, the contrastive learning loss can be reduced by focusing entirely on noise. Whereas, forward modeling is more robust as it needs to predict future observations, and the agent's state is more helpful for doing that than noise. This shows in the reconstructions (\cref{fig:gridworld-reconstruction}(d),~\cref{fig:vizdoom-reconstruction}(d), \cref{fig:vizdoom-dtc-reconstruction}(d)). As expected, the reconstructions for forward modeling continue to capture state-relevant information, whereas for temporal contrastive they focus on noise and miss relevant state information. We also formally prove that there exists an instance where forward modeling can recover the latent state for low-levels of exogenous noise, whereas temporal contrastive cannot do so for any level of exogenous noise. We defer readers to the instance construction and analysis in~\pref{app:temp-cont-fails-forward-wins}.

\looseness=-1

\begin{figure}   
    \centering
    {\includegraphics[width=0.7\textwidth]{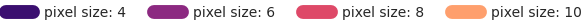}}\\
    \subfloat[\centering ACRO]{{\includegraphics[width=4.5cm]{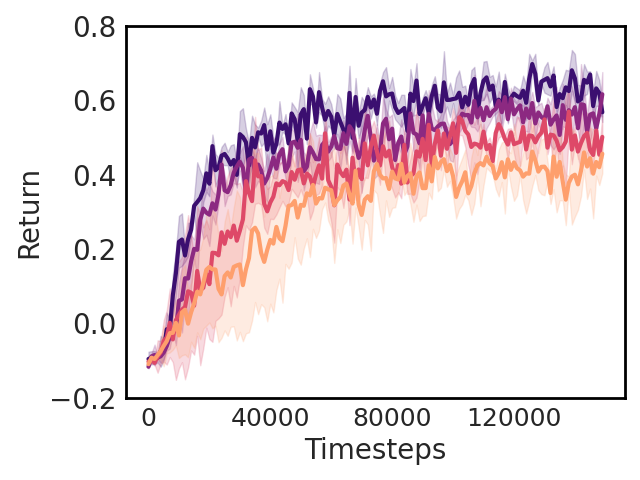} }}
    \subfloat[\centering Forward Modeling]{{\includegraphics[width=4.5cm]{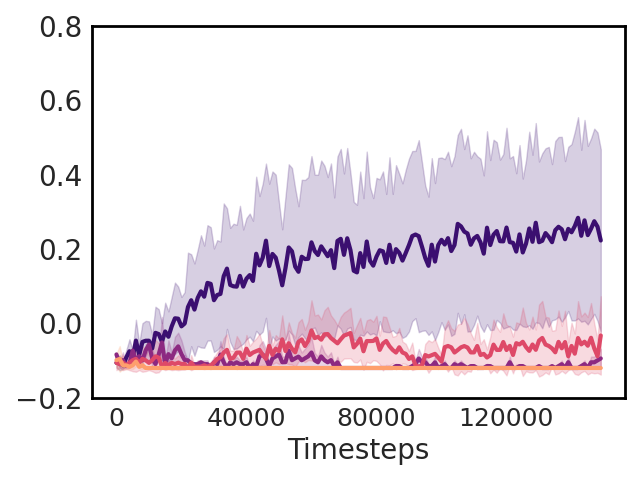} }}
    \caption{RL performance with varying size for exogenous noise in the GridWorld environment.}
    \label{fig:pixel_size}
\end{figure}

\paragraph{Harder Exogenous Noise.}  We increase the size of the exogenous noise variables (diamond shapes overlayed on the image) in the gridworld domain while keeping the number of exogenous variables fixed at 10 (\cref{fig:pixel_size}). We also increase the number of exogenous noise variables in the gridworld domain, while keeping their sizes fixed at 4 pixels (\cref{fig:gridworld_more_noise_main}). Both results show significant degradation in the performance of the forward modeling approach. 
This supports one of our main theoretical results that exogenous noise poses a challenge for video pre-training. Additional results for other video-based representation learning methods for this setting are shown in \cref{fig:gridworld_more_noise}.

\begin{figure}
    \centering
    {\includegraphics[width=13cm]{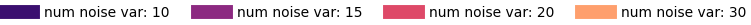}}\\
    \subfloat[\centering ACRO]{{\includegraphics[width=4.5cm]{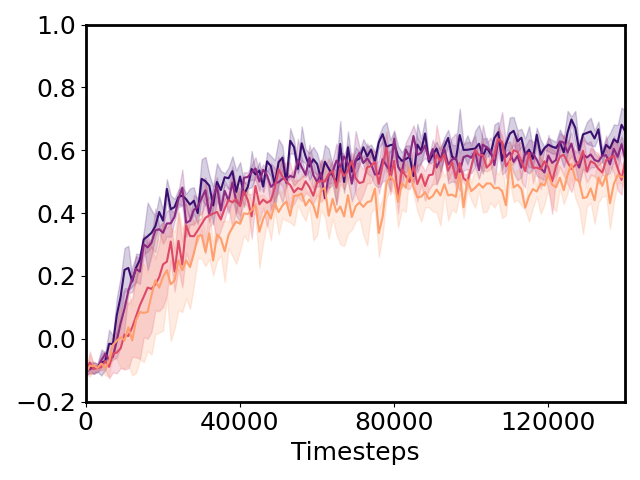} }}
    \subfloat[\centering Forward Modeling]{{\includegraphics[width=4.5cm]{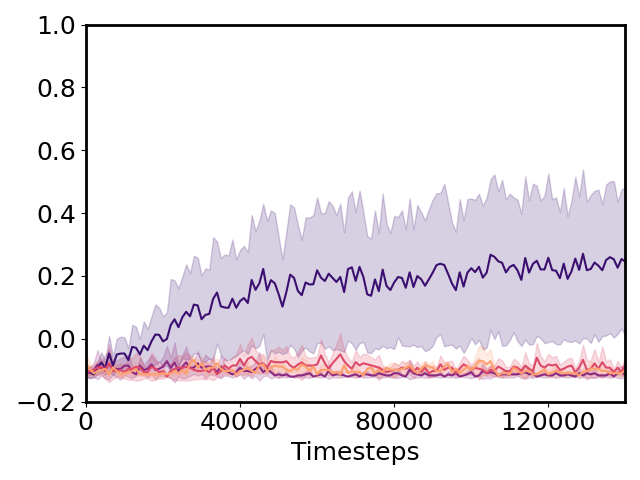} }}
    \caption{Gridworld experiments with exogenous noise of size 4 for a varying number of exogenous noise variables. Several representation learning algorithms using video data struggle to learn with a larger number of exogenous noise variables, whereas ACRO which uses trajectory data, still performs well.}
    \label{fig:gridworld_more_noise_main}
\end{figure}

\paragraph{Comparison with ACRO.} Finally, we draw a comparison between the performance of video-pretrained representation and ACRO which uses trajectory data. ACRO achieves the strongest performance across all tasks (\cref{fig:gridworld_exps}, \cref{fig:vizdoom_exps}, \cref{fig:vizdoom_dtc_exps}). This agrees with our theoretical finding (\pref{thm:exo_lower_bound}) that 
the sample complexity of video-based representations can be exponentially worse than trajectory-based representations in the presence of exogenous noise. Additionally, we also observe that as we increase the size of the exogenous noise elements in the observation space (\cref{fig:pixel_size}) or the number of exogenous noise variables (\cref{fig:gridworld_more_noise_main}), performance for forward-modeling based representation learning degrades more drastically compared to ACRO.

\section{Conclusion}
Learning models using unlabeled data such as audio, images, and video have had a transformative impact on deep learning research. An intriguing question is whether unlabeled data (videos) can have the same value for general reinforcement learning problems. Our work has sought to address this question both in theory and in practice. Theoretically, we have shown that certain representations trained on video can succeed when there is no exogenous noise. In the presence of exogenous noise, we present a lower bound establishing a separation between action-labeled pre-training for reinforcement learning and pre-training from unlabeled video data. We empirically validate our theoretical results on three visual domains. 

\paragraph{Acknowledgement.} We thank Sam Devlin, Ching-An Cheng, Andrey Kolobov, and Adith Swaminathan for useful discussions.

\newpage

\bibliography{arxiv2024}

\begin{thebibliography}{42}
\providecommand{\natexlab}[1]{#1}
\providecommand{\url}[1]{\texttt{#1}}
\expandafter\ifx\csname urlstyle\endcsname\relax
  \providecommand{\doi}[1]{doi: #1}\else
  \providecommand{\doi}{doi: \begingroup \urlstyle{rm}\Url}\fi

\bibitem[Agarwal et~al.(2020)Agarwal, Kakade, Krishnamurthy, and
  Sun]{agarwal2020flambe}
Alekh Agarwal, Sham Kakade, Akshay Krishnamurthy, and Wen Sun.
\newblock Flambe: Structural complexity and representation learning of low rank
  mdps.
\newblock \emph{Advances in neural information processing systems},
  33:\penalty0 20095--20107, 2020.

\bibitem[Aubret et~al.(2023)Aubret, Ernst, Teuli{\`e}re, and
  Triesch]{aubret2023time}
Arthur Aubret, Markus~R. Ernst, C{\'e}line Teuli{\`e}re, and Jochen Triesch.
\newblock Time to augment self-supervised visual representation learning.
\newblock In \emph{The Eleventh International Conference on Learning
  Representations}, 2023.
\newblock URL \url{https://openreview.net/forum?id=o8xdgmwCP8l}.

\bibitem[Baker et~al.(2022)Baker, Akkaya, Zhokov, Huizinga, Tang, Ecoffet,
  Houghton, Sampedro, and Clune]{baker2022vpt}
Bowen Baker, Ilge Akkaya, Peter Zhokov, Joost Huizinga, Jie Tang, Adrien
  Ecoffet, Brandon Houghton, Raul Sampedro, and Jeff Clune.
\newblock Video pretraining (vpt): Learning to act by watching unlabeled online
  videos.
\newblock \emph{Advances in Neural Information Processing Systems},
  35:\penalty0 24639--24654, 2022.

\bibitem[Bharadhwaj et~al.(2022)Bharadhwaj, Babaeizadeh, Erhan, and
  Levine]{bharadhwaj2022infopower}
Homanga Bharadhwaj, Mohammad Babaeizadeh, Dumitru Erhan, and Sergey Levine.
\newblock Information prioritization through empowerment in visual model-based
  rl.
\newblock \emph{arXiv preprint arXiv:2204.08585}, 2022.

\bibitem[Brown et~al.(2020)Brown, Mann, Ryder, Subbiah, Kaplan, Dhariwal,
  Neelakantan, Shyam, Sastry, Askell, et~al.]{brown2020language}
Tom Brown, Benjamin Mann, Nick Ryder, Melanie Subbiah, Jared~D Kaplan, Prafulla
  Dhariwal, Arvind Neelakantan, Pranav Shyam, Girish Sastry, Amanda Askell,
  et~al.
\newblock Language models are few-shot learners.
\newblock \emph{Advances in neural information processing systems},
  33:\penalty0 1877--1901, 2020.

\bibitem[Chen \& Jiang(2019)Chen and Jiang]{chen2019information}
Jinglin Chen and Nan Jiang.
\newblock Information-theoretic considerations in batch reinforcement learning.
\newblock In \emph{International Conference on Machine Learning}, pp.\
  1042--1051. PMLR, 2019.

\bibitem[Chevalier-Boisvert et~al.(2023)Chevalier-Boisvert, Dai, Towers,
  de~Lazcano, Willems, Lahlou, Pal, Castro, and Terry]{MinigridMiniworld23}
Maxime Chevalier-Boisvert, Bolun Dai, Mark Towers, Rodrigo de~Lazcano, Lucas
  Willems, Salem Lahlou, Suman Pal, Pablo~Samuel Castro, and Jordan Terry.
\newblock Minigrid \& miniworld: Modular \& customizable reinforcement learning
  environments for goal-oriented tasks.
\newblock \emph{CoRR}, abs/2306.13831, 2023.

\bibitem[Du et~al.(2019)Du, Krishnamurthy, Jiang, Agarwal, Dudik, and
  Langford]{du2019provably}
Simon Du, Akshay Krishnamurthy, Nan Jiang, Alekh Agarwal, Miroslav Dudik, and
  John Langford.
\newblock Provably efficient rl with rich observations via latent state
  decoding.
\newblock In \emph{International Conference on Machine Learning}, pp.\
  1665--1674. PMLR, 2019.

\bibitem[Efroni et~al.(2022)Efroni, Misra, Krishnamurthy, Agarwal, and
  Langford]{efroni2022ppe}
Yonathan Efroni, Dipendra Misra, Akshay Krishnamurthy, Alekh Agarwal, and John
  Langford.
\newblock Provably filtering exogenous distractors using multistep inverse
  dynamics.
\newblock In \emph{International Conference on Learning Representations}, 2022.
\newblock URL \url{https://openreview.net/forum?id=RQLLzMCefQu}.

\bibitem[Escontrela et~al.(2023)Escontrela, Adeniji, Yan, Jain, Peng, Goldberg,
  Lee, Hafner, and Abbeel]{escontrela2023video}
Alejandro Escontrela, Ademi Adeniji, Wilson Yan, Ajay Jain, Xue~Bin Peng, Ken
  Goldberg, Youngwoon Lee, Danijar Hafner, and Pieter Abbeel.
\newblock Video prediction models as rewards for reinforcement learning.
\newblock \emph{arXiv preprint arXiv:2305.14343}, 2023.

\bibitem[Geer(2000)]{geer2000empirical}
Sara~A Geer.
\newblock \emph{Empirical Processes in M-estimation}, volume~6.
\newblock Cambridge university press, 2000.

\bibitem[Goo \& Niekum(2019)Goo and Niekum]{goo2019one}
Wonjoon Goo and Scott Niekum.
\newblock One-shot learning of multi-step tasks from observation via activity
  localization in auxiliary video.
\newblock In \emph{2019 international conference on robotics and automation
  (ICRA)}, pp.\  7755--7761. IEEE, 2019.

\bibitem[Guo et~al.(2022)Guo, Thakoor, P{\^\i}slar, Avila~Pires, Altch{\'e},
  Tallec, Saade, Calandriello, Grill, Tang, et~al.]{guo2022byol}
Zhaohan Guo, Shantanu Thakoor, Miruna P{\^\i}slar, Bernardo Avila~Pires,
  Florent Altch{\'e}, Corentin Tallec, Alaa Saade, Daniele Calandriello,
  Jean-Bastien Grill, Yunhao Tang, et~al.
\newblock Byol-explore: Exploration by bootstrapped prediction.
\newblock \emph{Advances in neural information processing systems},
  35:\penalty0 31855--31870, 2022.

\bibitem[Hafner et~al.(2019)Hafner, Lillicrap, Ba, and
  Norouzi]{hafner2019dream}
Danijar Hafner, Timothy Lillicrap, Jimmy Ba, and Mohammad Norouzi.
\newblock Dream to control: Learning behaviors by latent imagination.
\newblock \emph{arXiv preprint arXiv:1912.01603}, 2019.
\newblock URL \url{https://arxiv.org/pdf/1912.01603.pdf}.

\bibitem[Hafner et~al.(2023)Hafner, Pasukonis, Ba, and
  Lillicrap]{hafner2023dreamerv3}
Danijar Hafner, Jurgis Pasukonis, Jimmy Ba, and Timothy Lillicrap.
\newblock Mastering diverse domains through world models.
\newblock \emph{arXiv preprint arXiv:2301.04104}, 2023.

\bibitem[Islam et~al.(2022)Islam, Tomar, Lamb, Efroni, Zang, Didolkar, Misra,
  Li, van Seijen, Combes, et~al.]{islam2022acro}
Riashat Islam, Manan Tomar, Alex Lamb, Yonathan Efroni, Hongyu Zang, Aniket
  Didolkar, Dipendra Misra, Xin Li, Harm van Seijen, Remi Tachet~des Combes,
  et~al.
\newblock Agent-controller representations: Principled offline rl with rich
  exogenous information.
\newblock \emph{arXiv preprint arXiv:2211.00164}, 2022.

\bibitem[Ke et~al.(2019)Ke, Singh, Touati, Goyal, Bengio, Parikh, and
  Batra]{ke2019learning}
Nan~Rosemary Ke, Amanpreet Singh, Ahmed Touati, Anirudh Goyal, Yoshua Bengio,
  Devi Parikh, and Dhruv Batra.
\newblock Learning dynamics model in reinforcement learning by incorporating
  the long term future, 2019.

\bibitem[Kearns \& Singh(2002)Kearns and Singh]{kearns2002near}
Michael Kearns and Satinder Singh.
\newblock Near-optimal reinforcement learning in polynomial time.
\newblock \emph{Machine learning}, 49:\penalty0 209--232, 2002.

\bibitem[Kempka et~al.(2016)Kempka, Wydmuch, Runc, Toczek, and
  Ja{\'s}kowski]{kempka2016vizdoom}
Micha{\l} Kempka, Marek Wydmuch, Grzegorz Runc, Jakub Toczek, and Wojciech
  Ja{\'s}kowski.
\newblock Vizdoom: A doom-based ai research platform for visual reinforcement
  learning.
\newblock In \emph{2016 IEEE conference on computational intelligence and games
  (CIG)}, pp.\  1--8. IEEE, 2016.

\bibitem[Lamb et~al.(2022)Lamb, Islam, Efroni, Didolkar, Misra, Foster, Molu,
  Chari, Krishnamurthy, and Langford]{lamb2022guaranteed}
Alex Lamb, Riashat Islam, Yonathan Efroni, Aniket Didolkar, Dipendra Misra,
  Dylan Foster, Lekan Molu, Rajan Chari, Akshay Krishnamurthy, and John
  Langford.
\newblock Guaranteed discovery of controllable latent states with multi-step
  inverse models.
\newblock \emph{arXiv preprint arXiv:2207.08229}, 2022.
\newblock URL \url{https://arxiv.org/pdf/2207.08229.pdf}.

\bibitem[Lin et~al.(2021)Lin, Bertasius, Wang, Chang, Parikh, and
  Torresani]{lin2021vx2text}
Xudong Lin, Gedas Bertasius, Jue Wang, Shih-Fu Chang, Devi Parikh, and Lorenzo
  Torresani.
\newblock Vx2text: End-to-end learning of video-based text generation from
  multimodal inputs.
\newblock In \emph{Proceedings of the IEEE/CVF Conference on Computer Vision
  and Pattern Recognition}, pp.\  7005--7015, 2021.

\bibitem[Liu et~al.(2021)Liu, Lamb, Kawaguchi, ALIAS PARTH~GOYAL, Sun, Mozer,
  and Bengio]{liu2021dvnc}
Dianbo Liu, Alex~M Lamb, Kenji Kawaguchi, Anirudh~Goyal ALIAS PARTH~GOYAL, Chen
  Sun, Michael~C Mozer, and Yoshua Bengio.
\newblock Discrete-valued neural communication.
\newblock In M.~Ranzato, A.~Beygelzimer, Y.~Dauphin, P.S. Liang, and J.~Wortman
  Vaughan (eds.), \emph{Advances in Neural Information Processing Systems},
  volume~34, pp.\  2109--2121. Curran Associates, Inc., 2021.
\newblock URL
  \url{https://proceedings.neurips.cc/paper_files/paper/2021/file/10907813b97e249163587e6246612e21-Paper.pdf}.

\bibitem[Liu et~al.(2019)Liu, Ott, Goyal, Du, Joshi, Chen, Levy, Lewis,
  Zettlemoyer, and Stoyanov]{liu2019roberta}
Yinhan Liu, Myle Ott, Naman Goyal, Jingfei Du, Mandar Joshi, Danqi Chen, Omer
  Levy, Mike Lewis, Luke Zettlemoyer, and Veselin Stoyanov.
\newblock Roberta: A robustly optimized bert pretraining approach.
\newblock \emph{arXiv preprint arXiv:1907.11692}, 2019.

\bibitem[Mhammedi et~al.(2023)Mhammedi, Foster, and
  Rakhlin]{mhammedi2023representation}
Zakaria Mhammedi, Dylan~J Foster, and Alexander Rakhlin.
\newblock Representation learning with multi-step inverse kinematics: An
  efficient and optimal approach to rich-observation rl.
\newblock \emph{arXiv preprint arXiv:2304.05889}, 2023.

\bibitem[Micheli et~al.(2023)Micheli, Alonso, and Fleuret]{micheli2023iris}
Vincent Micheli, Eloi Alonso, and Fran{\c{c}}ois Fleuret.
\newblock Transformers are sample-efficient world models.
\newblock In \emph{The Eleventh International Conference on Learning
  Representations}, 2023.
\newblock URL \url{https://openreview.net/forum?id=vhFu1Acb0xb}.

\bibitem[Misra et~al.(2020)Misra, Henaff, Krishnamurthy, and
  Langford]{misra2020kinematic}
Dipendra Misra, Mikael Henaff, Akshay Krishnamurthy, and John Langford.
\newblock Kinematic state abstraction and provably efficient rich-observation
  reinforcement learning.
\newblock In \emph{International conference on machine learning}, pp.\
  6961--6971. PMLR, 2020.

\bibitem[Nachum et~al.(2018)Nachum, Gu, Lee, and Levine]{nachum2018near}
Ofir Nachum, Shixiang Gu, Honglak Lee, and Sergey Levine.
\newblock Near-optimal representation learning for hierarchical reinforcement
  learning.
\newblock In \emph{International Conference on Learning Representations}, 2018.

\bibitem[Oord et~al.(2017)Oord, Vinyals, and Kavukcuoglu]{oord2017vq}
Aaron van~den Oord, Oriol Vinyals, and Koray Kavukcuoglu.
\newblock Neural discrete representation learning.
\newblock \emph{arXiv preprint arXiv:1711.00937}, 2017.

\bibitem[Parthasarathy et~al.(2022)Parthasarathy, Eslami, Carreira, and
  H{\'e}naff]{parthasarathy2022self}
Nikhil Parthasarathy, SM~Eslami, Jo{\~a}o Carreira, and Olivier~J H{\'e}naff.
\newblock Self-supervised video pretraining yields strong image
  representations.
\newblock \emph{arXiv preprint arXiv:2210.06433}, 2022.

\bibitem[Radford et~al.(2021)Radford, Kim, Hallacy, Ramesh, Goh, Agarwal,
  Sastry, Askell, Mishkin, Clark, et~al.]{radford2021learning}
Alec Radford, Jong~Wook Kim, Chris Hallacy, Aditya Ramesh, Gabriel Goh,
  Sandhini Agarwal, Girish Sastry, Amanda Askell, Pamela Mishkin, Jack Clark,
  et~al.
\newblock Learning transferable visual models from natural language
  supervision.
\newblock In \emph{International conference on machine learning}, pp.\
  8748--8763. PMLR, 2021.

\bibitem[Sikchi et~al.(2022)Sikchi, Saran, Goo, and Niekum]{sikchi2022ranking}
Harshit Sikchi, Akanksha Saran, Wonjoon Goo, and Scott Niekum.
\newblock A ranking game for imitation learning.
\newblock \emph{arXiv preprint arXiv:2202.03481}, 2022.

\bibitem[Sobal et~al.(2022)Sobal, SV, Jalagam, Carion, Cho, and
  LeCun]{sobal2022jepa}
Vlad Sobal, Jyothir SV, Siddhartha Jalagam, Nicolas Carion, Kyunghyun Cho, and
  Yann LeCun.
\newblock Joint embedding predictive architectures focus on slow features.
\newblock \emph{arXiv preprint arXiv:2211.10831}, 2022.

\bibitem[Srivastava et~al.(2015)Srivastava, Mansimov, and
  Salakhudinov]{srivastava2015lstm}
Nitish Srivastava, Elman Mansimov, and Ruslan Salakhudinov.
\newblock Unsupervised learning of video representations using lstms.
\newblock In Francis Bach and David Blei (eds.), \emph{Proceedings of the 32nd
  International Conference on Machine Learning}, volume~37 of \emph{Proceedings
  of Machine Learning Research}, pp.\  843--852, Lille, France, 07--09 Jul
  2015. PMLR.
\newblock URL \url{https://proceedings.mlr.press/v37/srivastava15.html}.

\bibitem[Sutton \& Barto(2018)Sutton and Barto]{sutton2018reinforcement}
Richard~S Sutton and Andrew~G Barto.
\newblock \emph{Reinforcement learning: An introduction}.
\newblock MIT press, 2018.

\bibitem[Tang et~al.(2023)Tang, Guo, Richemond, Avila~Pires, Chandak, Munos,
  Rowland, Gheshlaghi~Azar, Le~Lan, Lyle, Gy\"{o}rgy, Thakoor, Dabney, Piot,
  Calandriello, and Valko]{tang2023spr}
Yunhao Tang, Zhaohan~Daniel Guo, Pierre~Harvey Richemond, Bernardo Avila~Pires,
  Yash Chandak, Remi Munos, Mark Rowland, Mohammad Gheshlaghi~Azar, Charline
  Le~Lan, Clare Lyle, Andr\'{a}s Gy\"{o}rgy, Shantanu Thakoor, Will Dabney,
  Bilal Piot, Daniele Calandriello, and Michal Valko.
\newblock Understanding self-predictive learning for reinforcement learning.
\newblock In \emph{Proceedings of the 40th International Conference on Machine
  Learning}, volume 202 of \emph{Proceedings of Machine Learning Research},
  pp.\  33632--33656. PMLR, 23--29 Jul 2023.
\newblock URL \url{https://proceedings.mlr.press/v202/tang23d.html}.

\bibitem[Torabi et~al.(2019)Torabi, Warnell, and Stone]{torabi2019recent}
Faraz Torabi, Garrett Warnell, and Peter Stone.
\newblock Recent advances in imitation learning from observation.
\newblock \emph{arXiv preprint arXiv:1905.13566}, 2019.
\newblock URL \url{https://arxiv.org/pdf/1905.13566.pdf}.

\bibitem[Uehara et~al.(2021)Uehara, Zhang, and Sun]{uehara2021representation}
Masatoshi Uehara, Xuezhou Zhang, and Wen Sun.
\newblock Representation learning for online and offline rl in low-rank mdps.
\newblock \emph{arXiv preprint arXiv:2110.04652}, 2021.

\bibitem[Wang et~al.(2022)Wang, Du, Torralba, Isola, Zhang, and
  Tian]{wang2022denoised}
Tongzhou Wang, Simon~S Du, Antonio Torralba, Phillip Isola, Amy Zhang, and
  Yuandong Tian.
\newblock Denoised mdps: Learning world models better than the world itself.
\newblock \emph{arXiv preprint arXiv:2206.15477}, 2022.
\newblock URL \url{https://arxiv.org/pdf/2206.15477.pdf}.

\bibitem[Wydmuch et~al.(2018)Wydmuch, Kempka, and
  Ja{\'s}kowski]{wydmuch2018vizdoom}
Marek Wydmuch, Micha{\l} Kempka, and Wojciech Ja{\'s}kowski.
\newblock Vizdoom competitions: Playing doom from pixels.
\newblock \emph{IEEE Transactions on Games}, 11\penalty0 (3):\penalty0
  248--259, 2018.

\bibitem[Ye et~al.(2022)Ye, Zhang, Abbeel, and Gao]{ye2022become}
Weirui Ye, Yunsheng Zhang, Pieter Abbeel, and Yang Gao.
\newblock Become a proficient player with limited data through watching pure
  videos.
\newblock In \emph{The Eleventh International Conference on Learning
  Representations}, 2022.

\bibitem[Zhang et~al.(2020)Zhang, McAllister, Calandra, Gal, and
  Levine]{zhang2020bisim}
Amy Zhang, Rowan McAllister, Roberto Calandra, Yarin Gal, and Sergey Levine.
\newblock Learning invariant representations for reinforcement learning without
  reconstruction.
\newblock \emph{arXiv preprint arXiv:2006.10742}, 2020.

\bibitem[Zhao et~al.(2022)Zhao, Karamcheti, Kollar, Finn, and
  Liang]{Zhao2022video}
Tony Zhao, Siddharth Karamcheti, Thomas Kollar, Chelsea Finn, and Percy Liang.
\newblock What makes representation learning from videos hard for control?
\newblock \emph{RSS Workshop on Scaling Robot Learning}, 2022.
\newblock URL \url{https://api.semanticscholar.org/CorpusID:252635608}.

\end{thebibliography}
\bibliographystyle{arxiv2024}

\newpage
\appendix

{
\hypersetup{hidelinks}
\addcontentsline{toc}{section}{Appendix} 
\part{Appendix} 
\parttoc
}
\allowdisplaybreaks

\section{Additional Related Work}
\label{app:related-work}

\paragraph{Representation Learning for Reinforcement Learning} A line of research on recurrent state space models is essentially concerned with the next-frame approach, although typically with conditioning on actions.  Moreover, to model uncertainty in the observations, a latent variable with a posterior depending on the current observation (or even a sequence of future observations) is typically introduced.  \citep{ke2019learning} considered learning such a sequential prediction model which predicts observations and conditions on actions.  They used a latent variable with a posterior depending on future observations to model uncertainty.  These representations were used for model-predictive control and improved imitation learning.  Dreamer \citep{hafner2019dream, hafner2023dreamerv3} uses the next-frame objective but also conditions on actions.  The IRIS algorithm \citep{micheli2023iris} uses the next-frame objective but uses the transformer architecture, again conditioning on actions. The InfoPower approach \citep{bharadhwaj2022infopower} combines a one-step inverse model with a temporal contrastive objective.  \cite{sobal2022jepa} explored using semi-supervised objectives for learning representations in RL, yet used action-labeled data.  \cite{wang2022denoised} used a decoupled recurrent neural network approach to learn to extract endogenous states, but relied on action-labeled data to achieve the factorization. Deep Bisimulation for Control \citep{zhang2020bisim} introduced an objective to encourage observations with similar value functions to map to similar representations.

Self-prediction methods such as BYOL-explore \citep{guo2022byol} proposed learning reward-free representations for exploration, but depended on open-loop prediction of future states conditioned on actions .  An analysis paper studied a simplified action-free version of the self-prediction objective \citep{tang2023spr} and showed results in the absence of using actions, although this has not been instantiated empirically to our knowledge.  

A further line of work from theoretical reinforcemnt learning has examined provably efficient objectives for discovering representations.  \cite{efroni2022ppe} explored representation learning in the presence of exogenous noise, establishing a sample efficient algorithm.  However \cite{efroni2022ppe} and the closely related work on filtering exogenous noise required actions \citep{lamb2022guaranteed, islam2022acro}.  Other theoretical work on learning representations for RL has required access to action-labeled data \citep{misra2020kinematic}.  

\paragraph{Representation Learning from Videos} Self-supervised representation learning from videos has a long history. \cite{srivastava2015lstm} used recurrent neural networks with a pixel prediction objective on future frames.  \cite{parthasarathy2022self} explored temporal contrastive objectives for self-supervised learning from videos.  They also found that the features learned well aligned with human perceptual priors, despite the model not being explicitly trained to achieve such alignment.  \cite{aubret2023time} applied temporal contrastive learning to videos of objects being manipulated in a 3D space, showing that this outperformed standard augmentations used in computer vision. 

\paragraph{Using Video Data for Reinforcement Learning} The VIPER method \citep{escontrela2023video} uses a pre-trained autoregressive generative model over action-free expert videos as a reward signal for training an imitation learning agent. The Video Pre-training (VPT) algorithm \citep{baker2022vpt} trained an inverse kinematics model on a small dataset of Minecraft videos and used the model to label a large set of unlabeled Minecraft videos from the internet.  This larger dataset was then used for imitation learning and reinforcement learning for downstream tasks.  \cite{Zhao2022video} explicitly studied the challenges in using videos for representation learning in RL, identifying five key factors: task mismatch, camera configuration, visual feature shift, sub-optimal behaviors in the data, and robot morphology. \citet{goo2019one} learn reward functions for multi-step tasks from videos by leveraging a single video segmented with action labels (one-shot learning). \citet{sikchi2022ranking} propose a two-player ranking game between a policy and a reward function to satisfy pairwise performance rankings between behaviors. Their proposed method achieves state-of-the-art sample efficiency and can solve previously unsolvable tasks in the learning from observation (no actions) setting.

\section{Proofs of Theoretical Statements}

We state our setting and general assumptions before presenting method specific results. We also include a table of notations in \cref{tab:notation}.
\begin{table}[!h]
    \centering
    \begin{tabular}{l|l}
    \hline
         \textbf{Notation} & \textbf{Description} \\
         \hline\\[.01cm]
         $[N]$ & Denotes the set $\{1, 2, \cdots, N\}$\\
         $\Delta(\Ucal)$ & Denotes the set of all distributions over a set $\Ucal$\\
         $\unf(\Ucal)$ & Uniform distribution over $\Ucal$\\
         $\supp(\PP)$ & Support of a distribution $\PP \in \Delta(\Ucal)$, i.e., $\supp(\PP) = \{x \in \Ucal \mid \PP(x) > 0\}$.\\
         $\Xcal$ & Observation space\\
         $\Scal$ & Latent endogenous state\\
         $\Acal$ & Action space\\
         $T: \Scal \rightarrow \Acal \rightarrow \Delta(\Scal)$ & Transition dynamics\\
         $R: \Scal \times \Acal \rightarrow [0, 1]$ & Reward function\\
         $\mu$ & Start state distribution \\
         $H$ & Horizon indicating the maximum number of actions per episode\\
         $\phi^\star: \Xcal \rightarrow \Scal$ & Endogenous state decoder\\ 
         \hline
    \end{tabular}
    \caption{Description for mathematical notations.}
    \label{tab:notation}
\end{table}

We are given a dataset $\Dcal = \cbr{(x^{(i)}_1, x^{(i)}_2, \cdots, x^{(i)}_H)}_{i=1}^n$ of $n$ independent and identically distributed (iid) unlabeled episodes. We will use the word video and unlabeled episodes interchangeably. We assume the underlying data distribution is $D$. We denote the probability of an unlabeled episode as $D(x_1, x_2, \cdots, x_H)$. We assume that $D$ is generated by a mixture of Markovian policies $\dpol$, i.e., the generative procedure for $D$ is to sample a policy $\pi \in \dpol$ with probability $\Theta_\pi$ and then generate an entire episode using it. For this reason, we will denote $D = \Theta \circ \dpol$ where $\Theta$ is the mixture distribution. We assume \emph{no direct knowledge} about either $\dpol$ or $\Theta$, other than that the set of policies in $\dpol$ are Markovian. We define the \emph{underlying} distribution over the action-labeled episode as $D(x_1, a_1, x_2, \cdots, x_H, a_H)$, of which the agent \emph{only} gets to observe the $(x_1, x_2, \cdots, x_H)$. We will use the notation $D$ to refer to any distribution that is derived from the above joint distribution.

We assume that observations encode time steps. This can be trivially accomplished by simply concatenating the time step information to the observation. This also implies that observations from different time steps are different. Because of this property, we can assume that the Markovian policies used to realize $D$ were time homogenous, i.e., they only depend on observation and not observation and timestep pair (this is because we include timesteps in the observation). Therefore, for all $h \in [H]$ and $k \in \NN$ we have:

\begin{equation}
    D(x_{h+k}=x' \mid x_h=x) = D(x_{k+1}=x' \mid x_1=x)
\end{equation}

We denote $D(x_h)$ to define the marginal distribution over an observation $x_h$, and $D(x_h, x_{h+k})$ to denote the marginal distribution over a pair of observations $(x_h, x_{h+k})$ in the episode. We similarly define $D(x_h, a
_h)$ as the distribution over observation action pairs $(x_h, a_h)$.

We assume that the video data has good coverage. This is stated formally below:
\begin{assumption}[State Coverage by $D$]\label{assum:concentrability} Given our policy class $\Pi$, there exists an $\etamin > 0$ such thatif $\sup_{\pi \in \Pi}\PP_\pi(s_h=s) > 0$ for some $s \in \Scal$, then  we assume $D\rbr{\edo(x_h)=s} \ge \etamin$.
\end{assumption}

In practice, \cref{assum:concentrability} can be satisfied since videos are more easily available than labeled episodes and we can hope that a large diverse collection of videos can provide reasonable coverage over the underlying state action space. E.g., for tasks like gaming, one can use hours of streaming data from many users. 

Further, we also assume that the data policy depends only on the endogenous state. Recall that for an observation $x \in \Xcal$, its endogenous state is given by $\phi^\star(x) \in \Scal$.
\begin{assumption}[Noise-Free Video Distribution] For any $h$, $\pi \in \dpol$, $x_h\in\supp~\PP_\pi$ and $a \in \Acal$, we have
    \begin{equation*}
        \pi(a \mid x_h) = \pi(a \mid \phi^\star(x_h)).
    \end{equation*}
\end{assumption}
\paragraph{Justification of Noise-Free Policy.} Typically, video data is created by humans. E.g., a human may be playing a game and the video data is collected by recording the user's screen. A user is unlikely to take actions relying on iid or exogenous noise in the observation process. Therefore, the collected data can be expected to obey the noise-free assumption.

\paragraph{Multi-step transition.} We choose to analyze a multi-step variant of standard temporal contrastive and forward modeling algorithms that train on a dataset of pairs of observations $(x, x')$ that can be variable time steps apart. As our proof will show, this gives the algorithms more expressibility and allows them to learn correct representations for some problems that their single-step variants (i.e., the observations are adjacent) or fixed time-step variants (i.e., the observations are fixed time steps apart) cannot solve. We will use the variable $k$ to denote the time steps by which these observations differ. Formally, we will call $(x, k, x')$ as a multi-step transition where $x$ was observed at some time step $h$, and $x'$ was observed at $h+k$. For the single-step variant of the algorithms, we have $k=1$. For the fixed multi-step variant, we have $k>1$ but $k$ is fixed. Finally, in the general multi-step variant, we will assume that $k$ is picked from $\unf([K])$ where $K$ is a fixed upper bound.  

\paragraph{Extending episode to $H+K$.} When using $k>1$, we may want to collect a multi-step transition $(x, k, x')$ where $x = x_H$ to allow learning state representation for time step $H$. However, at this point, we don't have time steps left to observe $x_{H+k}$. We alleviate this by assuming that we can allow an episode to run till $H+K$ if necessary. In practice, this is not a problem where the algorithm sets the horizon and not the environment. However, if we cannot go past $H$, then we can instead assume that all states are reachable by the time step $H-K$ and so their state representation can be learned when $x$ is selected at $x_{H-K}$. In our analysis ahead, we make the former setting that the episodes can be extended to $H+K$, but it can be easily rephrased to work with the other setting.

For both the forward model and the temporal contrastive approach, we assume access to a dataset $\pairD = \rbr{(x^{(i)}, k^{(i)}, x'^{(i)})}_{i=1}^n$ of pairs of observations. We define a few different distributions that can be used to generate this set. For a given $k \in [K]$, we define a distribution $\kpairdist$ over $k$-step separate observations as:
\begin{equation}
    \kpairdist\rbr{X=x, X'=x'} = \frac{1}{H}\sum_{h=1}^{H} D(x_h=x, x_{h+k}=x') 
\end{equation}

We can sample $(x, k, x') \sim \kpairdist(X, X')$ by sampling an episode $(x_1, x_2, \cdots, x_H) \sim D$, and then sampling a $h \sim \unf([H])$, and choosing $x = x_h$ and $x'=x_{h+k}$. 

We also define a distribution $\upairdist$ where we also sample $k$ uniformly over available choices:
\begin{equation}
    \upairdist\rbr{X=x, k, X'=x'} = \frac{1}{K} \kpairdist(x_h=x, x_{h+k}=x') 
\end{equation}

We can sample $(x, k, x') \sim \upairdist(X, X')$ by sampling an episode $(x_1, x_2, \cdots, x_H) \sim D$, and then sampling $h \in [H]$, and sampling $k \in [K]$, and choosing $(x_h, x_{h+k})$ as the selected pair.

We define a useful notation $\rho \in \Delta(\Xcal)$ as:
\begin{equation}
    \rho(X=x) = \frac{1}{H} \sum_{h=1}^H D(x_h=x).
\end{equation}

The distribution $\rho(X)$ is a good distribution to sample from as it covers states across all time steps. Finally, because of~\pref{assum:concentrability}, we have the following:
\begin{equation}\label{eqn:rho-bound}
    \forall s \in \Scal, \qquad \rho(s) \ge \frac{\etamin}{H}
\end{equation}
This is because we assume every state $s \in \Scal$, is visited at some time step $t$, and so we have $D(s_t=s) \ge \etamin$, and $\rho(s) = \frac{1}{H}\sum_{h=1}^H D(s_h=s) \ge \frac{1}{H}D(s_t=s) \ge \frac{\etamin}{H}$. 

It can be easily verified that for both $\kpairdist(X, X')$ and $\upairdist(X, X')$, their marginals over $X$ is given by $\rho(X)$. Both $\kpairdist$ and $\upairdist$ satisfy the noise-free property. We prove this using the next two Lemma.s

\begin{lemma}[Property of Noise-Free policy]\label{lem:exo-free-policy} Let $\pi$ be a policy such that for any $x \in \Xcal$, we have $\pi(a \mid x) = \pi(a \mid \edo(x))$. Then for any $h \in [H]$ and $k \in [K]$ we have $\PP_\pi(x_{h+k}=x' \mid x_h=x)$ only depend on $\edo(x)$ and this common value is defined by $\PP_\pi(x_{h+k} \mid s_h=\edo(x))$.
\end{lemma}
\begin{proof} The proof is by induction on $k$. For $k=1$ we have:
\begin{align*}
    \PP_\pi(x_{h+1}=x' \mid x_h=x) &= \sum_{a \in \Acal}T(x' \mid x, a) \pi(a \mid x_h=x) = \sum_{a \in \Acal} T(x' \mid \edo(x), a) \pi(a \mid x_h=\edo(x)),
\end{align*}
and as the right hand side only depends on $\edo(x)$, the base case is proven. For the general case, we have:
\begin{align*}
\PP_\pi(x_{h+k}=x' \mid x_h=x) &= \sum_{\tilde{x} \in \Xcal}\PP_\pi(x_{h+k}=x', x_{h+k-1}=\tilde{x} \mid x_h=x)\\
&= \sum_{\tilde{x} \in \Xcal}\PP_\pi(x_{h+k}=x' \mid x_{h+k-1}=\tilde{x}) \PP_\pi(x_{h+k-1}=\tilde{x} \mid x_h=x)\\
&= \sum_{\tilde{x} \in \Xcal}\PP_\pi(x_{h+k}=x' \mid x_{h+k-1}=\tilde{x}) \PP_\pi(x_{h+k-1}=\tilde{x} \mid x_h=\edo(x)),
\end{align*}
where the second step uses the fact that $\pi$ is Markovian and the last step uses the inductive case for $k-1$.
\end{proof}

\begin{lemma}[Distribution over Pairs]\label{lem:exo-free-conditionals} Let $k \in [K]$, $x \in \supp~\rho(X)$, then the distribution $\kpairdist(X' \mid x)$ only depends on $\edo(x)$. This allows us to define $\kpairdist(X' \mid \edo(x))$ as this common value.
Similarly, the distribution $\upairdist(X' \mid x, k)$ depends only on $\edo(x)$ and $k$. We define this common value as $\upairdist(X' \mid \edo(x), k)$.
\end{lemma}
\begin{proof} For any $k$ we have:
\begin{align*}
    D_k(X=x, X'=x') &= \frac{1}{H} \sum_{h=1}^H D(x_h=x, x_{h+k}=x') \\
    &= \frac{1}{H} \sum_{h=1}^H \sum_{\pi \in \dpol}\Theta_\pi \PP_\pi(x_h=x, x_{h+k}=x') \\
    &= \frac{1}{H} \sum_{h=1}^H \sum_{\pi \in \dpol}\Theta_\pi \PP_\pi(x_h=x) \PP(x_{h+k}=x' \mid x_h=x) \\
    &= \frac{1}{H} \sum_{h=1}^H \sum_{\pi \in \dpol}\Theta_\pi \PP_\pi(x_h=x) \PP_\pi(x_{h+k}=x' \mid s_h=\edo(x)), \quad \mbox{(using~\pref{lem:exo-free-policy})}\\
    &= \frac{q(x \mid \edo(x))}{H} \sum_{h=1}^H \sum_{\pi \in \dpol}\Theta_\pi \PP_\pi(s_h=\edo(x)) \PP_\pi(x_{h+k}=x' \mid s_h=\edo(x))
\end{align*}
The marginal $D_k(X=x)$ is given by:
\begin{align*}
D_k(X=x) = \frac{1}{H} \sum_{h=1}^H \sum_{\pi \in \dpol}\Theta_\pi q(x \mid \edo(x))\PP_\pi(s_h=\edo(x)) = \frac{q(x \mid \edo(x))}{H}\sum_{h=1}^H D_k(s_h=\edo(x)).
\end{align*}
The conditional $D_k(X'=x' \mid X=x)$ is given by:
\begin{align*}
    D_k(X'=x' \mid X=x) &= \frac{D_k(X=x, X'=x')}{D_k(x)}\\
    &= \frac{\sum_{h=1}^H \sum_{\pi \in \dpol}\Theta_\pi \PP_\pi(s_h=\edo(x)) \PP_\pi(x_{h+k}=x' \mid s_h=\edo(x))}{\sum_{h=1}^H D_k(s_h=\edo(x))}
\end{align*}
Therefore, the conditional $D_k(X'=x' \mid X=x$ only depends on $\edo(x)$, and we define this common value as $D_k(X'=x' \mid s=\edo(x))$.

The proof for $\upairdist$ is similar. We can use the property of $\kpairdist$ that we have proven to get:
\begin{align*}
    \upairdist(X'=x' \mid X=x, k) &= \frac{\upairdist(X=x, k, X'=x')}{\sum_{\tilde{x} \in \Xcal}\upairdist(X=x, k, X'=\tilde{x})}\\ 
    &= \frac{\kpairdist(X=x, X'=x')}{\sum_{\tilde{x} \in \Xcal}\kpairdist(X=x, X'=\tilde{x})}\\ 
    &= \frac{\kpairdist(X'=x' \mid X=x)}{\sum_{\tilde{x} \in \Xcal}\kpairdist(X'=\tilde{x} \mid X=x)}\\ 
    &= \frac{\kpairdist(X'=x' \mid X=\edo(x))}{\sum_{\tilde{x} \in \Xcal}\kpairdist(X'=\tilde{x} \mid X=\edo(x))}. 
\end{align*}
Therefore, $\upairdist(X'=x' \mid X=x, k)$ only depends on $\edo(x)$. We will define the common values as $\upairdist(X'=x' \mid s=\edo(x), k)$.
\end{proof}

\pref{lem:exo-free-conditionals} allows us to define $\kpairdist(x' \mid \edo(x))$ and $\upairdist(x' \mid \edo(x), k)$, as the distribution only depends on the latent state.

\subsection{Upper Bound for the Forward Model Baseline}

Let $\pairD = \{(x^{(i)},k^{(i)}, x'^{(i)})\}_{i=1}^n$ be a pair of iid multi-step observations. We will collect this dataset in one of three ways:

\begin{enumerate}
    \item Single step $(k=1)$, in this case we will sample $(x^{(i)}, x'^{(i)}) \sim D_k(X, X')$. As explained before, we can get this sample using the episode data. We save $(x^{(i)}, k, x'^{(i)})$ as our sample.
    \item Fixed multi-step. We use a fixed $k > 1$, and sample $(x^{(i)}, x'^{(i)}) \sim D_k(X, X')$. We save $(x^{(i)}, k, x'^{(i)})$ as our sample.
    \item Variable multi-step. We sample $(x, k, x') \sim \upairdist(X, k, X')$ and use it as our sample.
\end{enumerate}

We will abstract these three choices using a general notion of $\pairdist \in \Delta(\Xcal \times [K] \times \Xcal)$. In the first two cases, we assume we have point-mass distribution over $k$ and given this $k$, we sample from $D_k(X, X')$. We will assume $(x^{(i)}, k^{(i)}, x'^{(i)}) \sim \pairdist$. We can create $\pairD$ from the dataset $\Dcal$ of $n$ episodes sampled from $D$ using the sampling procedures explained earlier. Note that as marginals over both $D_k(X)$ and $\upairdist(X)$ is $\rho(X)$, therefore, the marginals over $\pairdist(X)$ is also $\rho(X)$. Additionally, we will define $\pairdist(k)$ as the marginal over $k$ which is either point-mass in the first two sampling procedures and $\unf([K])$ in the third procedure.

We assume access to two function classes. The first is a decoder class $\Phi_N: \Xcal \rightarrow [N]$ where $N$ is a given number that satisfies $N \ge |\Scal|$. The second is a conditional probability class $\Fcal: [N] \times [K] \rightarrow \Delta(\Xcal)$. 

\begin{assumption}(Realizability of $\Phi$ and $\Fcal$)\label{assum:realizability-block-mdp-forward} We assume that there exists $\phi^\circ \in \Phi_N$ and $f^\circ \in \Fcal$ such that $f^\circ(x' \mid \phi^\circ(x), k) = \pairdist(x' \mid x, k) = \pairdist(x' \mid \edo(x), k)$ for all $(x, k) \sim \pairdist(\cdot, \cdot)$.
\end{assumption}

This assumption firstly is non-vacuous as $\pairdist(x' \mid x) = \pairdist(x' \mid \edo(x))$, and therefore, we can apply a bottleneck function $\phi$ and still assume realizability. For example, we can assume that $\tilde{\phi}$ is the same as $\edo$ up to the relabeling of its output, and $\tilde{f}(x' \mid i) = \pairdist(x' \mid s)$.

Let $\hat{f}, \hat{\phi}$ be the empirical solution to the following maximum likelihood problem.
\begin{equation}
    \hat{f}, \hat{\phi} = \arg\max_{f \in \Fcal, \phi \in \Phi_N} \frac{1}{n} \sum_{i=1}^n \ln f\rbr{x'^{(i)} \mid \phi(x^{(i)}), k^{(i)}}
\end{equation}

Note that when $k$ is fixed (we sample from $\kpairdist$), then information theoretically there is no advantage of condition on $k$ and it can be dropped from optimization.

As we are in a realizable setting (\pref{assum:realizability-block-mdp-forward}), we can use standard maximum likelihood guarantees to get the following result.
\begin{proposition}[Generalization Bound] Fix $\delta \in (0, 1)$, then with probability at least $1-\delta$, we have:
\begin{equation*}
    \EE_{(x, k) \sim \pairdist}\sbr{\TV{\pairdist(X' \mid x, k) - \hat{f}(X' \mid \hat{\phi}(x), k)}^2} \le \Delta^2(n; \delta),
\end{equation*}
where $\Delta^2(n; \delta) = \frac{2}{n} \ln\rbr{\frac{|\Phi|\cdot|\Fcal|}{\delta}}$.
\end{proposition}
For proof see Chapter 7 of~\citet{geer2000empirical}.

Finally, we assume that the forward modeling objective is expressive to allow the separation of states. While, this seems like assuming that the objective works, our goal is to establish a formal notion of the margin so we can verify it later in different settings to see when it holds.
\begin{assumption}(Forward Modeling Margin). We assume there exists a $\formargin \in (0, 1)$ such that:
    \begin{equation*}
       \inf_{s_1, s_2 \in \Scal, s_1 \ne s_2} \EE_{k \sim \pairdist}\sbr{\TV{\pairdist(X' \mid s_1, k) - \pairdist(X' \mid s_2, k)}} \ge \formargin
    \end{equation*}
\end{assumption}
Note that this defines two types of margin depending on $\pairdist$. When $k$ is a fixed value, the margin is given by:
\begin{equation*}
    \formargin^{(k)} = \inf_{s_1, s_2 \in \Scal, s_1 \ne s_2} \TV{\pairdist(X' \mid s_1, k) - \pairdist(X' \mid s_2, k)}.
\end{equation*}
When we sample $k \sim \unf([K])$ then the margin is given by:
\begin{equation*}
    \formargin^{(u)} = \inf_{s_1, s_2 \in \Scal, s_1 \ne s_2} \frac{1}{K}\sum_{k=1}^K\TV{\pairdist(X' \mid s_1, k) - \pairdist(X' \mid s_2, k)}.
\end{equation*}
We will use the abstract notion $\formargin$ for forward margin which will be equal to $\formargin^{(k)}$ or $\formargin^{(u)}$ depending on our sampling procedure. It is easy to see that $\formargin^{(u)} = \frac{1}{K}\sum_{k=1}^K \formargin^{(k)}$.

We are now ready to state our first main result.
\begin{proposition}[Recovering Endogenous State.]\label{prop:coupling-forward-modeling} Fix $\delta \in (0, 1)$, then with probability at least $1-\delta$ we learn $\phihat$ that satisfies:
\begin{equation*}
    \PP_{x_1, x_2 \sim \rho}\rbr{\phi^\star(x_1) \ne \phi^\star(x_2) \land \hat{\phi}(x_1) = \hat{\phi}(x_2)} \le \frac{2\Delta(n, \delta)}{\formargin}.
\end{equation*}
\end{proposition}
\begin{proof} We start with a coupling argument where we sample $x_1, x_2$ independently from $\pairdist(X)$ which is the same as $\rho(X)$.
\begin{align*}
    &\EE_{x_1, x_2 \sim \pairdist, k \sim \pairdist}\sbr{\one\cbr{\hat{\phi}(x_1) = \hat{\phi}(x_2)}\TV{\pairdist(X' \mid x_1, k) - \pairdist(X' \mid x_2, k)}}\\
    &\le \EE_{x_1, x_2 \sim \pairdist, k \sim \pairdist}\sbr{\one\cbr{\hat{\phi}(x_1) = \hat{\phi}(x_2)}\TV{\hat{f}(X' \mid \hat{\phi}(x_1), k) - \pairdist(X' \mid x_1, k)}}\\ 
    &+ \EE_{x_1, x_2 \sim \pairdist, k \sim \pairdist}\sbr{\one\cbr{\hat{\phi}(x_1) = \hat{\phi}(x_2)}\TV{\hat{f}(X' \mid \hat{\phi}(x_1), k) - \pairdist(X' \mid x_2, k)}}
\end{align*}
We bound these two terms separately
\begin{align*}
    &\EE_{x_1, x_2 \sim \pairdist, k \sim \pairdist}\sbr{\one\cbr{\hat{\phi}(x_1) = \hat{\phi}(x_2)}\TV{\hat{f}(X' \mid \hat{\phi}(x_1), k) - \pairdist(X' \mid x_1, k))}} \\
    &\le \sqrt{\EE_{x_1, x_2 \sim \pairdist, k \sim \pairdist}\sbr{\one\cbr{\hat{\phi}(x_1) = \hat{\phi}(x_2)}}} \cdot \sqrt{\EE_{x_1, x_2 \sim \pairdist, k \sim 
 \pairdist}\sbr{\TV{\hat{f}(X' \mid \hat{\phi}(x_1), k) - \pairdist(X' \mid x_1, k))}^2}}\\
    &= \sqrt{\EE_{x_1, x_2 \sim \pairdist}\sbr{\one\cbr{\hat{\phi}(x_1) = \hat{\phi}(x_2)}}} \cdot \sqrt{\EE_{(x, k) \sim \pairdist}\sbr{\TV{\hat{f}(X' \mid \hat{\phi}(x) - \pairdist(X' \mid x))}^2}}\\
    &\le b \cdot \Delta,
\end{align*}
where $b = \sqrt{\EE_{x_1, x_2 \sim \pairdist}\sbr{\one\cbr{\hat{\phi}(x_1) = \hat{\phi}(x_2)}}}$ and the second step uses Cauchy-Schwarz inequality. It is straightforward to verify that $b \in [0, 1]$. We bound the second term similarly
\begin{align*}
   &\EE_{x_1, x_2 \sim \pairdist, k \sim \pairdist}\sbr{\one\cbr{\hat{\phi}(x_1) = \hat{\phi}(x_2)}\TV{\hat{f}(X' \mid \hat{\phi}(x_1), k) - \pairdist(X' \mid x_2, k)}} \\
   &=\EE_{x_1, x_2 \sim \pairdist}\sbr{\one\cbr{\hat{\phi}(x_1) = \hat{\phi}(x_2)}\TV{\hat{f}(X' \mid \hat{\phi}(x_2), k) - \pairdist(X' \mid x_2, k)}}\\
   &\le b \cdot \Delta,
\end{align*}
where the second step uses the crucial coupling argument that we can replace $x_1$ with $x_2$ because of the indicator $\one\cbr{\hat{\phi}(x_1) = \hat{\phi}(x_2)}$, and the last step follows as we reduce it to the first term except we switch the names of $x_1$ and $x_2$. Combining the two upper bounds we get:
\begin{align*}
    \EE_{x_1, x_2 \sim \pairdist, k \sim \pairdist}\sbr{\one\cbr{\hat{\phi}(x_1) = \hat{\phi}(x_2)}\TV{\pairdist(X' \mid x_1, k) - \pairdist(X' \mid x_2, k)}} \le 2 b\cdot \Delta
\end{align*}
or, equivalently, 
\begin{align*}
    \EE_{x_1, x_2 \sim \pairdist}\sbr{\one\cbr{\hat{\phi}(x_1) = \hat{\phi}(x_2)}\underbrace{\EE_{k \sim \pairdist}\sbr{\TV{\pairdist(X' \mid x_1, k) - \pairdist(X' \mid x_2, k)}}}_{\defeq \Gamma(x_1, x_2)}} \le 2 b\cdot \Delta
\end{align*}

Let $\Gamma(x_1, x_2) = \EE_{k \sim \pairdist}\sbr{\TV{\pairdist(X' \mid x_1, k) - \pairdist(X' \mid x_2, k)}}$.
For any two observations, if $\edo(x_1) = \edo(x_2)$, then $\TV{\pairdist(X' \mid x_1) - \pairdist(X' \mid x_2)} = 0$, and therefore, $\Gamma(x_1, x_2)=0$ because of~\pref{lem:exo-free-conditionals}. Otherwise, $\Gamma(x_1, x_2)$ is at least $\formargin$, by~\pref{assum:realizability-block-mdp-forward}. Combining these two observations we get:
\begin{align*}
    \Gamma(x_1, x_2) \ge \formargin \one\{\phi^\star(x_1) \ne \phi^\star(x_2)\}
\end{align*}

Combining the previous two inequalities we get:
\begin{align*}
    \EE_{x_1, x_2 \sim \pairdist}\sbr{\one\cbr{\hat{\phi}(x_1) = \hat{\phi}(x_2) \land \phi^\star(x_1) \ne \phi^\star(x_2)}} \le \frac{2 b\cdot \Delta}{\formargin}
\end{align*}

This directly gives
\begin{equation*}
    \PP_{x_1, x_2 \sim \pairdist}\rbr{\hat{\phi}(x_1) = \hat{\phi}(x_2) \land \phi^\star(x_1) \ne \phi^\star(x_2)} \le \frac{2b \Delta}{\formargin} \le \frac{2\Delta}{\formargin}.
\end{equation*}
The proof is completed by recalling that marginal $\pairdist(X)$ is the same as $\rho(X)$.
\end{proof}

\pref{prop:coupling-forward-modeling} shows that the learned $\phihat$ has one-sided error. If it merges two observations, then with high probability they are not from the same state. As $N=|\Scal|$, we will show below that the reverse is also true.

\begin{theorem}\label{thm:bijection-forward} If $N = |\Scal|$, then there exists a bijection $\alpha: [N] \rightarrow \Scal$ such that for any $s \in \Scal$ we have:
\begin{equation*}
\PP_{x \sim q(\cdot \mid s)}\rbr{\phihat(x) = \alpha(s) \mid \edo(x)=s} \ge 1 - \frac{4N^3H^2\Delta}{\etamin^2\formargin},
\end{equation*}
provided $\Delta < \frac{\etamin^2\formargin}{N^2H^2}$.
\end{theorem}
\begin{proof} We define a few shorthand below for any $j \in [N]$ and $\tilde{s} \in \Scal$
\begin{align*}
    \PP(j, \tilde{s}) &= \PP_{x \sim \rho}\rbr{\hat{\phi}(x) = j \land \edo(x) = \tilde{s}}\\
    \rho(j) &= \PP_{x \sim \rho}\rbr{\hat{\phi}(x) = j}\\
    \rho(\tilde{s}) &= \PP_{x \sim \rho}\rbr{\edo(x) = \tilde{s}}.
\end{align*}
It is easy to verify that $\PP(j, \tilde{s})$ is a joint distribution with $\rho(j)$ and $\rho(\tilde{s})$ as its marginals.

Fix $i \in [N]$ and $s \in \Scal$. 
\begin{align*}
    &\PP_{x_1, x_2 \sim \rho}\rbr{\hat{\phi}(x_1) = \hat{\phi}(x_2) \land \edo(x_1) \ne \edo(x_2)}\\
    &=\PP_{x_1, x_2 \sim \rho}\rbr{\cup_{\tilde{s} \in \Scal, j \in [N]} \cbr{\hat{\phi}(x_1) = j \land \hat{\phi}(x_2) =j \land \edo(x_1) = \tilde{s} \land \edo(x_2) \ne \tilde{s}}}\\
    &\ge \PP_{x_1, x_2 \sim \rho}\rbr{\hat{\phi}(x_1) = i \land \hat{\phi}(x_2) =i \land \edo(x_1) = s \land \edo(x_2) \ne s}\\
    &= \PP_{x_1 \sim \rho}\rbr{\hat{\phi}(x_1) = i \land \edo(x_1) = s} \PP_{x_2 \sim \rho}\rbr{\hat{\phi}(x_2)=i \land \edo(x_2) \ne s}\\
    &= \PP_{x \sim \rho}\rbr{\hat{\phi}(x) = i \land \edo(x) = s} \rbr{\sum_{s' \in \Scal}\PP_{x \sim \rho}\rbr{\hat{\phi}(x) = i \land \edo(x) = s'} - \PP_{x\sim\rho}(\hat{\phi}(x) = i \land \edo(x) = s)}\\
    &= \PP(i, s) \rbr{\sum_{s' \in \Scal} \PP(i, s') 
 - \PP(i, s)}\\
    &= \PP(i, s) \rbr{\rho(i) - \PP(i, s)}.
\end{align*}

Combining this with~\pref{prop:coupling-forward-modeling}, we get:
\begin{equation*}
  \forall i \in [N], s \in \Scal, \qquad  \PP(i, s) \rbr{\rho(i) - \PP(i, s)} \le \Delta' \defeq \frac{2\Delta}{\formargin}
\end{equation*}
where we have used a shorthand $\Delta' = 2\Delta/\formargin$. We define a mapping $\alpha: \Scal \rightarrow [N]$ where for any $s \in \Scal$:
\begin{equation}\label{eqn:learned-bijection}
    \alpha(s) = \arg\max_{j \in [N]} \PP(j, s) 
\end{equation}
We immediately have:
\begin{equation}\label{eqn:rho-i-bound}
\PP(\alpha(s), s) = \max_{j \in [N]}\PP(j, s) \ge \frac{1}{N} \sum_{j=1}^N \PP(j, s) = \frac{1}{N} \rho(s)  \ge \frac{\etamin}{NH},  
\end{equation}
where we use the fact that max is greater than average in the first inequality, and~\pref{eqn:rho-bound}. Further, for every $s \in \Scal$, we have:
\begin{equation*}
    \PP(\alpha(s), s) \rbr{\rho(\alpha(s)) - \PP(\alpha(s), s)} \le \Delta'.
\end{equation*}
Plugging the lower bound $\PP(\alpha(s), s) \ge \frac{\etamin}{NH}$, we get:
\begin{equation}\label{eqn:abstract-state-converge}
    \PP(\alpha(s), s) \ge \rho(\alpha(s)) - \frac{NH\Delta'}{\etamin}.
\end{equation}

We now show that if $\Delta' < \frac{\etamin^2}{2N^2H^2}$, then $\alpha(s)$ is a bijection. Let $s_1$ and $s_2$ be such that $\alpha(s_1) = \alpha(s_2) = i$. Then using the above~\pref{eqn:abstract-state-converge} we get 
$\PP(i, s_1) \ge \rho(i) - \frac{NH\Delta'}{\etamin}$ and $\PP(i, s_2) \ge \rho(i) - \frac{NH\Delta'}{\etamin}$. We have:
\begin{equation*}
    \rho(i) = \sum_{\tilde{s} \in \Scal} \PP(i, \tilde{s}) \ge \PP(i, s_1) + \PP(i, s_2) \ge 2\rho(i) - \frac{2NH\Delta'}{\etamin}
\end{equation*}

This implies $\frac{2N\Delta'}{\etamin} \ge \rho(i)$ but as $\rho(i) = \rho(\alpha(s_1))\ge \PP(\alpha(s_1), s_1) \ge \frac{\etamin}{NH}$ (\pref{eqn:rho-i-bound}), we get $\frac{2NH \Delta'}{\etamin} \ge \frac{\etamin}{NH}$ or $\Delta' \ge \frac{\etamin^2}{2N^2H^2}$. However, as we assume that $\Delta' < \frac{\etamin^2}{2N^2H^2}$, therefore, this is a contradiction. This implies $\alpha(s_1) \ne \alpha(s_2)$ for any two different states $s_1$ and $s_2$. Since we assume $|N| = |\Scal|$, this implies $\alpha$ is a bijection.

Fix $s \in \Scal$ and let $i \ne \alpha(s)$. As $\alpha$ is a bijection, let $\tilde{s} = \alpha^{-1}(i)$, we can show that $\PP(i, s)$ is small:
\begin{equation}\label{eqn:failure-bound}
    \PP(i, s) \le \rho(i) - \PP(i, \tilde{s}) = \rho(\alpha(\tilde{s})) - \PP(\alpha(\tilde{s}), \tilde{s}) \le \frac{NH\Delta'}{\etamin}
\end{equation}
where we use $s \ne \tilde{s}$ and~\pref{eqn:abstract-state-converge}.

This allows us to show that $\PP(\alpha(s) \mid s)$ is high as follows:
\begin{align*}
    \PP(\alpha(s) \mid s) = \frac{\PP(\alpha(s), s)}{\rho(s)}
    &= \frac{\PP(\alpha(s), s)}{\PP(\alpha(s), s) + \sum_{i=1, i \ne \alpha(s)}^N \PP(i, s)} \\
    &\ge \frac{\PP(\alpha(s), s)}{\rho(\alpha(s)) + \frac{N^2H\Delta'}{\etamin}},\\
    &\ge \frac{\rho(\alpha(s)) - \frac{NH\Delta'}{\etamin}}{\rho(\alpha(s)) + \frac{N^2H\Delta'}{\etamin}}\\
    &= 1 - \frac{\rbr{\frac{N^2H\Delta'}{\etamin} + \frac{NH\Delta'}{\etamin}}}{\rho(\alpha(s)) + \frac{N^2H\Delta'}{\etamin}}\\
    &\ge 1 - \frac{2N^2H^2\Delta'}{\etamin\rho(\alpha(s))}\\
        &\ge 1 - \frac{2N^3H^2\Delta'}{\etamin^2},
\end{align*}
where the first inequality uses~\cref{eqn:failure-bound} and $\rho(\alpha(s)) \ge \PP(\alpha(s), s)$, second inequality uses~\cref{eqn:abstract-state-converge}, and the last step uses $\rho(\alpha(s)) \ge \PP(\alpha(s), s) \ge \frac{\etamin}{NH}$.

The proof is completed by noting that:
\begin{equation*}
    \PP_{x \sim q(\cdot \mid s)}\rbr{\phihat(x) = \alpha(s)} = \PP_{x \sim \rho}\rbr{\phihat(x) = \alpha(s) \mid \edo(x)=s} = \PP(\alpha(s) \mid s).
\end{equation*}
\end{proof}

Let $\mathscr{A}$ be a PAC RL algorithm for tabular MDPs. We assume that this algorithm's sample complexity is given by $\nsamp(S, A, H, \epsilon, \delta)$ where $S$ and $A$ are the size of the state space and action space of the tabular MDP, $H$ is the horizon, and $(\epsilon, \delta)$ are the typical PAC RL hyperparameters denoting tolerance and failure probability. Formally, the algorithm $\mathscr{A}$ interacts with a tabular MDP $\MM$ for $\nsamp(S, A, H, \epsilon, \delta)$ episodes and outputs a policy $\hat{\varphi}: \Scal \times [H] \rightarrow \Acal$ such that with probability at least $1-\delta$ we have:
\begin{equation*}
    \sup_{\varphi \in \psiall} V_\MM(\varphi) - V_\MM(\hat{\varphi}) \le \epsilon,
\end{equation*}
where $\psiall$ is the space of all policies of the type $\Scal \times [H] \rightarrow \Acal$.

We assume that we are given knowledge of the desired $(\epsilon, \delta)$ hyperparameters in the downstream RL task during the representation pre-training phase so we can use the right amount of data.

\paragraph{Induced Finite MDP.} The latent MDP inside a block MDP is a tabular MDP with state space $\Scal$, action space $\Acal$, horizon $H$, transition dynamics $T$, reward function $R$, and a start state distribution of $\mu$. If we directly had access to this latent MDP, say via the true decoding function $\edo$, then we can apply the algorithm $\mathscr{A}$ and learn the optimal latent policy $\varphi^\star$ which we can couple with $\edo$ and learn the optimal observation-based policy. Formally, we write this observation-based policy as $\varphi \circ \edo: \Xcal \times [H] \rightarrow \Acal$ given by $\varphi(\edo(x), h)$. We dont have access to $\edo$, but we have access to $\phihat$ that with high probability for a given $x$ outputs a state which is same as $\edo(x)$ up to the learned $\alpha$-bijection. We, therefore, define the induced MDP $\MM$ as the finite MDP with state space $\widehat{\Scal}$, action space $\Acal$, transition function $\widehat{T}$, reward function $\widehat{R}$ and start state distribution $\widehat{\mu}$. These same as the latent Block MDP but where the true state $s$ is replaced by $\alpha(s)$. It is this induced $\MM$ that the tabular MDP algorithm $\mathscr{A}$ will see with high probability.

\begin{proposition}[PAC RL Bound]\label{prop:tabular-bound} Let $\mathscr{A}$ be a PAC RL algorithm for tabular MDPs and $\nsamp$ is its sample complexity. Let $\phihat: \Xcal \rightarrow [N]$ be a decoder pre-trained using video data and $\alpha: \Scal \rightarrow [N]$ is a bijection such that:
\begin{equation*}
    \forall s \in \Scal, \qquad \PP_{x \sim q(\cdot \mid s)}\rbr{\phihat(x) = \alpha(s)} \ge 1 - \vartheta,
\end{equation*}
then let $\hat{\varphi}$ be the policy returned by $\mathscr{A}$ on the tabular MDP induced by $\phihat(x)$. Then we have with probability at least $1 - \delta - \nsamp(S, A, H, \epsilon, \delta) H \vartheta$:
\begin{equation*}
    \sup_{\pi \in \Pi} V(\pi) - V(\varphi \circ \phihat) \le \epsilon + 2H^2 \vartheta 
\end{equation*}
\end{proposition}
\begin{proof} The algorithm runs for $\nsamp(S, A, H, \epsilon, \delta)$ episodes. This implies the agent visits $\nsamp(S, A, H, \epsilon, \delta) H$ many latent states. If the decoder maps every such state $s$ to the correct permutation $\alpha(s)$, then the tabular MDP algorithm is running as if it ran on the induced MDP $\MM$. The probability of failure is bounded by $\nsamp(S, A, H, \epsilon, \delta) H \vartheta$ as all these failures are independent given the state. Further, the failure probability of the tabular MDP algorithm itself is $\delta$. This leads to the total failure probability of $\delta + \nsamp(S, A, H, \epsilon, \delta) H \vartheta$.

Let $\Pi$ be the set of observation-based policies we are competing with and which includes the optimal observation-based policy $\pi^\star$. We can write $\sup_{\pi \in \Pi} V(\pi) = V_{\MM}(\varphi^\star)$ where we use the subscript $\MM$ to denote that the latent policy is running in the induced MDP $\MM$. Further, for any latent policy $\varphi$ we have $V(\varphi \circ \alpha \circ \edo) = V_\MM(\varphi)$ as the decoder $\alpha \circ \edo: x \mapsto \alpha(\edo(x))$ give me access to the true state of the induced MDP $\MM$. Then with probability at least $1-\delta$, we have:
\begin{align*}
V_{\MM}(\varphi^\star) - V_\MM(\hat{\varphi}) \le \epsilon
\end{align*}

This allows us to bound the sub-optimality of the learned observation-based policy $\hat{\varphi}\circ \phihat$ as:
\begin{align*}
    \sup_{\pi \in \Pi} V(\pi) - V(\hat{\varphi}\circ\phihat) &= V(\varphi^\star \circ \alpha \circ \edo) - V(\hat{\varphi} \circ \alpha \circ \edo) + V(\hat{\varphi} \circ \alpha \circ \edo) - V(\hat{\varphi} \circ \phihat)\\
    &= V_\MM(\varphi^\star) - V_\MM(\hat{\varphi}) + V(\hat{\varphi} \circ \alpha \circ \edo) - V(\hat{\varphi} \circ \phihat)\\
    &\le \epsilon + V(\hat{\varphi} \circ \alpha \circ \edo) - V(\hat{\varphi} \circ \phihat)
\end{align*}
Here we use $\hat{\varphi}\circ \alpha \circ \phi^\star$ to denote an observation-based policy that takes action as $\hat{\varphi}(\alpha(\edo(x)), h)$.

We bound $V(\hat{\varphi} \circ \alpha \circ \edo) - V(\hat{\varphi} \circ \phihat)$ below. Let $\Ecal_h = \{\phihat(x_h) = \alpha(\edo(x_h))\}$ and $\Ecal = \cap_{h=1}^H \Ecal_h$ be two events. We have $\PP(\Ecal_h) \ge 1 - \vartheta$. Further, using union bound we have $\PP(\Ecal^c) = \PP(\cup_{h=1}^H \Ecal^c_h) \le \sum_{h=1}^H \PP(\Ecal^c_h) \le H \vartheta$. 

We first prove an upper bound on $V(\hat{\varphi} \circ \alpha \circ \edo)$:
\begin{align*}
    V(\hat{\varphi} \circ \alpha \circ \edo) &= \EE_{\hat{\varphi} \circ \alpha \circ \edo}\sbr{\sum_{h=1}^H r_h} \\
    &= \EE_{\hat{\varphi} \circ \alpha \circ \edo}\sbr{\sum_{h=1}^H r_h \mid \Ecal} \PP_{\hat{\varphi} \circ \alpha \circ \edo}(\Ecal) + \EE_{\hat{\varphi} \circ \alpha \circ \edo}\sbr{\sum_{h=1}^H r_h \mid \Ecal^c} \PP_{\hat{\varphi} \circ \alpha \circ \edo}(\Ecal^c) \\
    &\le \EE_{\hat{\varphi} \circ \alpha \circ \edo}\sbr{\sum_{h=1}^H r_h \mid \Ecal} + H^2 \vartheta\\
    &= \EE_{\hat{\varphi} \circ \phihat}\sbr{\sum_{h=1}^H r_h \mid \Ecal} + H^2 \vartheta
\end{align*}
Here we have used the fact that value of any policy is in $[0, H]$ since the horizon is $H$ and the rewards are in $[0, 1]$.

We next prove a lower bound on $V(\hat{\varphi} \circ \phihat)$:
\begin{align*}
    V(\hat{\varphi} \circ \phihat) &= \EE_{\hat{\varphi} \circ \phihat}\sbr{\sum_{h=1}^H r_h} \\
    &= \EE_{\hat{\varphi} \circ \phihat}\sbr{\sum_{h=1}^H r_h \mid \Ecal} \PP_{\hat{\varphi} \circ \phihat}(\Ecal) + \EE_{\hat{\varphi} \circ \phihat}\sbr{\sum_{h=1}^H r_h \mid \Ecal^c} \PP_{\hat{\varphi} \circ \phihat}(\Ecal^c)\\
    &\ge \EE_{\hat{\varphi} \circ \phihat}\sbr{\sum_{h=1}^H r_h \mid \Ecal} \PP_{\hat{\varphi} \circ \phihat}(\Ecal)\\
    &\ge \EE_{\hat{\varphi} \circ \phihat}\sbr{\sum_{h=1}^H r_h \mid \Ecal} - \EE_{\hat{\varphi} \circ \phihat}\sbr{\sum_{h=1}^H r_h \mid \Ecal} H\vartheta\\
    &\ge \EE_{\hat{\varphi} \circ \phihat}\sbr{\sum_{h=1}^H r_h \mid \Ecal} - H^2\vartheta
\end{align*}

Combining the two upper bounds we get:
\begin{equation*}
V(\hat{\varphi} \circ \alpha \circ \edo) - V(\hat{\varphi} \circ \phihat) \le \EE_{\hat{\varphi} \circ \phihat}\sbr{\sum_{h=1}^H r_h \mid \Ecal} + H^2 \vartheta - \EE_{\hat{\varphi} \circ \phihat}\sbr{\sum_{h=1}^H r_h \mid \Ecal} + H^2\vartheta \le 2H^2\vartheta
\end{equation*}

Therefore, with probability at least $1 - \delta -  \nsamp(S, A, H, \epsilon, \delta) H \vartheta$, learn a policy $\hat{\varphi} \circ \phihat$ such that:
\begin{equation*}
    \sup_{\pi \in \Pi} V(\pi) - V(\hat{\varphi} \circ \phihat) \le \epsilon + 2H^2 \vartheta.
\end{equation*}
\end{proof}

\begin{theorem}[Wrapping up the proof.] Fix $\epsilon_\circ > 0$ and $\delta_\circ \in (0, 1)$ and let $\mathscr{A}$ be any PAC RL algorithm for tabular MDPs with sample complexity $\nsamp(S, A, H, \epsilon, \delta)$. If $n$ satisfies:
\begin{equation*}
    n = \Ocal\rbr{\cbr{\frac{N^4H^4}{\etamin^4\formargin^2} + \frac{N^6H^8}{\epsilon^2_\circ\etamin^4\formargin^2} + \frac{N^6H^6\nsamp^2(S, A, H, \epsilon_\circ/2, \delta_\circ/4)}{\delta_\circ^2 \etamin^4\formargin^2}}\ln\rbr{\frac{|\Fcal||\Phi|}{\delta_\circ}}},
\end{equation*}
then forward modeling learns a decoder $\phihat: \Xcal \rightarrow N$. Further, running $\mathscr{A}$ on the tabular MDP with  
    induced by $\phihat$ with hyperparameters $\epsilon=\epsilon_\circ/2$, $\delta=\delta_\circ/4$, returns a latent policy $\hat{\varphi}$. Then there exists a bijective mapping $\alpha: \Scal \rightarrow [|\Scal|]$ such that with probability at least $1-\delta$ we have:
\begin{equation*}
    \forall s \in \Scal, \qquad \PP_{x \sim q(\cdot \mid s)}\rbr{\phihat(x) = \alpha(s) \mid \edo(x)=s} \ge 1 - \frac{4N^3H^2\Delta}{\etamin^2\formargin},
\end{equation*}
and 
    \begin{equation*}
        V(\pi^\star) - V(\hat{\varphi} \circ \phihat) \le \epsilon_\circ
    \end{equation*}
Further, the amount of online interactions in the downstream RL is given by $\nsamp(S, A, H, \epsilon_\circ/2, \delta_\circ/4)$ and doesn't scale with $\ln|\Phi|$.
\end{theorem}
\begin{proof}
We showed in~\pref{thm:bijection-forward} that we learn a $\phihat$ such that:
    \begin{equation*}
\PP_{x \sim q(\cdot \mid s)}\rbr{\phihat(x) = \alpha(s) \mid \edo(x)=s} \ge 1 - \frac{4N^3H^2\Delta}{\etamin^2\formargin},
\end{equation*}
provided $\Delta < \frac{\etamin^2\formargin}{N^2H^2}$.

Let $\vartheta = \frac{4N^3H^2\Delta}{\etamin^2\formargin}$. Then from~\pref{prop:tabular-bound} we learn a $\hat{\varphi}$ such that:
\begin{equation*}
    V(\pi^\star) - V(\hat{\varphi} \circ \phihat) \le \epsilon + 2H^2\vartheta,
\end{equation*}
with probability at least $1 - \delta - \nsamp(S, A, H, \epsilon, \delta)H\vartheta$. The failure probability $\delta - \nsamp(S, A, H, \epsilon, \delta)H\vartheta$ was when condition in~\pref{thm:bijection-forward} holds which holds with $\delta$ probability. Hence, total failure probability is:
\begin{equation*}
    2 \delta + \nsamp(S, A, H, \epsilon, \delta)H\vartheta.
\end{equation*}

We set $\delta$ both in our representation learning analysis and in PAC RL to $\delta_\circ/4$. We also set $\epsilon$ in the PAC RL algorithm to $\epsilon_\circ/2$. This means the PAC RL algorithm runs for $\nsamp(S, A, H, \epsilon_\circ/2, \delta_\circ/4)$ episodes.

We enforce $\vartheta \le \frac{\delta_\circ}{2\nsamp(S, A, H, \epsilon_\circ/2, \delta_\circ/4)H}$. Then the total failure probability becomes:
\begin{equation*}
2\delta_\circ/4) + \delta_\circ/4 + \delta_\circ/2 \le \delta_\circ    
\end{equation*}

We also enforce $2H^2\vartheta \le \epsilon_\circ/2$. The sub-optimality of the PAC RL policy is given by:
\begin{equation*}
    \epsilon_\circ/2 + \epsilon_\circ/2 \le \epsilon_\circ
\end{equation*}
This gives us our derived PAC RL bound. 

We now accumulate all conditions:
\begin{align*}
    \Delta &= \sqrt{\frac{2}{n}\ln\rbr{\frac{4|\Fcal||\Phi|}{\delta_\circ}}}\\
    \vartheta &= \frac{4N^3H^2\Delta}{\etamin^2\formargin}\\
    \Delta &< \frac{\etamin^2\formargin}{N^2H^2}\\
    \vartheta &\le \frac{\delta_\circ}{2\nsamp(S, A, H, \epsilon_\circ/2, \delta_\circ/4)H}\\
    2H^2\vartheta &\le \epsilon_\circ/2
\end{align*}
This simplifies to
\begin{align*}
    \Delta &\le \frac{\etamin^2\formargin}{N^2H^2}\\
    \Delta &\le \frac{\delta_\circ\etamin^2\formargin}{8N^3 H^3\nsamp(S, A, H, \epsilon_\circ/2, \delta_\circ/4)}\\
    \Delta &\le \frac{\epsilon_\circ \etamin^2 \formargin}{16N^3 H^4}
\end{align*}
Or,
\begin{equation*}
    n = \Ocal\rbr{\cbr{\frac{N^4H^4}{\etamin^4\formargin^2} + \frac{N^6H^8}{\epsilon^2_\circ\etamin^4\formargin^2} + \frac{N^6H^6\nsamp^2(S, A, H, \epsilon_\circ/2, \delta_\circ/4)}{\delta_\circ^2 \etamin^4\formargin^2}}\ln\rbr{\frac{|\Fcal||\Phi|}{\delta_\circ}}}
\end{equation*}
This completes the proof.
\end{proof}

\subsection{Upper Bound for the Temporal Contrastive Approach}

We first convert our video dataset $\Dcal$ into a dataset suitable for contrastive learning. We first split the datasets into $\lfloor n/2 \rfloor$ pairs of videos. For each video pair $\cbr{\rbr{x^{(2l)}_1, x^{(2l)}_2, \cdots, x^{(2k)}_H}, \rbr{x^{(2l+1)}_1, x^{(2l+1)}_2, \cdots, x^{(2l+1)}_H}}$, we create a tuple $(x, x', k, z)$ where $z \in \{0, 1\}$ as follows. As in forward modeling, we will either use a fixed value of $k$, or sample $k \in \unf([K])$. We denote this general distribution over $k$ by $\omega \in \Delta([K])$ which is either point mass, or $\unf([K])$. We sample $k \sim \omega$ and $z \sim \unf(\{0, 1\})$ and $h \in \unf([H])$. We set $x = x^{(2l)}_h$. If $z=1$, then we set $x' = x^{(2l)}_{h+k}$, otherwise, we sample $h' \sim \unf(\{0, 1\})$ and select $x' = x^{(2l)}_{h'}$. This way, we collect a dataset $\Dcal_\cont$ of $\lfloor n/2 \rfloor$ tuples $(x, k, x', z)$. We view a tuple $(x, k, x', z)$ as a \emph{real observation pair} when $z=1$, and a \emph{fake observation pair} when $z=0$. Note that our sampling process leads to all data points being iid.

We define the distribution $\distcont(X, k, X', Z)$ as the distribution over $(x, k, x', z)$. We can express this distribution as:
\begin{align*}
    \distcont(X=x, k, X'=x', Z=1) &= \frac{\omega(k)}{2H} \sum_{h=1}^{H}D(x=x_h, x'=x_{h+k}) \\
    &= \frac{\omega(k)}{2} \rho(x) D(x_{k+1}=x' \mid x_1=x)\\
    \distcont(X=x, X'=x', Z=0) &= \frac{\omega(k)}{2H^2}\sum_{h=1}^{H}D(x=x_h) \sum_{h'=1}^{H} D(x'=x_{h'})\\
    &= \frac{\omega(k)}{2} \rho(x) \rho(x')
\end{align*}
where we use the time homogeneity of $D$ and definition of $\rho$. We will use a shorthand to denote $D(x_{k+1}=x' \mid x_1=x)$ as $D(x' \mid x, k)$ in this analysis. It is easy to verify that $D(x' \mid x, k) = D(x' \mid \edo(x), k)$. The marginal distribution $\distcont(x, k, x')$ is given by:
\begin{equation}\label{eqn:marginal-tempcont-dist}
    \distcont(x, k, x') = \frac{\omega(k)\rho(x)}{2}\rbr{D(x' \mid x, k) + \rho(x')}
\end{equation}
Note that $\distcont(X)$ is the same as $\rho(X)$.

We will use $\distcont$ for any marginal and conditional distribution derived from $\distcont(X, k, X', Z)$. We assume a model class $\Gcal: \Xcal \times [K] \times \times [N] \rightarrow [0, 1]$ that we use for solving the prediction problem. We will also reuse the decoder class $\phi: \Xcal \rightarrow [N]$ that we defined earlier, and we will assume that $N = |\Scal|$. This can be relaxed by doing clustering or working with a different induced MDP (e.g., see the clustering algorithm in~\cite{misra2020kinematic}). However, this is not the main point of the analysis.

We define the expected risk minimizer of the squared loss problem below:
\begin{equation}\label{eqn:temp-cont}
    \ghat, \phihat = \arg\min_{g \in \Gcal, \phi \in \Phi} \frac{1}{\lfloor n/2 \rfloor} \sum_{i=1}^{\lfloor n/2 \rfloor} \rbr{g(\phi(x^{(i)}), k^{(i)}, x'^{(i)}) - z^{(i)}}^2
\end{equation}

We express the Bayes classifier of this problem below:
\begin{lemma}[Bayes Classifier]\label{lem:bayes-classifier} The Bayes classifier of the problem posed in~\pref{eqn:temp-cont} is given by $\distcont(z = 1 \mid x, k, x')$ which satisfies:
\begin{align*}
    \distcont(z = 1 \mid x, k, x') &= \frac{D(\edo(x') \mid \edo(x), k)}{D(\edo(x') \mid \edo(x), k) + \rho(\edo(x'))}.
\end{align*}   
\end{lemma}
\begin{proof} We can express the Bayes classifier as:
\begin{align*}
    \distcont(z = 1 \mid x, k, x') &= \frac{\distcont(x, k, x', z=1)}{\distcont(x, k, x', z=1) + \distcont(x, k, x', z=0)}\\
     &= \frac{\omega(k)/2 \rho(x) D(x' \mid x)}{\omega(k)/2 \rho(x) D(x' \mid x) + \omega(k)/2 \rho(x) \rho(x')}\\
     &= \frac{D(x' \mid x, k)}{D(x' \mid x, k) + \rho(x')}\\
     &= \frac{D(x' \mid \edo(x), k)}{D(x' \mid \edo(x), k) + \rho(x')}\\
     &= \frac{q(x' \mid \edo(x))D(\edo(x') \mid \edo(x), k)}{q(x' \mid \edo(x))D(\edo(x') \mid \edo(x), k) + q(x' \mid \edo(x))\rho(\edo(x'))}\\
     &= \frac{D(\edo(x') \mid \edo(x), k)}{D(\edo(x') \mid \edo(x), k) + \rho(\edo(x'))}.
\end{align*}
\end{proof}

\begin{assumption}[Realizability]\label{assum:temporal-contrastive-realizability} There exists $g^\star \in \Gcal$ and $\phi^\circ \in \Phi$ such that for all $(x, k, x') \in \supp~\distcont(X, k, X')$, we have $\distcont(z=1 \mid x, k, x') = g^\star(\phi^\circ(x), k, x')$.
\end{assumption}
We will use the shorthand to denote $g^\star(x, k, x') = g^\star(\phi^\circ(x), k, x')$.

As before, we start with typical square loss guarantees in the realizable setting.
\begin{theorem} Fix $\delta \in (0, 1)$. Under realizability (\pref{assum:temporal-contrastive-realizability}), the ERM solution of $\fhat, \phihat$ in~\cref{eqn:temp-cont} satisfies: 
    \begin{equation*}
        \EE_{(x, k, x') \sim \distcont}\sbr{\rbr{\ghat(\phihat(x), k, x') - g^\star(x, k, x')}^2} \le \deltacont^2 = \frac{2}{n} \ln \frac{|\Gcal|.|\Phi|}{\delta}
    \end{equation*}
\end{theorem}
For proof see Proposition 12 in ~\cite{misra2020kinematic}.

We will prove a coupling result similar to the case for forward modeling. However, to do this, we need to define a coupling distribution:
\begin{equation*}
    \dcoup(X_1=x_1, X_2=x_2, k, X'=x') = \omega(k)\distcont(X=x_1) \distcont(X=x_2) \distcont(X'=x')
\end{equation*}

We will derive a useful importance ratio bound.
\begin{equation}\label{eqn:coupling-ratio}
    \frac{\dcoup(x_1, k, x')}{\distcont(x_1, k, x')} = \frac{2\rho(x_1) \rho(x')}{\rho(x_1) D(x' \mid x_1, k) + \rho(x_1) \rho(x')} \le 2
\end{equation}

We now prove an analogous result to~\pref{prop:coupling-forward-modeling}.

\begin{theorem}[Coupling for Temporal Contrastive Learning]\label{thm:coupling-temporal-contrastive} With probability at least $1-\delta$ we have:
    \begin{equation*}
        \EE_{(x_1, x_2, k, x') \sim \dcoup}\sbr{\one\cbr{\phihat(x_1) = \phihat(x_2)}\abr{g^\star(x_1, k, x') - g^\star(x_2, k, x')}} < 4 \deltacont(n, \delta)
    \end{equation*}
\end{theorem}
\begin{proof} We start with triangle inequality:
\begin{align*}
    &\EE_{(x_1, x_2, k, x') \sim \dcoup}\sbr{\one\cbr{\phihat(x_1) = \phihat(x_2)}\abr{g^\star(x_1, k, x') - g^\star(x_2, k, x')}} \\
    &\le \EE_{(x_1, x_2, k, x') \sim \dcoup}\sbr{\one\cbr{\phihat(x_1) = \phihat(x_2)}\abr{g^\star(x_1, k, x') - \ghat(\phihat(x_1), k, x')}} + \\
    &\qquad \EE_{(x_1, x_2, k, x') \sim \dcoup}\sbr{\one\cbr{\phihat(x_1) = \phihat(x_2)}\abr{\ghat(\phihat(x_1), k, x') - g^\star(x_2, k, x')}}
\end{align*}
We bound the first term as:
\begin{align*}
    &\EE_{(x_1, x_2, k, x') \sim \dcoup}\sbr{\one\cbr{\phihat(x_1) = \phihat(x_2)}\abr{g^\star(x_1, k, x') - \ghat(\phihat(x_1), k, x')}}\\
    &\le \underbrace{\sqrt{\EE_{(x_1, x_2, k, x') \sim \dcoup}\sbr{\one\cbr{\phihat(x_1) = \phihat(x_2)}}}}_{\defeq b} \cdot \sqrt{\EE_{(x_1, x_2, k, x') \sim \dcoup}\sbr{\abr{g^\star(x_1, k, x') - \ghat(\phihat(x_1), k, x')}^2}}\\
    &= b \sqrt{\EE_{(x_1, k, x') \sim \dcoup}\sbr{\rbr{g^\star(x_1, k, x') - \ghat(\phihat(x_1), k, x')}^2}}\\
    &= b \sqrt{\EE_{(x_1, k, x') \sim \distcont}\sbr{\frac{\dcoup(x_1, k, x')}{\distcont(x_1, k, x')}\rbr{g^\star(x_1, k, x') - \ghat(\phihat(x_1), k, x')}^2}}\\
    &\le b \sqrt{2\EE_{(x_1, k, x') \sim \distcont}\sbr{\rbr{g^\star(x_1, k, x') - \ghat(\phihat(x_1), k, x')}^2}}\\
    &\le \sqrt{2} b\deltacont,
\end{align*}
where we use Cauchy-Schwartz's inequality in the first step and~\pref{eqn:coupling-ratio} in the second inequality.
The second term is bounded as:
\begin{align*}
    &\EE_{(x_1, x_2, k, x') \sim \dcoup}\sbr{\one\cbr{\phihat(x_1) = \phihat(x_2)}\abr{\ghat(\phihat(x_1), k, x') - g^\star(x_2, k, x')}}\\
    &=\EE_{(x_1, x_2, k, x') \sim \dcoup}\sbr{\one\cbr{\phihat(x_1) = \phihat(x_2)}\abr{\ghat(\phihat(x_2), k, x') - g^\star(x_2, k, x')}}\\
    &=\EE_{(x_1, x_2, k, x') \sim \dcoup}\sbr{\one\cbr{\phihat(x_1) = \phihat(x_2)}\abr{\ghat(\phihat(x_1), k, x') - g^\star(x_1, k, x')}}\\
    &\le \sqrt{2} b\deltacont,
\end{align*}
where we use the coupling argument in the first step and then reduce it to the first term using symmetric of $(x_1, x_2)$ in $\dcoup$. Combining the upper bounds of the two terms and using $b \le 1$ and $2\sqrt{2} < 4$ completes the proof.
\end{proof}

\begin{assumption}[Temporal Contrastive Margin] We assume that there exists a $\tempmargin >0$ such that for any two different states $s_1$ and $s_2$:
    \begin{equation*}
        \frac{1}{2}\EE_{k \sim \omega, s' \sim \rho}\sbr{\abr{g^\star(s_1, k, s') - g^\star(s_2, k, s')}} \ge \tempmargin
    \end{equation*} 
\end{assumption}
The factor of $\frac{1}{2}$ is chosen for comparison with forward modeling as will become clear later at the end of the proof. As before, if $k$ is fixed, the margin is given by 
\begin{equation*}
        \tempmargin^{(k)} \defeq \frac{1}{2} \inf_{s_1 \ne s_2; s_1, s_2 \in \Scal} \EE_{s' \sim \rho}\sbr{\abr{g^\star(s_1, k, s') - g^\star(s_2, k, s')}}
    \end{equation*} 
and when $k \sim \unf([K])$ the margin is given by
\begin{equation*}
        \tempmargin^{(u)} \defeq \frac{1}{2} \inf_{s_1 \ne s_2; s_1, s_2 \in \Scal} \EE_{k \sim \unf([K]), s' \sim \rho}\sbr{\abr{g^\star(s_1, k, s') - g^\star(s_2, k, s')}}
    \end{equation*} 
We directly have $\tempmargin^{(u)} \ge \frac{1}{K}\sum_{k=1}^K \tempmargin^{(k)}$.
    
\begin{lemma}
   \begin{equation*}
       \PP_{x_1, x_2 \sim \rho}\rbr{\phihat(x_1) = \phihat(x_2) \land \edo(x_1) \ne \edo(x_2)} \le \frac{2\deltacont(n, \delta)}{\tempmargin}
   \end{equation*} 
\end{lemma}
\begin{proof}
    We start with the left-hand side in~\cref{thm:coupling-temporal-contrastive}.
    \begin{align*}
        &\EE_{(x_1, k, x_2, x') \sim \dcoup}\sbr{\one\cbr{\phihat(x_1) = \phihat(x_2)}\abr{g^\star(x_1, k, x') - g^\star(x_2, k, x')}}\\
        &=\EE_{(x_1, x_2) \sim \dcoup}\sbr{\one\cbr{\phihat(x_1) = \phihat(x_2)} \EE_{k \sim \omega, x' \sim \rho}\sbr{\abr{g^\star(x_1, k, x') - g^\star(x_2, k, x')}}}\\
        &=\EE_{(x_1, x_2) \sim \rho}\sbr{\one\cbr{\phihat(x_1) = \phihat(x_2)} \EE_{k \sim \omega, s' \sim \rho}\sbr{\abr{g^\star(x_1, k, s') - g^\star(x_2, k, s')}}}\\
        &\ge 2\tempmargin \EE_{(x_1, x_2) \sim \rho}\sbr{\one\cbr{\phihat(x_1) = \phihat(x_2) \land \edo(x_1) \ne \edo(x_2)}}\\
        &= 2\tempmargin \PP_{(x_1, x_2) \sim \rho}\sbr{\phihat(x_1) = \phihat(x_2) \land \edo(x_1) \ne \edo(x_2)},
    \end{align*}
    where we use the definition of $\tempmargin$, the fact that marginal over $\dcoup(X)$ is $\rho$, and that $g^\star(x, k, x')$ only depends on $\edo(x')$ and $\edo(x)$ (\pref{lem:bayes-classifier}). Combining with the inequality proved in~\cref{thm:coupling-temporal-contrastive}, completes the proof.
\end{proof}

We have now reduced this analysis to an almost identical one to the forward analysis case (\pref{prop:coupling-forward-modeling}). We can, therefore, use the same steps and derive identical bounds. All what changes is that $\formargin$ is replaced by $\tempmargin$ and in $\Delta$ we replace $\ln|\Fcal|$ with $\ln|\Gcal|$. At this point, we can clarify that the factor of $\frac{1}{2}$ was chosen in the definition of $\tempmargin$ so that $\formargin$ can be replaced by $\tempmargin$ rather than $\frac{\tempmargin}{2}$ which will make it harder to compare margins, as we will do later.

\subsection{Proof of Lower Bound for Exogenous Block MDPs}
\label{sec:lb_proof}

\begin{proof}[\cpfname{thm:exo_lower_bound}]

We present a hard instance using a family of exogenous block MDPs, with $H = 2$, $\Acal = \{1,2\}$, and a single binary endogenous factor and $d-1$ exogenous binary factors for each level, where each endogenous and exogenous factor. We first fix an absolute constant $p \in [0,1]$.

Each MDP $M_i$ is indexed by $i \in [d]$, and is specified as follows:
\begin{itemize}
    \item \textbf{State space:} The state is represented by $x_h \coloneqq [s_h^\fct{1},s_h^\fct{2},\dotsc,s_h^\fct{d}]$, where the superscript denotes different factors. For MDP $M_i$, only the $i$-th factor $s_h^\fct{i}$ is an endogenous state for all $h$, and the other factors are exogenous. Each factor has values of $\{0, 1\}$.
    \item \textbf{Transition:} For the MDP instance $M_i$: it has
    \begin{enumerate}
        \item For the $i$-th factor (endogenous factor), $\PP(s_2^\fct{i} \mid s_1^\fct{i}, a) = \1[s_2^\fct{i} = \1(s_1^\fct{i} = a)]$. That is, the endogenous states have deterministic dynamics. If $s_1^\fct{i} = a$, then it transitions to $s_2^\fct{i} = 1$, otherwise it transitions to $s_2^\fct{i} = 0$.
        \item For the $j$-th factor with $j \neq i$ (exogenous factor), $\PP(s_2^\fct{j} \mid s_1^\fct{j}) = (1-p)\1(s_2^\fct{j} = s_1^\fct{j}) + p \1(s_2^\fct{j} \neq s_1^\fct{j})$ for any $s_2^\fct{j}$ and $s_1^\fct{j}$. That is, the $j$-th factor has probability of $1-p$ of transiting to the same state (i.e., $s_1^\fct{j} = 0 \to s_2^\fct{j} = 0$ or $s_1^\fct{j} = 1 \to s_2^\fct{j} = 1$), and probability of $p$ of transiting to the different state (i.e., $s_1^\fct{j} = 0 \to s_2^\fct{j} = 1$ or $s_1^\fct{j} = 1 \to s_2^\fct{j} = 0$).
    \end{enumerate}
    Note that the MDP terminates at $h = 2$. 
    \item \textbf{Initial state distribution and reward:}
    The marginal distribution of $s_1^\fct{j}$ is uniformly distributed at random over $\{0,1\}$ for all $j \in [d]$, and all factors are independent from each other. For MDP $M_i$, the agent only receive reward signal after taking action at $h = 2$, with $R(s_2^\fct{i},a) = s_2^\fct{i}$. That is, it always reward $1$ at $s_2^\fct{i} = 1$ and reward $0$ at $s_2^\fct{i} = 0$ no matter which action it takes.
    \item \textbf{Data collection policy for video data:} We assume that the data collection policy always pick action $0$ with probability $p$ and action $1$ with probability $1-p$ for all states.
\end{itemize}

Now we use the following two steps to establish the proof.

\paragraph{Uninformative video data for learning the state decoder}

Since video data only contains state information, from the MDP family construction above, we can easily verify that all MDP instances in such a family will have an identical video data distribution, {\em regardless of the choice of constant $p$}. This implies that the video data is uninformative for the agent to distinguish the MDP instance from the MDP family. 
Now, we assume $\Dcal^\fct{i}$ is the video data from the instance $M^\fct{i}$, and $\phi^\fct{i}$ is the state decoder learned from an arbitrary algorithm $\mathscr A_1$ with $\Dcal^\fct{i}$. Then, for any arbitrary algorithm $\mathscr A_2$ that uses the state decoder $\phi^\fct{i}$ in its execution, it is equivalent to such an $\mathscr A_2$ that uses the state decoder $\phi^\fct{j}$ in its execution, where $j$ can be selected arbitrarily from $[d]$.

\paragraph{State decoder requiring exponential length}

Without loss of generality, we further restrict the state decoder $\phi$ used in the execution of $\mathscr A_2$ for all MDP instance to be some $\phi_h: \Xcal \to [L]$, where $h \in \{1,2\}$ and $L \leq 2^d$. Then we will argue that there must exists a $k \in [d]$, such that 
\begin{align}
\label{eq:incorrect_prob}
    \sum_{x_1,\widetilde x_1 \in \Xcal} \PP\left(\phi_1(x_1) = \phi_1(\widetilde x_1) \vee \left( s_1^\fct{k} \neq \widetilde s_1^\fct{k}\right)\right) > \frac{2^d - L}{d 2^d},
\end{align}
where $x_1 \coloneqq [s_1^\fct{1},s_1^\fct{2},\dotsc,s_1^\fct{d}]$ and $\widetilde x_1 \coloneqq [\widetilde s_1^\fct{1}, \widetilde s_1^\fct{2},\dotsc,\widetilde s_1^\fct{d}]$. Note that, \cref{eq:incorrect_prob} means there must be a probability of at least $\nicefrac{2^d - L}{d 2^d}$ that $\phi_1$ will incorrectly group two different $s_1^\fct{k}$ together. 

We now prove \cref{eq:incorrect_prob}. Based on the construct above, we know that $|\Xcal| = 2^d$, and each state in $\Xcal$ has the same occupancy for $x_1$ based on the defined initial state distribution (this holds for all instances in the MDP family, as we are now only talking about the initial state $x_1$). Thus, we have
\begin{align}
\label{eq:corr_prob}
\sum_{x_1,\widetilde x_1 \in \Xcal} \PP\left[ \phi_1(x_1) = \phi_1(\widetilde x_1) \vee \left( s_1^\fct{1} = \widetilde s_1^\fct{1} \right) \vee \left( s_1^\fct{2} = \widetilde s_1^\fct{2} \right) \vee \cdots \vee \left( s_1^\fct{d} = \widetilde s_1^\fct{d}\right) \right] \leq \frac{L}{2^d},
\end{align}
because we defined $\phi_1: \Xcal \to [L]$, it means that such $\phi_1$ is only able to distinguish the number of $L$ different states from $\Xcal$. Then, we obtain
\begin{align*}
&~ \sum_{j \in [d]} \sum_{x_1,\widetilde x_1 \in \Xcal} \PP\left(\phi_1(x_1) = \phi_1(\widetilde x_1) \vee \left( s_1^\fct{j} \neq \widetilde s_1^\fct{j}\right)\right)
\\
= &~ \sum_{x_1,\widetilde x_1 \in \Xcal} \PP\left(\phi_1(x_1) = \phi_1(\widetilde x_1) \right)
\\
&~ - \sum_{x_1,\widetilde x_1 \in \Xcal} \PP\left[ \phi_1(x_1) = \phi_1(\widetilde x_1) \vee \left( s_1^\fct{1} = \widetilde s_1^\fct{1} \right) \vee \left( s_1^\fct{2} = \widetilde s_1^\fct{2} \right) \vee \cdots \vee \left( s_1^\fct{d} = \widetilde s_1^\fct{d}\right) \right]
\\
= &~ \frac{2^d - L}{2^d}. \tag{by \cref{eq:corr_prob}}
\\
\Longrightarrow &~ \max_{j \in [d]}\sum_{x_1,\widetilde x_1 \in \Xcal} \PP\left(\phi_1(x_1) = \phi_1(\widetilde x_1) \vee \left( s_1^\fct{j} \neq \widetilde s_1^\fct{j}\right)\right) > \frac{2^d - L}{d 2^d}.
\end{align*}
So this proves \cref{eq:incorrect_prob}.

From \cref{eq:incorrect_prob}, we know that for the MDP instance $M^\fct{k}$, $\phi_{1}$ will have probability at least $\nicefrac{2^d - L}{2 \cdot d 2^d}$ to mistake the endogenous state, which implies that for any policy that is represented using the state decoder $\phi$, it must have sub-optimality at least $\nicefrac{2^d - L}{2 \cdot d 2^d}$. Therefore, it is easy to verify that, for any $\varepsilon > 0$, we can simply pick $d = \nicefrac{1}{4\varepsilon}$, and obtain
\begin{align*}
\text{sub-optimality} > \frac{2^d - L}{2 \cdot d 2^d} \geq \varepsilon, \quad \forall L \leq 2^{\nicefrac{1}{4\varepsilon}-1}.
\end{align*}
Then, any arbitrary algorithm $\mathscr A_2$ that uses the state decoder $\phi$ in its execution, where $\phi_h: \Xcal \to [L]$ can be chosen arbitrarily for $h \in \{1,2\}$ and $L \leq 2^{\nicefrac{1}{4\varepsilon}-1}$, must have sub-optimality larger than $\varepsilon$.

\paragraph{Additional characteristics of MDP family and video data}
Note that, by combining the arguments of uninformative video data and a state decoder requiring exponential length, we obtain impossible results. We now discuss the following:
\begin{enumerate}
    \item The margin condition defined in \cref{assum:margin} regarding the constructed MDPs
    \item The PAC learnability of the constructed MDPs
    \item The coverage condition of video data.
\end{enumerate}

\newcommand{\Pfor}{\PP_{\textrm{for}}}

For the defined margin condition of forward modeling, we have: for the MDP instance $M_i$ with constant $p$, we can bound the forward margin as below ($\Pfor$ denotes the video distribution)
\begin{align*}
&~ \TV{\Pfor(X_2 \mid s_1^\fct{i} = 0) - \Pfor(X_2 \mid s_1^\fct{i} = 1)}
\\
= &~ \frac{1}{2} \sum_{X_2} \left| \Pfor(X_2 \mid s_1^\fct{i} = 0) - \Pfor(X_2 \mid s_1^\fct{i} = 1) \right|
\\
= &~ \frac{1}{2} \sum_{X_2} \big| \Pfor(s_2^\fct{i} = 0 \mid s_1^\fct{i} = 0) \PP(X_2 \mid s_2^\fct{i} = 0) + \Pfor(s_2^\fct{i} = 1 \mid s_1^\fct{i} = 0) \PP(X_2 \mid s_2^\fct{i} = 1)
\\
&~\qquad - \Pfor(s_2^\fct{i} = 0 \mid s_1^\fct{i} = 1) \PP(X_2 \mid s_2^\fct{i} = 0) + \Pfor(s_2^\fct{i} = 1 \mid s_1^\fct{i} = 1) \PP(X_2 \mid s_2^\fct{i} = 1) \big|
\\
= &~ \frac{1}{2} \sum_{X_2} \left| (1 - 2p) \left[ \PP(X_2 \mid s_2^\fct{i} = 0) - \PP(X_2 \mid s_2^\fct{i} = 1) \right] \right|.
\\
\overset{\text(a)}{=} &~ \frac{|1 - 2p|}{2} \sum_{X_2}  \PP(X_2 \mid s_2^\fct{i} = 0) + \frac{|1 - 2p|}{2} \sum_{X_2}  \PP(X_2 \mid s_2^\fct{i} = 1)
\\
= &~ |1 - 2p|,
\end{align*}
where step (a) is because $s_2^\fct{i}$ is a part of $X_2$, and then we know $\PP(X_2 \mid s_2^\fct{i} = 0)$ and $\PP(X_2 \mid s_2^\fct{i} = 1)$ cannot be nonzero simultaneously. So picking $p \neq 0.5$ implies positive forward margin.

For the temporal contrastive learning, it is easy to verify that $|\Pfor(z=1 \mid s_1^\fct{i}=1, X_2) - \Pfor(z=1 \mid s_1^\fct{i}=1, X_2)| = |1 - 2p|$, so picking $p \neq 0.5$ also implies positive margin for temporal contrastive learning.

As for the PAC learnability, since the latent dynamics of our constructed MDPs are deterministic, they are provably PAC learnable by \citet{efroni2022ppe}.

As for the coverage property of the video data, it is easy to verify
\begin{align*}
\max_{\pi \in \Pi, x_1 \in \Xcal} \frac{\PP_\pi(x_1,a_1)}{\Pfor(x_1,a_1)} = \max_{\pi \in \Pi, x_2 \in \Xcal} \frac{\PP_\pi(x_2)}{\Pfor(x_2)} = \max\left\{\nicefrac{1}{p},\nicefrac{1}{1-p}\right\}.
\end{align*}
Therefore, we can simply pick $p = \nicefrac{1}{3}$ and obtain the desired MDP and video data properties. This completes the proof.
\end{proof}

\paragraph{Addition remark of \cref{thm:exo_lower_bound}}
In the proof of \cref{thm:exo_lower_bound}, if we pick $p = 0.5$ for that hard instance, the constructed MDP family reduces to a block MDP without exogenous noise, but the margin becomes $0$ for both forward modeling and temporal contrastive learning. Therefore, it implies that either the exogenous noise or zero forward margin could make the learnability of the problem impossible.

\subsection{Relation Between Margins}
\label{app:margin-relations}

We defined margins $\formargin$ for forward modeling and $\tempmargin$ for temporal contrastive learning. The larger the values of these margins, the more easy it is to separate observations from different endogenous states. This can be directly inferred from the sample complexity bounds which scale inversely with these margins. In particular, both $\formargin$ and $\tempmargin$ depend on the way we sample the multi-step variable $k$. We consider two special cases: one where $k \in [K]$ is fixed, we instantiate these margins as $\formargink$ and $\tempmargink$, and second where $k$ is uniformly sampled from $[K]$ and we instantiate those margins as $\formarginu$ and $\tempmarginu$.

A natural question is how these margins are related. The sample complexity bounds of forward modeling and temporal contrastive are almost identical except for the difference in margins ($\formargin$ vs $\tempmargin$) and the function classes ($\Fcal$ vs $\Gcal$). If the function classes were of similar complexity, then having a larger margin will make it easier to learn the right representation.\footnote{This inference has to be made with a caveat that since we are comparing upper bounds, we cannot guarantee this to hold.}

\begin{theorem}[Margin Relation]\label{thm:margin-relation} For any Block MDP and $K \in \NN$, the margins $\formargink, \formarginu, \tempmargink, \tempmarginu > 0$ are related as:
    \begin{align*}
        &\frac{1}{K} \formargink \le \formarginu \\
        &\frac{1}{K} \tempmargink \le \tempmarginu \\
        &\frac{\etamin^2}{4H^2}\formargink \le \tempmargink \le \formargink \\
        &\frac{\etamin^2}{4H^2} \formarginu \le \tempmarginu \le \formarginu. 
    \end{align*}
\end{theorem}
\begin{proof}
    We first prove the first two relations. Fix any $k \in [K]$ then,
    \begin{align*}
        \formarginu &= \inf_{s_1 \ne s_2, s_1, s_2 \in \Scal}\EE_{k' \sim \unf([K])}\sbr{\TV{\pairdist(X' \mid s_1, k') - \pairdist(X' \mid s_2, k')}}, \\
        &\ge \frac{1}{K} \sum_{k'=1}^K \inf_{s_1 \ne s_2, s_1, s_2 \in \Scal}\TV{\pairdist(X' \mid s_1, k') - \pairdist(X' \mid s_2, k')},\\
        &\ge \frac{1}{K} \inf_{s_1 \ne s_2, s_1, s_2 \in \Scal}\TV{\pairdist(X' \mid s_1, k) - \pairdist(X' \mid s_2, k)},\\
        &= \frac{1}{K} \formargink.
    \end{align*}
Similarly, 
\begin{align*}
        \tempmargin^{(u)} &= \frac{1}{2}\inf_{s_1 \ne s_2, s_1, s_2 \in \Scal}\EE_{k' \sim \unf([K]), s' \sim \rho}\sbr{\abr{g^\star(s_1, k', s') - g^\star(s_2, k', s')}},\\
        &\ge \frac{1}{2K} \sum_{k'=1}^K \inf_{s_1 \ne s_2, s_1, s_2 \in \Scal}\EE_{s' \sim \rho}\sbr{\abr{g^\star(s_1, k', s') - g^\star(s_2, k', s')}},\\
        &\ge \frac{1}{2K}\inf_{s_1 \ne s_2, s_1, s_2 \in \Scal}\EE_{s' \sim \rho}\sbr{\abr{g^\star(s_1, k, s') - g^\star(s_2, k, s')}},\\
        &= \frac{1}{K} \tempmargink.
    \end{align*}
We now prove the next two relations. We will prove these bounds for a generic distribution $\omega \in \Delta([K])$ over $k$. Recall that $\omega$ is point-mass over $k$ for $\tempmargink$ and $\unf([K])$ for $\tempmarginu$. We denote our generic margins as $\formargin$ and $\tempmargin$ for $k \sim \omega$. We use a shorthand notation $W_k(s, s') = \frac{\rho(s')}{\pairdist(s' \mid s, k) + \rho(s')}$ for a given pair of states $s, s'$ and integer $k \in [K]$. It is easy to see that $W_k(s, s') \le 1$ as $\pairdist(s' \mid s, k), \rho(s') \in (0, 1]$. Further, we have $W_k(s, s') \ge \frac{\rho(s')}{2} \ge \frac{\etamin}{2H}$ where we use $\pairdist(s' \mid s, k), \rho(s') \in (0, 1]$, and~\pref{eqn:rho-bound}.

We have $g^\star(s, k, s') = \distcont(z=1 \mid s, k, s') = g^\star(s, k, s') = \frac{\pairdist(s' \mid s, k)}{\pairdist(s' \mid s, k) + \rho(s')}$ using the definition of $\distcont$ in~\pref{lem:bayes-classifier} and~\pref{assum:temporal-contrastive-realizability}. We can use the shorthand $W_k$ and the definition of $g^\star$ to show
\begin{align}
    \tempmargin &= \frac{1}{2}\inf_{s_1 \ne s_2, s_1, s_2 \in \Scal}\EE_{k \sim \omega, s' \sim \rho}\sbr{\abr{g^\star(s_1, k, s') - g^\star(s_2, k, s')}}, \nonumber \\
    &= \frac{1}{2}\inf_{s_1 \ne s_2, s_1, s_2 \in \Scal}\sum_{k=1}^K \omega(k)\sum_{s' \in \Scal} \rho(s') \abr{g^\star(s_1, k, s') - g^\star(s_2, k, s')}, \nonumber\\
    &= \frac{1}{2}\inf_{s_1 \ne s_2, s_1, s_2 \in \Scal}\sum_{k=1}^K \omega(k)\sum_{s' \in \Scal} W_k(s_1, s') W_k(s_2, s') \abr{\pairdist(s' \mid s_1, k) - \pairdist(s' \mid s_2, k)}.\label{eqn:temp-for-margin}
\end{align}

As $W_k(s_1, s') \le 1$ and $W_k(s_2, s') \le 1$ we have
\begin{align*}
    \formargin &= \frac{1}{2}\inf_{s_1 \ne s_2, s_1, s_2 \in \Scal}\sum_{k=1}^K \omega(k)\sum_{s' \in \Scal} \underbrace{W_k(s_1, s')}_{\le 1} \underbrace{W_k(s_2, s')}_{\le 1} \abr{\pairdist(s' \mid s_1, k) - \pairdist(s' \mid s_2, k)},\\
    & \le \frac{1}{2}\inf_{s_1 \ne s_2, s_1, s_2 \in \Scal}\sum_{k=1}^K \omega(k)\sum_{s' \in \Scal} \abr{\pairdist(s' \mid s_1, k) - \pairdist(s' \mid s_2, k)},\\
    &= \inf_{s_1 \ne s_2, s_1, s_2 \in \Scal} \EE_{k \sim \omega}\sbr{\TV{\pairdist(s' \mid s_1, k) - \pairdist(s' \mid s_2, k)}} \\
    &= \formargin.
\end{align*}

This gives us $\tempmargink \le \formargink$ and $\tempmarginu \le \formarginu$. 
Finally, we prove the lower bounds. Starting from~\pref{eqn:temp-for-margin} and using $W_k(s_1, s') \ge \frac{\etamin}{2H}$ and $W_k(s_2, s') \le \frac{\etamin}{2H}$ we get the following:
\begin{align*}
    \formargin &= \frac{1}{2}\inf_{s_1 \ne s_2, s_1, s_2 \in \Scal}\sum_{k=1}^K \omega(k)\sum_{s' \in \Scal} \underbrace{W_k(s_1, s')}_{\ge \nicefrac{\etamin}{2H}} \underbrace{W_k(s_2, s')}_{\ge \nicefrac{\etamin}{2H}} \abr{\pairdist(s' \mid s_1, k) - \pairdist(s' \mid s_2, k)},\\
    & \ge \frac{\etamin^2}{4H^2} \cdot \frac{1}{2}\inf_{s_1 \ne s_2, s_1, s_2 \in \Scal}\sum_{k=1}^K \omega(k)\sum_{s' \in \Scal} \abr{\pairdist(s' \mid s_1, k) - \pairdist(s' \mid s_2, k)},\\
    &= \frac{\etamin^2}{4H^2} \inf_{s_1 \ne s_2, s_1, s_2 \in \Scal} \EE_{k \sim \omega}\sbr{\TV{\pairdist(s' \mid s_1, k) - \pairdist(s' \mid s_2, k)}} \\
    &= \frac{\etamin^2}{4H^2} \formargin.
\end{align*}
This gives us $\tempmargink \ge \frac{\etamin^2}{4H^2}\formargink$ and $\tempmarginu \ge \frac{\etamin^2}{4H^2}\formarginu$ which completes the proof.
\end{proof}

The main finding of the above theorem is that forward modeling has a higher margin than temporal contrastive learning. However, typically the function class used for forward modeling has a higher statistical complexity than those for temporal contrastive learning as the latter is solving a simpler binary classification problem than generating an observation.

\subsection{Why temporal contrastive learning is more susceptible to exogenous noise than forward modeling}
\label{app:temp-cont-fails-forward-wins}

\pref{thm:exo_lower_bound} shows that in the presence of exogenous noise, no video-based representation learning approach can be efficient in the worst case. However, this result only presents a worst-case analysis. In this section, we show an instance-dependent analysis. The main finding is that the temporal contrastive approach is very susceptible to even the smallest amount of exogenous noise, while forward modeling is more robust to the presence of exogenous noise. However, both approaches fail when there is a significant amount of exogenous noise, consistent with~\pref{thm:exo_lower_bound}.

\paragraph{Problem Instance.} We consider a Block MDP with exogenous noise with a state space of $\Scal = \{0, 1\}$, action space of $\Acal = \{0, 1\}$ and exogenous noise space of $\xi = \{0, 1\}$. We consider $H=1$ with a uniform distribution over $s_1$ and $\xi_1$, i.e., the start state $s_1$ and the start exogenous noise variable $\xi_1$ are chosen uniformly from $\{0, 1\}$. The transition dynamics are deterministic and given as follows: given action $a _1\in \{0, 1\}$ and state $s_1 \in \{0, 1\}$, we deterministically transition to $s_2=1-s_1$ if $s_1=a_1$, otherwise, we remain in $s_2=s_1$. The exogenous noise variable deterministically transitions from $\xi_1$ to $\xi_2 = 1-\xi_1$. The reward function is given by $R(s_2, s_1) = \one\{s_2=s_1\}$. We use the indicator notation $\one\{\Ecal\}$ to denote $1$ if the condition $\Ecal$ is true and 0 otherwise. The observation space is given by $\Xcal = \{0, 1\}^{m+2}$ where $(m+2)$ is the dimension of observation space. Given the endogenous state $s$ and exogenous noise $\xi$, the environment generates an observation stochasticaly as $x = [\xi, v_1, \cdots, v_l, w_1, \cdots, w_{m-l}, s]$ where $v_i \sim \psamp(\cdot \mid \xi)$ and $w_j \sim \psamp(\cdot \mid s)$ for all $i \in [l]$ and $j \in [m-l]$. The distribution $\psamp(u \mid s)$ generates $u=s$ with a probability 0.8 and $u=1-s$ with a probability 0.2. 
The hyperparameter $l$ is a fixed integer controlling what portion of the observation is generated by the exogenous noise compared to the endogenous state. If $l=1$, we only have a small amount of exogenous noise, while if $l = m-1$ we have the maximal amount of exogenous noise. The state $s$ and exogenous noise $\xi$ are both decodable from the observation $x$. The optimal policy achieves a return of 1 and takes action $a_1=1$ if $s_1=0$ and $a_1=0$ if $s_1=1$. As the optimal policy depends on the value of $s_1$, we must learn the latent state to realize the optimal policy.

\paragraph{Learning Setting.} We assume a decoder class $\Phi = \{\edo, \exo\}$ consisting of the true decoder $\edo$ and the incorrect decoder $\exo$ which maps observation to the exogenous noise $\xi$. Both decoders take an observation and map it to a value in $\{0, 1\}$. We assume access to an arbitrarily large dataset $\Dcal$ consisting of tuples $(x_1, x_2)$ collecting iid using a fixed data policy $\pidata$. This policy takes action $a_1=0$ in $s_1=0$ and action $a_1=1$ in $s_1=1$. Let $D(x_1, x_2)$ be the data distribution induced by $\pidata$. We will use $D$ to define other distributions induced by $D(x_1, x_2)$, for example $D(x_2)$ or $D(s_2)$.  We also assume access to two model classes $\Fcal: \{0, 1\} \rightarrow \Delta(\Xcal)$ and $\Gcal: \{0, 1\}^2 \rightarrow [0, 1]$. We assume these model classes are finite and contain certain constructions that we define later.

\paragraph{Overview:} As we increase the value of $l$, the amount of exogenous noise in the environment increases. We will prove that irrespective of the value of $l$, temporal contrastive learning assigns the same loss for both the correct decoder $\edo$ and the incorrect decoder $\exo$. In contrast, the forward modeling approach is able to prefer $\edo$ over $\exo$ when the noise is limited, specifically, when $l < m/2$. This will establish that temporal contrastive is very susceptible to exogenous noise whereas forward modeling is more robust. However, both approaches provably fail when there is $l \ge m/2$.

As we have $H=1$, we will denote $x_2, s_2, \xi_2$ by $x', s', \xi'$ and $x_1, s_1, \xi_1$ by $x, s, \xi$ respectively. Note that unless specified otherwise, $s$ and $\xi$ are the endogenous state and exogenous noise of the observation $x$. Similarly, $s'$ and $\xi'$ are the endogenous state and exogenous noise of $x'$. We will also use a shorthand $q(x')$ to denote the emission probability $q(x' \mid \exo(x'), \edo(x'))$ given its endogenous state and exogenous noise. We first state the conditional data distribution $D(x' \mid x)$.
\begin{align}
    D(x' \mid x) &=  q(x') T_\xi\rbr{\xi' \mid \xi} \sum_{a \in \Acal} T(s' \mid s, a) \pidata(a \mid s), \nonumber\\ 
    &= q(x') \one\cbr{\xi' = 1-\xi} \one\cbr{s'=1-s}\label{eqn:conditional-instance},
\end{align}
where we use $T_\xi(\xi' \mid \xi) = \one\cbr{\xi' = 1- \xi}$ and $\sum_{a \in \Acal} T(s' \mid s, a) \pidata(a \mid s) = \one\cbr{s'=1-s}$ which follows from the definition of $\pidata$. Note that $D(x' \mid x)$ only depends on $x$ via $s, \xi$, therefore, we can define $D(x' \mid x) = D(x' \mid s, \xi)$.

Let $\xtil$ be an observation variable with endogenous state $\stil$ and exogenous noise $\xitil$, i.e., $\stil = \edo(\xtil)$ and $\xitil = \exo(\xtil)$. We use this to derive the marginal data distribution $\rho$ over $x'$ as follows:
\begin{align}
    \rho(x') =  \sum_{s, \xi \in \{0, 1\}} D(x', s, \xi) &= \sum_{s, \xi \in \{0, 1\}} D(x' \mid s, \xi) \mu(s) \mu_\xi(\xi), \nonumber \\
    &= \frac{q(x')}{4} \sum_{s, \xi \in \{0, 1\}} \one\cbr{\xi' = 1-\xi} \one\cbr{s'=1-s}, \nonumber\\
    &= \frac{q(x')}{4}, \label{eqn:marginal-instance}
\end{align}
where in the second step uses the fact that $\mu$ and $\mu_\xi$ are uniform and \cref{eqn:conditional-instance}. We are now ready to prove our desired result.

\paragraph{Temporal contrastive learning cannot distinguish between good and bad decoder for all $l \in [m-1]$.} We first recall that temporal contrastive learning approach use the given observed data $(x_1, x_2)$ to compute a set of real and fake observation tuples. This is collected into a dataset $(x, x', z)$ where $z=1$ indicates that $(x_1=x, x_2=x')$ was observed in the dataset, and $z=0$ indicates that $(x_1=x, x_2=x')$ was not observed, or is an imposter. We sample $z$ uniformly in $\{0, 1\}$. The fake data is constructed by take $x=x_1$ from one tuple and $x'=x_2$ from another observed tuple. We start by computing the optimal Bayes classifier for the temporal contrastive learning approach using the definition of Bayes classifier in \cref{lem:bayes-classifier}.
\begin{align*}
    \distcont(z=1 \mid x, x') = \frac{D(x' \mid x)}{D(x' \mid x) + \rho(x')} = \frac{\one\{s'=1-s\}\one\{\xi'=1-\xi\}}{\one\{s'=1-s\}\one\{\xi'=1-\xi\} + 1/4},
\end{align*}
where we use \cref{lem:bayes-classifier} in the first step and \cref{eqn:conditional-instance,eqn:marginal-instance} in the second step. Recall that $z=1$ denotes whether a given observation tuple $(x, x')$ is real rather than an imposter/false. Note that since we have $k=1$, as it is a $H=1$ problem, we drop the notation $k$ from all terms.

The marginal distribution over $(x, x')$ for the temporal contrastive is given by \cref{eqn:marginal-tempcont-dist} which in our case instantiates to:
\begin{align}
    \distcont(x, x') &= \frac{D(x)}{2} \cbr{D(x' \mid x) + \rho(x')}, \nonumber \\
    &= \frac{1}{8} q(x') q(x) \cbr{\one\{s'=1-s\}\one\{\xi'=1-\xi\} + 1/4}, \label{eqn:tempcont-marginal-instance}
\end{align}
where we use \cref{eqn:conditional-instance,eqn:marginal-instance}, and $D(x) = q(x) \mu(s) \mu_\xi(\xi) = \nicefrac{q(x)}{4}$.

Let $g \in \Gcal$ be any classifier head. Given a decoder $\phi$, we define $g \circ \phi: (x, x') \mapsto g(\phi(x), \phi(x'))$ as a model for temporal contrastive learning, with an expected contrastive loss of:
\begin{align*}
    &\losscont(g, \edo) \\
    &= \EE_{(x, x') \sim \distcont, z \sim \distcont(\cdot \mid x, x')}\sbr{\rbr{z - g\rbr{\edo(x), \edo(x')}}^2} \\
    &= \EE_{(x, x') \sim \distcont}\sbr{\distcont(z=1 \mid x, x')\rbr{1-2g\rbr{\edo(x), \edo(x')}} + g\rbr{\edo(x), \edo(x')}^2} \\
    &= \frac{1}{8}\sum_{s, \xi, s', \xi'} \cbr{\one\{s'=1-s\}\one\{\xi'=1-\xi\} + \frac{1}{4}} \rbr{\frac{\one\{s'=1-s\}\one\{\xi'=1-\xi\}}{\one\{s'=1-s\}\one\{\xi'=1-\xi\} + \frac{1}{4}}(1-g(s, s')) + g(s, s')^2}
\end{align*}

Similarly, the expected temporal contrastive loss of the model $g \circ \edo$ with the bad decoder $\exo$ is given by:
\begin{align*}
    &\losscont(g, \exo) \\
    &= \EE_{(x, x') \sim \distcont, z \sim \distcont(\cdot \mid x, x')}\sbr{\rbr{z - g\rbr{\exo(x), \exo(x')}}^2} \\
    &= \frac{1}{8}\sum_{s, \xi, s', \xi'} \cbr{\one\{s'=1-s\}\one\{\xi'=1-\xi\} + \frac{1}{4}} \rbr{\frac{\one\{s'=1-s\}\one\{\xi'=1-\xi\}}{\one\{s'=1-s\}\one\{\xi'=1-\xi\} + \frac{1}{4}}(1-g(\xi, \xi')) + g(\xi, \xi')^2}
\end{align*}

Note that by interchanging $s$ with $\xi$ and $s'$ with $\xi'$, we can show $\losscont(g, \exo) = \losscont(g, \edo)$. Therefore, $\inf_{g \in \Gcal}\losscont(g, \exo) = \inf_{g \in \Gcal} \losscont(g, \edo)$. This implies that for any value of $l$, the temporal contrastive loss assigns the same loss to the good decoder $\edo$ and the bad decoder $\exo$. Hence, in practice, temporal contrastive cannot distinguish between the good and bad decoder and may converge to the latter leading to poor downstream performance. This convergence to the bad decoder may happen if it is easier to overfit to noise. For example, in our gridworld example, it is possibly easier for the model to overfit to the predictable motion of noise than understand the underlying dynamics of the agent. This is observed in \cref{fig:gridworld-reconstruction} where the representation learned via temporal contrastive tends to overfit to the noisy exogenous pixels and perform poorly on downstream RL tasks (\cref{fig:gridworld_exps}).

\paragraph{Forward modeling learns the good decoder if $l < \lfloor m/2 \rfloor$.} We likewise analyze the expected forward modeling loss of the good and bad decoder. For any $f \in \Fcal$, we have $f(x' \mid u)$ as the generator head that acts on a given decoder's output $u \in \{0, 1\}$ and generates the next observation $x'$. 

If we use the good decoder $\edo$, then we cannot predict the exogenous noise $\xi$ or $\xi'$ which can be either 0 or 1 with equal probability. This implies that for the $l$ noisy bits $v_1, \cdots, v_l$ in $x'$, the best prediction is that each one has an equal probability of taking 0 or 1. To see this, fix $i \in [l]$ and recall that $\PP(v_i=\xi' \mid \xi') = 0.8$ and $\PP(v_i=1-\xi' \mid \xi') = 0.2$. As $\xi'$ has equal probability of taking value 0 or 1, therefore, $\PP(v_i=u) = \sum_{\xi' \in \{0, 1\}} \PP(v_i = u \mid \xi') \nicefrac{1}{2} = \frac{0.8 + 0.2}{2} = 0.5$. However, since we can deterministically predict $s'$, therefore, we can predict the true distribution over $w_j$ for all $j \in [m-l]$. Let $f_{\textrm{good}}$ be this generator head. Formally, we have:
\begin{equation*}
    f_{\textrm{good}}(x' \mid \edo(x)) = \underbrace{(1/2)}_\text{due to $x'_1=\xi'$} \cdot \underbrace{(1/2)^l}_\text{due to $v_{1:l}$} \cdot \underbrace{\prod_{j=l+2}^{m+1} \psamp(x'_j \mid 1 - \edo(x))}_\text{due to $w_{1:m-l}$} \cdot \underbrace{\one\{x'_{m+2} = 1-\edo(x)\}}_\text{due to $x'_{m+2}=s'$}
\end{equation*}

The Bayes distribution is given by:
\begin{align*}
    &D(x' \mid x) \\
    &= q\rbr{x'}\cdot\one\{\edo(x') = 1 - \edo(x)\}\cdot \one\{\exo(x') = 1 - \exo(x)\}\\
    &= \one\cbr{x'_1 = 1 - \exo(x)} \cdot \prod_{i=1}^l \psamp(x'_{i+1} \mid 1 - \exo(x)) \cdot \prod_{j=l+2}^{m+1} \psamp(x'_j \mid 1 - \edo(x)) \one\cbr{x'_{m+2}=1-\edo(x)}.
\end{align*}

As we are optimizing the log-loss, we look at the expected KL divergence $\ell_{kl}$ between the $D(x' \mid x)$ and $f_{\textrm{good}}(x' \mid \edo(x))$ which gives:
\begin{align*}
    &\ell_{kl}(f_{\textrm{good}}, \edo) \\
    &=\EE_{x}\sbr{\sum_{x'} D(x' \mid x) \ln \frac{D(x' \mid x)}{f_{\textrm{good}}(x' \mid \edo(x))}} \\
    &= \EE_{x}\sbr{\sum_{x'} D(x' \mid x) \ln \frac{\one\cbr{x'_1 = 1 - \exo(x)} \cdot \prod_{i=1}^l \psamp(x'_{i+1} \mid 1 - \exo(x))}{(1/2)^{l+1}}}\\
    &= (l+1)\ln(2) + \EE_{x}\sbr{\sum_{x'} D(x' \mid x) \ln \rbr{\one\cbr{x'_1 = 1 - \exo(x)} \cdot \prod_{i=1}^l \psamp(x'_{i+1} \mid 1 - \exo(x))}}\\
    &= (l+1)\ln(2) + \EE_{x}\sbr{\sum_{i=1}^l \sum_{x'_{i+1} \in \{0, 1\}} \psamp(x'_{i+1} \mid 1 - \edo(x)) \ln  \psamp(x'_{i+1} \mid 1 - \edo(x))}\\
    &= (l+1) \ln(2) - l H(\psamp),
\end{align*}
where $H(\psamp)$ denotes the conditional entropy given by $-\nicefrac{1}{2}\sum_{s \in \{0, 1\}} \sum_{v \in \{0, 1\}} \psamp(v \mid s)\ln \psamp(v \mid s).$ As $\psamp(u \mid u) = 0.8$ and $\psamp(1-u \mid u) = 0.2$, we have $H(\psamp) = -0.8\ln(0.8) - 0.2\ln(0.2) \approx 0.500$. Plugging this in, we get $\ell_{kl}(f_{\textrm{good}}, \edo) = l \ln(2) - 0.5l + \ln(2) = \ln(2) + 0.193 l$.

Finally, the analysis when we use the $\exo$ decoder is identical to above. In this case, we can predict $\exo(x')$ and correctly predict the $\psamp$ distribution over all the $l$-noisy bits $v_{1:l}$. However, for the $w_{1:m-l}$ bits and the $x'[m+2]$, our best bet is to predict a uniform distribution. We capture this by the generator $f_{\textrm{bad}}$ which gives:
\begin{equation*}
    f_{\textrm{bad}}(x' \mid \edo(x)) = \underbrace{(1/2)}_\text{due to $x'_{m+2}=s'$} \cdot \underbrace{(1/2)^{m-l}}_\text{due to $w_{1:m-l}$} \cdot \underbrace{\prod_{i=2}^{l+1} \psamp(x'_i \mid 1 - \exo(x))}_\text{due to $v_{1:l}$} \cdot \underbrace{\one\{x'_{1} = 1-\exo(x)\}}_\text{due to $x'_1=\xi'$}
\end{equation*}

The expected KL loss $\ell_{kl}(f_{\textrm{bad}}, \exo)$ can be computed almost exactly as before and is equal to $\ln(2) + 0.193 (m-l)$. We can see that for $\ell_{kl}(f_{\textrm{good}}, \edo) < \ell_{kl}(f_{\textrm{bad}}, \exo)$ we must have $\ln(2) + 0.193 l < \ln(2) + 0.193 (m-l)$, or equivalently, $l < m/2$. This completes the analysis.

\section{Additional Experimental Details}
\label{appendix:exps}

All results are reported with mean and standard error computed over 3 seeds. All the code for this work was run on A100, V100, P40 GPUs, with a compute time of approx. 12 hours for grid world experiments and 6 hours for ViZDoom experiments. Data collection was done via pretrained PPO policies along with random walks for diversity in the observation space.

\subsection{Hyperparameters}
In~\cref{tab:hyp}, we report the hyperaparameter values used for experiments in this work with the GridWorld and ViZDoom environments.

\begin{table}
    \centering
    \begin{tabular}{|c|c|} \hline 
    Hyperparameter & Value
\\ \hline \hline
         batch size& 128
\\ \hline 
         learning rate & 0.001
\\ \hline 
         epochs& 400
\\ \hline 
         \# of exogenous variables& 10
\\ \hline 
         exogenous pixel size& 4
\\ \hline 
         \# of VQ heads& 2
\\ \hline 
         VQ codebook size& 100
\\ \hline 
         VQ codebook temperature& 0
\\ \hline 
         VQ codebook dimension& 32
\\ \hline 
         VQ bottleneck dimension& 1024\\ \hline
    \end{tabular}
    \caption{Hyperparameters used for experiments with the GridWorld and ViZDoom domains.}
    \label{tab:hyp}
\end{table}

\subsection{Additional Results}

\paragraph{I.I.D. Noise in the Basic ViZDoom environment.} We evaluate the  representation learning methods on the basic ViZDoom domain but with independent and identically distributed (iid) noise. We add iid Gaussian noise to each pixel sampled from a 0 mean Gaussian distribution with a standard deviation of 0.001. Based on theory, we expect temporal contrastive objectives to be substantially better at filtering out Gaussian iid noise, which is validated experimentally for the basic ViZDoom Environment (\cref{fig:vizdoom_iid_noise})(a).~\cref{fig:vizdoom_iid_noise}(b) refers the basic ViZDoom result from for convenient comparison.

\begin{figure}%
    \centering
    {\includegraphics[width=10cm]{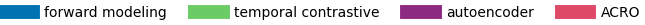}}\\
    \subfloat[\centering I.I.D. Gaussian Noise]
    {\includegraphics[width=4.5cm]{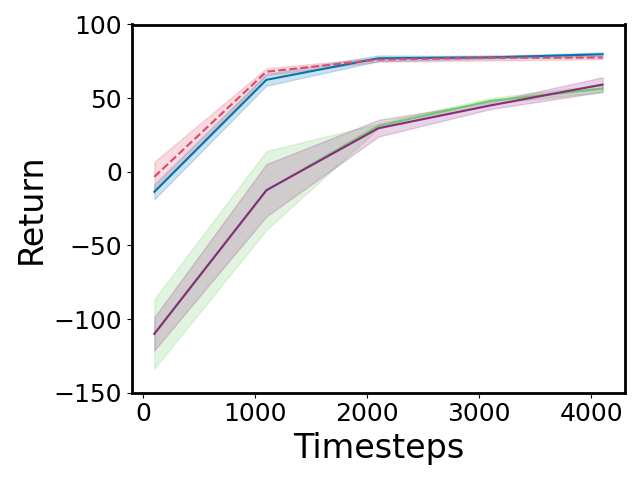} }
    \hspace{7mm}
    \subfloat[\centering Exogenous Noise]{{\includegraphics[width=4.5cm]{./figures/vizdoom/final_plot_enabl-exo_True_exo-reward_False.png} }}\\
    \caption{Experiments with (a) Guassian iid noise for the ViZDoom environment and (b) exogenous noise. 
    }
\end{figure}

\paragraph{Harder Exogenous Noise.} We provide additional experiments in gridworld where we increase the number of exogenous noise variables (diamond shapes overlayed on image) while keeping their sizes fixed at 4 pixels. We present result in ~\cref{fig:gridworld_more_noise} which shows significant degradation in the performance of video-based representation methods, whereas ACRO which uses trajectory data continues to perform well. This supports one of our main theoretical result that video-based representations can be exponentially worse than trajectory-based representations in the presence of exogenous noise.

\begin{figure}%
    \centering
    {\includegraphics[width=13cm]{figures/gridworld/more_noise/num_noise_leg.png}}\\
    \subfloat[\centering Autoencoder]{{\includegraphics[width=4.5cm]{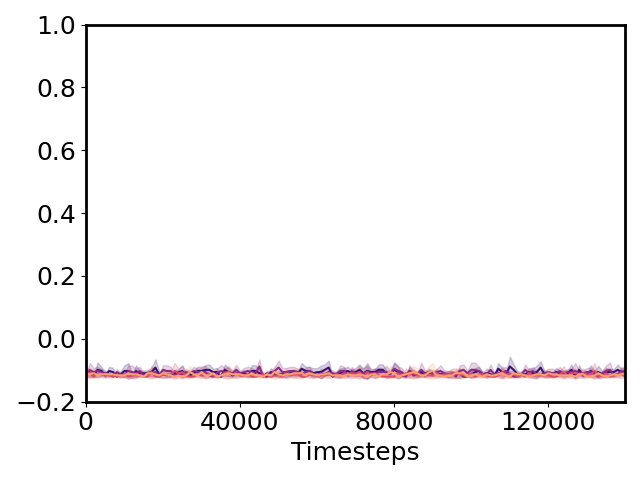} }}
    \subfloat[\centering VQ-Autoencoder]{{\includegraphics[width=4.5cm]{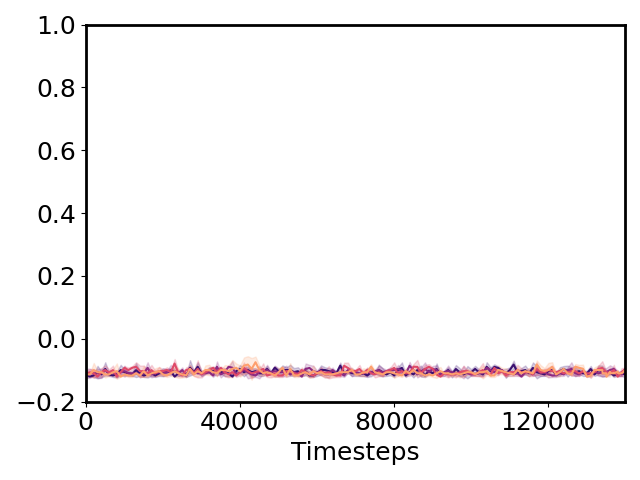} }}
    \subfloat[\centering Forward Modeling]{{\includegraphics[width=4.5cm]{figures/gridworld/more_noise/alg_next-frame_vq_-1.png} }}\\
    \subfloat[\centering VQ-Forward Modeling]{{\includegraphics[width=4.5cm]{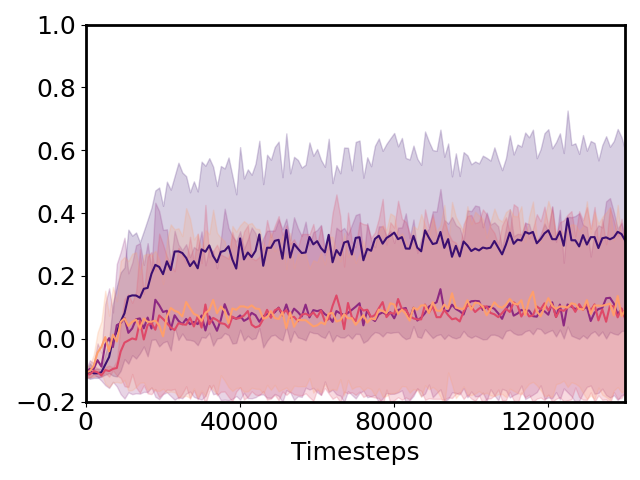} }}
    \subfloat[\centering Temporal Contrastive]{{\includegraphics[width=4.5cm]{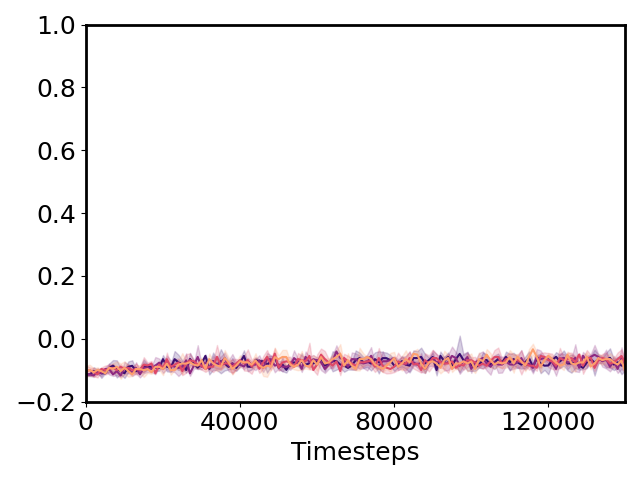} }}
    \subfloat[\centering ACRO]{{\includegraphics[width=4.5cm]{figures/gridworld/more_noise/alg_invdynamics_vq_-1.png} }}
    
    \caption{Gridworld experiments with exogenous noise of size 4 for varying number of exogenous noise variables. Several representation learning algorithms using video data struggle to learn with larger number of exogenous noise variables, whereas ACRO which uses trajectory data, still performs well.}
    \label{fig:gridworld_more_noise}%
\end{figure}

\paragraph{Additional decoded image reconstructions} are shown in ~\cref{fig:reconstruction_gridworld2} for the GridWorld environment and in ~\cref{fig:vizdoom-reconstruction2} for the basic ViZDoom environment. We highlight that important parts of the observation space are recovered successfully by the forward modeling approach under varying levels of exogenous noise, whereas temporal contrastive learning often fails.

\begin{figure}
\centering
    \subfloat{{\includegraphics[width=3.5cm]{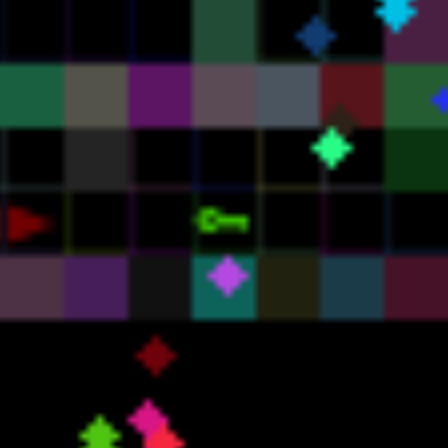} }}
    \subfloat{{\includegraphics[width=3.5cm]{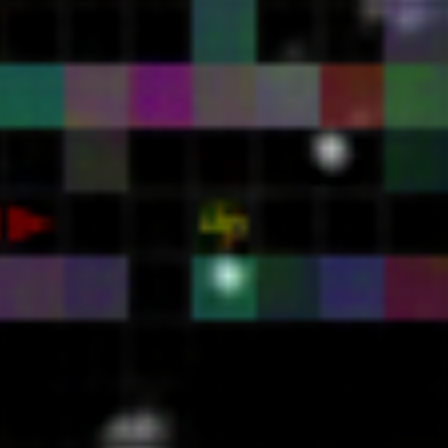} }}
    \subfloat{{\includegraphics[width=3.5cm]{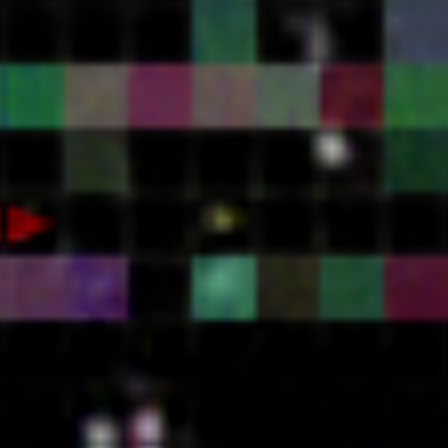} }}
    \subfloat{{\includegraphics[width=3.5cm]{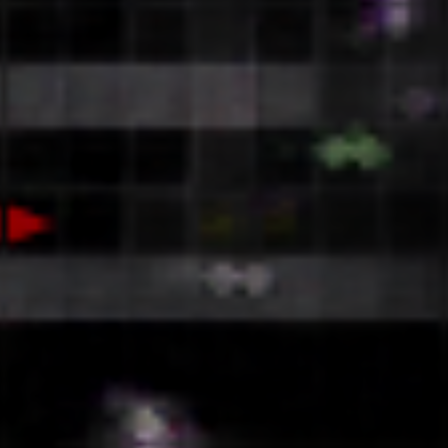} }}\\
        \subfloat{{\includegraphics[width=3.5cm]{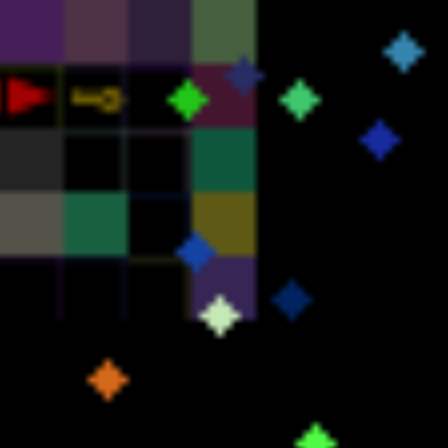} }}
    \subfloat{{\includegraphics[width=3.5cm]{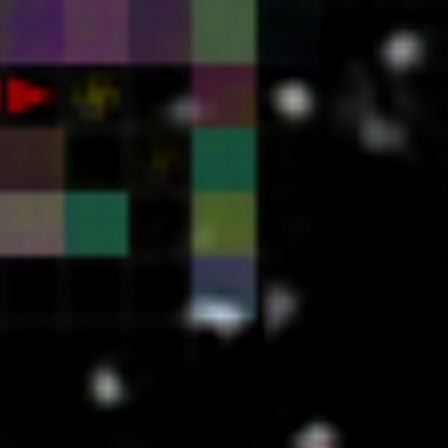} }}
    \subfloat{{\includegraphics[width=3.5cm]{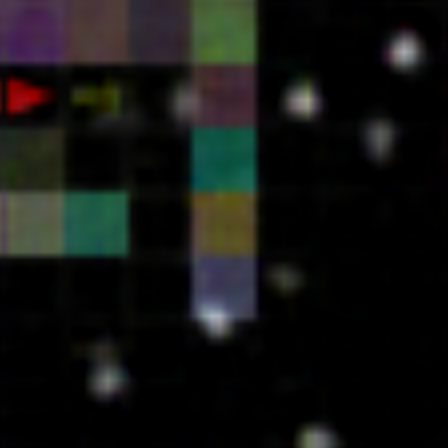} }}
    \subfloat{{\includegraphics[width=3.5cm]{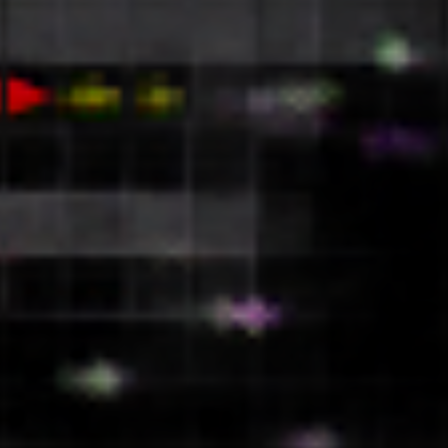} }}\\
        \subfloat[\centering Original]{{\includegraphics[width=3.5cm]{./figures/gridworld/recon_original/observation_33.png} }}
    \subfloat[\centering Forward Modeling]{{\includegraphics[width=3.5cm]{./figures/gridworld/recon_nextframe/prediction_33.png} }}
    \subfloat[\centering Autoencoder]{{\includegraphics[width=3.5cm]{./figures/gridworld/recon_autoencoder/prediction_33.png} }}
    \subfloat[\centering Temporal Contrastive]{{\includegraphics[width=3.5cm]{./figures/gridworld/recon_tempcont/prediction_33.png} }}
    \caption{Decoded image reconstructions from different latent representation learning methods in the GridWorld environment. We train a decoder on top of frozen representations trained with the three video pre-training approaches. }
    \label{fig:reconstruction_gridworld2}
\end{figure}

\begin{figure}[!thb]
\centering
    \subfloat{{\includegraphics[width=3.5cm]{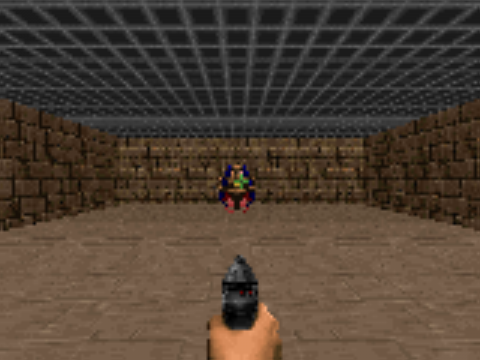} }}
    \subfloat{{\includegraphics[width=3.5cm]{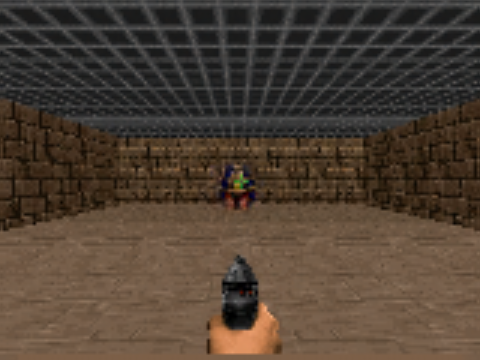} }}
    \subfloat{{\includegraphics[width=3.5cm]{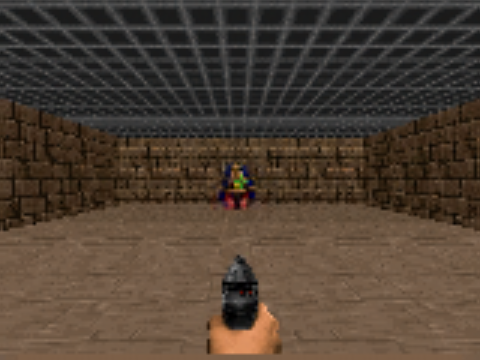} }}
    \subfloat{{\includegraphics[width=3.5cm]{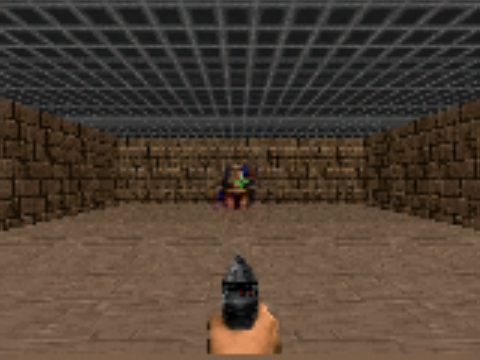} }}\\
    \subfloat{{\includegraphics[width=3.5cm]{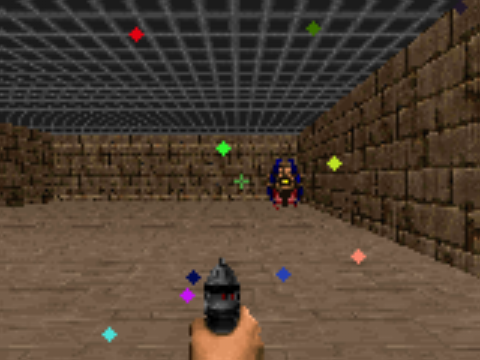} }}
    \subfloat{{\includegraphics[width=3.5cm]{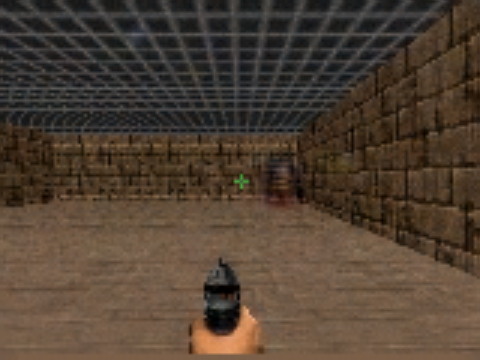} }}
    \subfloat{{\includegraphics[width=3.5cm]{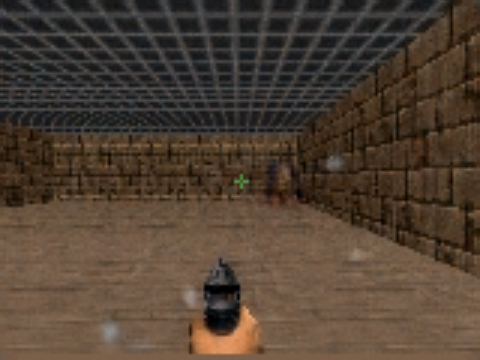} }}
    \subfloat{{\includegraphics[width=3.5cm]{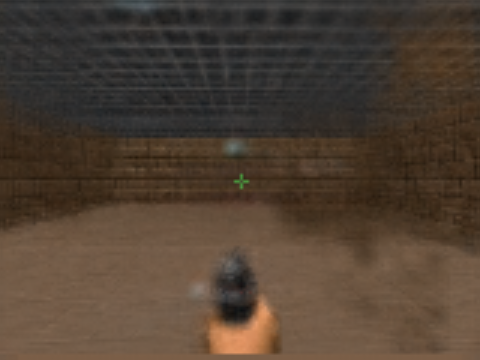} }}\\
    \subfloat[\centering Original]{{\includegraphics[width=3.5cm]{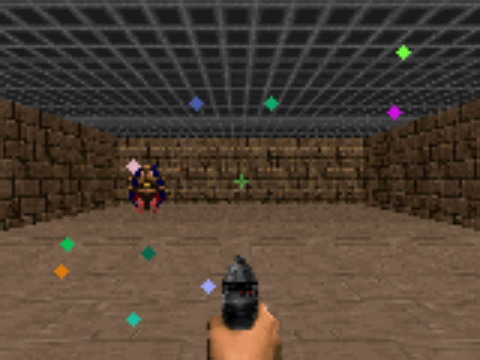} }}
    \subfloat[\centering Forward Modeling]{{\includegraphics[width=3.5cm]{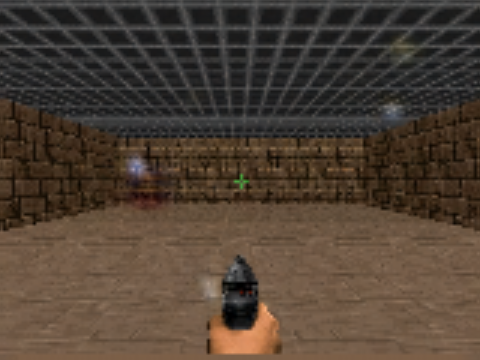} }}
    \subfloat[\centering Autoencoder]{{\includegraphics[width=3.5cm]{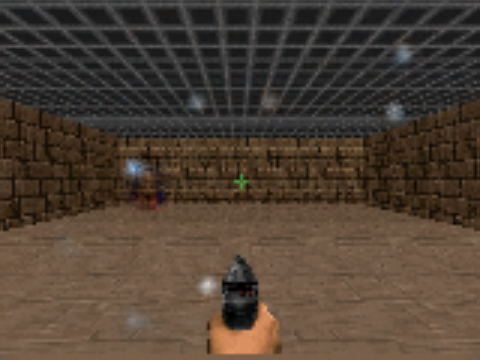} }}
    \subfloat[\centering Temporal Contrastive]{{\includegraphics[width=3.5cm]{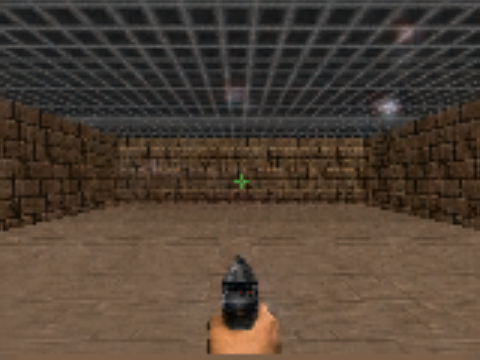} }}
    \caption{Decoded  image reconstructions from different latent representation learning methods in the ViZDoom environment. We train a decoder on top of frozen representations trained with the three video pre-training approaches.  
    }
    \label{fig:vizdoom-reconstruction2}
\end{figure}

\end{document}